\documentclass[11pt]{article}

\usepackage[preprint]{acl}

\usepackage{times}
\usepackage{latexsym}
\usepackage[ruled,vlined]{algorithm2e}

\usepackage[T1]{fontenc}

\usepackage[utf8]{inputenc}
\usepackage{microtype}
\usepackage{booktabs}

\usepackage{microtype}

\usepackage{inconsolata}

\usepackage{subfigure}
\usepackage{graphicx}
\usepackage{array}
\usepackage{multirow}
\usepackage{booktabs}
\usepackage{colortbl}
\usepackage{xcolor}
\usepackage{vcell}
\usepackage{arydshln}
\usepackage{amsmath}
\usepackage{graphicx}   
\usepackage{caption}    
\usepackage{subcaption} 
\usepackage{amsfonts}
\usepackage{tabularx}
\usepackage{makecell}
\usepackage[normalem]{ulem}
\usepackage{CJKutf8}
\usepackage{pifont}
\usepackage{arydshln}
\usepackage{tikz}
\usepackage{pgfplots}
\usepackage{inconsolata}
\usepackage{fancyvrb}
\usepackage{tabularx}
\usepackage{colortbl}
\usepackage{xcolor}
\usepackage{multirow}
\usepackage{makecell}
\usepackage{tabularx}
\usepackage{xcolor}
\usepackage{colortbl}
\usepackage{tcolorbox}
\usepackage{amsthm}
\newtheorem{theorem}{Theorem}

\tcbuselibrary{skins,breakable}
\newtcolorbox{mtcard}{
  enhanced,
  colback=gray!5,
  colframe=gray!40,
  boxrule=0.3pt,
  arc=2pt,
  left=4pt,
  right=4pt,
  top=4pt,
  bottom=4pt,
}
\newtcolorbox{mtbox}{
  colback=gray!5,
  colframe=gray!40,
  boxrule=0.3pt,
  arc=2pt,
  left=3pt,right=3pt,top=3pt,bottom=3pt
}

\newtcolorbox{pebox}{
  colback=green!5,
  colframe=green!40,
  boxrule=0.3pt,
  arc=2pt,
  left=3pt,right=3pt,top=3pt,bottom=3pt
}

\newtcolorbox{attrbox}{
  colback=white,
  colframe=gray!20,
  boxrule=0.2pt,
  arc=1pt,
  left=2pt,right=2pt,top=2pt,bottom=2pt
}

\usepackage{tcolorbox}

\definecolor{lightgray}{gray}{0.95}
\definecolor{lightgreen}{rgb}{0.90,0.97,0.90}

\usepackage{todonotes}
\usepackage[capitalise,noabbrev]{cleveref}
\usepackage{listings}
\usepackage{fancyvrb} 
\makeatletter
\def\adl@drawiv#1#2#3{%
        \hskip.5\tabcolsep
        \xleaders#3{#2.5\@tempdimb #1{1}#2.5\@tempdimb}%
                #2\z@ plus1fil minus1fil\relax
        \hskip.5\tabcolsep}

\definecolor{verbgray}{gray}{0.9}

\usepackage{colortbl}
\definecolor{lightgray}{rgb}{0.7,0.7,0.7}

\definecolor{light-orange}{HTML}{fee9d4}
\definecolor{light-green}{HTML}{d8f0d3}
\definecolor{light-blue}{HTML}{dae8f5}
\definecolor{light-red}{HTML}{FBC7C4}
\definecolor{set10-red}{HTML}{e41a1c}
\definecolor{set10-blue}{HTML}{377eb8}
\definecolor{set10-green}{HTML}{4daf4a}
\definecolor{bblue}{HTML}{4F81BD}
\definecolor{rred}{HTML}{c4260b}
\definecolor{ggreen}{HTML}{098c1f}
\definecolor{ppurple}{HTML}{9F4C7C}
\definecolor{oorange}{HTML}{F79646}
\usepackage{colortbl}
\usepackage{xcolor}
\definecolor{oursorange}{RGB}{255,245,230}

\usepackage{enumitem}
\setlist[itemize,enumerate]{leftmargin=*}
\usepackage[normalem]{ulem}
\usepackage{pgfplots}
\usetikzlibrary{pgfplots.groupplots}
\pgfplotsset{compat=1.3}
\usepackage{tikz}
\usetikzlibrary{patterns}
\usepackage{xcolor}
\usepackage{todonotes}
\usepackage{listings}
\usepackage{multirow}
\usepackage{graphicx}
\definecolor{CustomBlue}{RGB}{57,83,191}
\usepackage{tcolorbox}

\newtcolorbox{casebox}{
  colback=gray!3,
  colframe=gray!40,
  boxrule=0.5pt,
  arc=2mm,
  left=6pt,
  right=6pt,
  top=6pt,
  bottom=6pt
}

\title{PEGRL: Improving Machine Translation by Post-Editing Guided Reinforcement Learning}
\newcommand*{\affmark}[1][*]{\textsuperscript{#1}}
\newcommand*{\affaddr}[1]{#1}
\newcommand*{\email}[1]{\texttt{#1}}

\author{
Yunzhi Shen\affmark[1]\quad 
Hao Zhou\affmark[1]\quad 
Xin Huang\affmark[2]\thanks{Co-corresponding authors.}\quad
Xue Han\affmark[2] \quad 
Junlan Feng\affmark[2] \\
\textbf{Shujian Huang\affmark[1]\footnotemark[1]}
\\
\affaddr{\affmark[1]National Key Laboratory for Novel Software Technology, Nanjing University}\\
\affaddr{\affmark[2]China Mobile Research Beijing, China}\\
\email{\{shenyunzhi, zhouh\}@smail.nju.edu.cn}\quad
\email{huangsj@nju.edu.cn} \\
\email{\{huangxin, hanxuejt, fengjunlan\}@cmjt.chinamobile.com}\\
}

\usepackage{tcolorbox}
\begin{document}
\maketitle
\begin{abstract}
Reinforcement learning (RL) has shown strong promise for LLM-based machine translation, with recent methods such as GRPO demonstrating notable gains; nevertheless, translation-oriented RL remains challenged by noisy learning signals arising from Monte Carlo return estimation, as well as a large trajectory space that favors global exploration over fine-grained local optimization.
We introduce \textbf{PEGRL}, a \textit{two-stage} RL framework that uses post-editing as an auxiliary task to stabilize training and guide overall optimization.
At each iteration, translation outputs are sampled to construct post-editing inputs, allowing return estimation in the post-editing stage to benefit from conditioning on the current translation behavior, while jointly supporting both global exploration and fine-grained local optimization.
A task-specific weighting scheme further balances the contributions of translation and post-editing objectives, yielding a biased yet more sample-efficient estimator.
Experiments on English$\to$Finnish, English$\to$Turkish, and English$\leftrightarrow$Chinese show consistent gains over RL baselines, and for English$\to$Turkish, performance on COMET-KIWI is comparable to advanced LLM-based systems (DeepSeek-V3.2). Our code and a set of representative pretrained models are publicly available at \url{https://github.com/NJUNLP/peg-rl} and \url{https://huggingface.co/collections/DGME/pegrl}.
\end{abstract}

\section{Introduction}
\label{sec:intro}

Reinforcement learning (RL) techniques on large language models (LLMs) have achieved notable advances, exemplified by DeepSeek-R1~\citep{deepseekai2025deepseekr1incentivizingreasoningcapability}, which demonstrates strong performance on verifiable tasks such as mathematical reasoning and code generation. More recently, RL-based methods, such as GRPO~\citep{shao2024deepseekmathpushinglimitsmathematical}, have been adapted for machine translation through the use of automatic evaluation metrics, including BLEU~\citep{post-2018-call} and COMET-style metrics~\citep{rei-etal-2022-comet,rei2023scalingcometkiwiunbabelist2023}, as reward signals~\citep{he2025r1t1fullyincentivizingtranslation,feng2025mtr1zeroadvancingllmbasedmachine}. Despite these initial improvements, \citet{zeng2025shrinkingvarianceshrinkagebaselines} show that the Monte Carlo group-wise baseline used in GRPO may suffer from high estimation variance, causing instability in training and suggesting opportunities for further refinement.

Moreover, the large trajectory space in translation-oriented RL tends to emphasize \textbf{global exploration}, while providing limited optimization signals for fine-grained local improvements. Thus the corresponding translation quality is limited, especially for those low-resource translation directions, or those models that are not thoroughly trained.

\begin{figure}[t]
    \centering
    \includegraphics[width=0.43\textwidth]{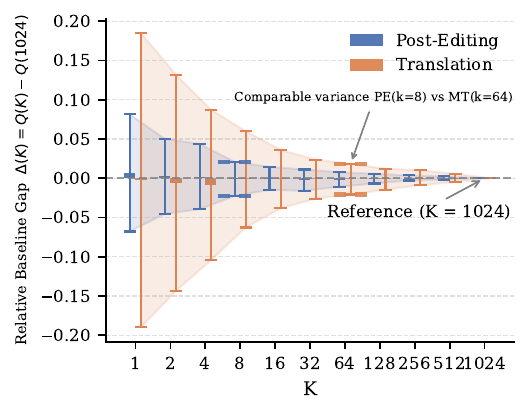}
\caption{Convergence of the GRPO \emph{group-wise baseline} with respect to the number of sampled trajectories $K$.
For each of 100 instances, we roll out 1024 trajectories and use the resulting baseline as a reference.
We report the mean and standard deviation (error bars) of the relative gap $\Delta(K)=Q(K)-Q(1024)$, where $K$ denotes the GRPO group size.
Larger $K$ reduces Monte Carlo variance (Appendix~\ref{sec:appendix_mc}), making $Q(1024)$ a potential proxy for the true baseline $\mathbb{E}[R]$.
Smaller error bars indicate more stable baseline estimation.}

    \label{fig:baseline}
\end{figure}

Compared to machine translation, post-editing refines an existing target-side draft with typically minor edits~\citep{melby-1984-machine,do2021review,lim2025mufumultilingualfusedlearning}, enabling \textbf{exploration within a more localized output neighborhood} for a given translation trajectory. As shown in Figure~\ref{fig:baseline}, post-editing also exhibits substantially lower baseline variance than translation, indicating potentially smaller policy gradient variance and more stable training. 

We propose to model the translation workflow as a two-step process: translation followed by post-editing. This allows post-editing to perform fine-grained exploration of the output space based on the initial translation trajectory for improved translations. 
As a subsequent stage, the post-edited outputs directly reflect the quality of the edited translation, providing more stable learning signals for optimizing the translation policy, which helps mitigate the noise introduced by return estimation in the translation task itself.

This workflow is formulated as a two-stage RL problem. Under Monte Carlo sampling, the joint policy gradient decomposes into additive contributions from translation and post-editing, naturally aligning with the intuition outlined in the previous paragraph (see Section~\ref{sec:problem} for details). Motivated by variance considerations in return estimation, we introduce a task-specific weighting scheme that places greater emphasis on the post-editing learning signal, whose baseline provides a more stable estimate of the optimized return, while down-weighting the translation term that involves additional variability. Although this results in a biased estimator, we demonstrate both theoretically and empirically that it is more sample-efficient than its unbiased counterpart. To optimize the weighted objective, we introduce \textbf{PEGRL}, a GRPO-based dual-task training framework in which translation produces on-policy data for post-editing at each iteration. 
This design enables comprehensive exploration while ensuring that the post-editing objective, whose return estimation benefits from conditioning on the current translation policy, is optimized under up-to-date translation behavior. 
Our experiments further show that local exploration induced by post-editing promotes more efficient global exploration (see Section~\ref{sec:ablation}).

We evaluate our approach on English$\to$Finnish, English$\to$Turkish, and English$\leftrightarrow$Chinese translation using the WMT24 and FLORES benchmarks. Across chrF++, COMETKIWI, and XCOMET, our method consistently outperforms the RL baseline MT-R1-Zero~\citep{feng2025mtr1zeroadvancingllmbasedmachine}, with particularly strong gains in less-covered language directions for the base model (EN$\to$FI and EN$\to$TR). Notably, on English$\to$Turkish, our COMET-KIWI scores are competitive with state-of-the-art LLMs such as DeepSeek-V3.2~\citep{deepseekai2025deepseekv32pushingfrontieropen}. These results demonstrate the effectiveness of our framework in leveraging more stable learning signals to improve translation quality. Our main contributions are as follows:
\begin{itemize}
\setlength{\itemsep}{0.3em}     
\setlength{\parsep}{0em}      
\setlength{\parskip}{0em}     
\setlength{\topsep}{0em}
    \item We analyze the policy gradients of post-editing and show that, under GRPO, the corresponding baseline is substantially easier to estimate than that of direct translation.

    \item We propose a two-stage translation framework that integrates translation and post-editing to enable joint global and local RL exploration, with task-specific gradient weighting that exploits the lower-variance post-editing signal for more stable and sample-efficient learning.

    \item We implement a GRPO-based dual-task RL framework and demonstrate its effectiveness on WMT24 and FLORES datasets (EN$\to$FI, EN$\to$TR, EN$\leftrightarrow$ZH), outperforming strong RL baselines, and achieving performance on some metrics and directions comparable to SOTA LLMs.

\end{itemize}

\section{Related Work}

\paragraph{LLMs for Post-Editing}  
LLMs have shown strong inference-time post-editing performance on WMT benchmarks~\citep{raunak2023leveraginggpt4automatictranslation}, but training-time LLM post-editing remains underexplored. \textit{Mufu}~\citep{lim2025mufumultilingualfusedlearning} uses a teacher--student setup with auxiliary translations but relies on a strong teacher and surface metrics. In contrast, we model post-editing as a learned policy within a unified RL framework, evaluated with both lexical and semantic metrics.

\paragraph{RL for Machine Translation}  
Inspired by RL successes on verifiable reasoning tasks \citep{deepseekai2025deepseekr1incentivizingreasoningcapability}, recent work adapts RL to translation using GRPO-style optimization with diverse reward designs. For example, R1-T1~\citep{he2025r1t1fullyincentivizingtranslation} combines COMET-based rewards with format signals, MT-R1-Zero~\citep{feng2025mtr1zeroadvancingllmbasedmachine} uses hybrid BLEU+COMET rewards, and DeepTrans~\citep{wang2025deeptransdeepreasoningtranslation} and SSR-Zero~\citep{yang2025ssrzerosimpleselfrewardingreinforcement} adopt trajectory-level generative rewards. These works focus primarily on reward design, while trajectory sampling and multi-stage or multi-task setups, which can significantly affect translation performance, have received less attention.

\paragraph{RL Algorithms for LLMs}  
Policy gradient methods for LLM post-training optimize expected reward:
\begin{align}
\mathcal{J}_{\mu}(\theta)
&= \mathbb{E}_{\tau \sim \pi_\theta(\cdot \mid q)}[R(\tau \mid q)], \notag \\
\nabla_\theta \mathcal{J}_{\mu}(\theta) 
&= \mathbb{E}_{\tau}[\widehat{A} (\tau,q) \,\nabla_\theta \log \pi_\theta(\tau \mid q)], \notag
\end{align}
with different methods computing the advantage $\widehat{A}$. PPO~\citep{schulman2017proximalpolicyoptimizationalgorithms} uses GAE~\citep{schulman2018highdimensionalcontinuouscontrolusing}, while GRPO~\citep{shao2024deepseekmathpushinglimitsmathematical} normalizes rewards over a group.

\section{Formal Framework}
\label{sec:problem}
We formulate machine translation and post-editing as sequential decision processes within a unified RL framework.
Let $q$ denote the initial translation prompt, and let $\tau_0 = (a_0, a_1, \ldots, a_{|\tau_0|})$ be the translation trajectory, where each $a_i$ is a translation token.
Conditioned on $\tau_0$, the model generates a post-editing trajectory $\tau_1 = (b_0, b_1, \ldots, b_{|\tau_1|})$, where each $b_i$ is a post-editing token.
The post-editing policy is additionally conditioned on an auxiliary prompt $p$, which, together with $q$, is derived from the same source input.

Let $\pi_\theta$ denote the LLM with parameters $\theta$.
We optimize a trajectory-level RL objective:
\begin{equation}
\max_{\theta}\;
\mathbb{E}_{\tau_0 \sim \pi_\theta(\cdot \mid q),\;\tau_1 \sim \pi_\theta(\cdot \mid p, \tau_0)}
\big[ R(\tau_1) \big].
\label{eq:obj}
\end{equation}

where the reward $R(\tau_1)$ is assigned to the post-editing trajectory.
The policy gradient of this objective is given by (see Appendix~\ref{app:pg} for details):
\begin{align}
&\nabla_\theta 
\mathbb{E}_{\tau_0 \sim \pi_\theta(\cdot \mid q),\;\tau_1 \sim \pi_\theta(\cdot \mid p, \tau_0)}
\big[ R(\tau_1) \big] \notag \\
&=
\mathbb{E}_{\tau_0,\tau_1}
\big[
\nabla_\theta \log \pi_\theta(\tau_1 \mid p, \tau_0)\, R(\tau_1)
\big] \notag \\
&\quad +
\mathbb{E}_{\tau_0}
\big[
\nabla_\theta \log \pi_\theta(\tau_0 \mid q)\;
\mathbb{E}_{\tau_1}[R(\tau_1)]
\big]. \label{eq:master_obj}
\end{align}

\subsection{Two-stage Monte-Carlo Estimation}
\label{sec:two_stage}
The policy gradient in the right hand side of Eq.~\eqref{eq:master_obj} involves nested expectations over
$\tau_0$ and $\tau_1$, which are intractable to compute exactly. 
To address this, we adopt a two-stage Monte Carlo estimator~\citep{metropolis1953equation} that removes the double expectation. 

Given a query $q$, we first sample $N$ trajectories 
$\{\tau_0^{(i)}\}_{i=1}^{N}$ from $\pi(\cdot \mid q)$. 
For each $\tau_0^{(i)}$, we then sample $M$ trajectories 
$\{\tau_1^{(i,j)}\}_{j=1}^{M}$ from $\pi(\cdot \mid p, \tau_0^{(i)})$.

We refer to the following term as the \emph{post-editing policy gradient}.
Using Monte Carlo sampling and expanding only the expectation over $\tau_0$, \textbf{the inner expectation reduces to a standard policy gradient} for post-editing conditioned on a fixed input $\tau_0^{(i)}$, derived via the log-derivative trick (Appendix~\ref{app:logderiv}).

\begin{align}
&\mathbb{E}_{\tau_0} \Big[ \mathbb{E}_{\tau_1} [\nabla_\theta \log \pi_\theta(\tau_1 \mid p,\tau_0)\,R(\tau_1)] \Big] \notag \\
&\approx \frac{1}{N} \sum_{i=1}^{N} \mathbb{E}_{\tau_1} [\nabla_\theta \log \pi(\tau_1 \mid p,\tau_0^{(i)})\,R(\tau_1)] \notag
\end{align}

Analogously, we refer to the following term as the \emph{translation policy gradient}.
Expanding only the expectation over $\tau_1$ yields \textbf{the policy gradient of the translation task} with respect to the input $q$, derived via the log-derivative trick (Appendix~\ref{app:logderiv}).

\begin{align}
&\mathbb{E}_{\tau_0} \Big[ \nabla_\theta \log \pi_\theta(\tau_0 \mid q)\;\mathbb{E}_{\tau_1}[R(\tau_1)] \Big] \notag \\
& \approx \mathbb{E}_{\tau_0^{(i)}}\Big[ \nabla_\theta \log \pi_\theta(\tau_0^{(i)} \mid q)\; \frac{1}{M} \sum_{j=1}^{M} R(\tau_1^{(i,j)})\Big]. \notag
\end{align}

\subsection{Optimization with GRPO}

Following the decomposition in Section~\ref{sec:two_stage}, we estimate both policy gradients using GRPO. For post-editing, the group-normalized advantage is computed directly from the post-editing reward $R(\tau_1)$. For translation, we use the average reward of the associated post-editing candidates to compute the group-normalized advantage for updating the translation policy.

\begin{equation}
\bar{R}^{(i)}_{\text{pe}} \;=\; \frac{1}{M} \sum_{j=1}^{M} R\!\left(\tau_1^{(i,j)}\right),
\label{eq:pe_avg_reward}
\end{equation}
where $\tau_1^{(i,j)}$ denotes the $j$-th post-editing trajectory associated with the $i$-th translation sample. Formally, this guides Stage~1 toward optimization directions that improve Stage~2 output quality.

\subsection{Variance Analysis of RL Baseline}
\label{sec:variance_analysis}

As illustrated in Section~\ref{sec:intro} and Fig.~\ref{fig:baseline}, 
starting from the baseline construction in GRPO advantage estimation, 
for a fixed draft trajectory $\tau_0$, 
the post-editing baseline
\(
\mathbb{E}_{\tau_1 \sim \pi_\theta(\cdot \mid \tau_0, p)} \big[ R(\tau_1) \big]
\)
provides a more accurate estimate than the translation-level baseline
\(
\mathbb{E}_{\tau_0 \sim \pi_\theta(\cdot \mid q)} \big[ R(\tau_0) \big].
\)
Moreover, \textbf{the translation gradient discussed in Section~\ref{sec:two_stage} }
requires estimating the nested expectation
\(
\mathbb{E}_{\tau_0 \sim \pi_\theta(\cdot \mid q),\;
\tau_1 \sim \pi_\theta(\cdot \mid p, \tau_0)} 
\big[ R(\tau_1) \big].
\)

The variance of the estimator,
$\mathrm{Var}_{\tau_0 \sim \pi_\theta(\cdot \mid q),\, \tau_1 \sim \pi_\theta(\cdot \mid \tau_0,p)}[R(\tau_1)]$,
decomposes into a non-negative between-$\tau_0$ term and
$\mathbb{E}_{\tau_0}[\mathrm{Var}_{\tau_1 \mid \tau_0}(R(\tau_1))]$ (Appendix \ref{sec:appendix_variance}).
The latter corresponds exactly to the variance of the post-editing estimator conditioned on a fixed $\tau_0$,
i.e., $\mathrm{Var}_{\tau_1 \sim \pi_\theta(\cdot \mid \tau_0,p)}[R(\tau_1)]$.
Therefore, conditioning on $\tau_0$ removes the between-$\tau_0$ variability and yields a lower-variance estimator in most cases. Accordingly, within our framework, the post-editing policy gradient baseline provides a lower-variance estimate than the translation policy gradient baseline.

\section{Methodology}

Based on the theoretical derivations presented earlier, we propose a GRPO-based RL training framework that jointly integrates the training of translation and post-editing. Unlike simple mixed RL training schemes \citep{deepseekai2025deepseekv32pushingfrontieropen}, the two tasks in our framework are tightly coupled: the translation component generates training data online for post-editing, while feedback from post-editing guides the translation model toward outputs that better facilitate downstream post-editing. We train a single model with both tasks simultaneously. \textbf{In a single training step}, trajectories are sampled from both tasks and contribute carefully weighted gradients (see Section~\ref{sec:problem_weight}) for model updates.

\subsection{Hybrid Sampling for Online Post-Editing Data Generation}

In our framework, translation and post-editing use separate prompts, reflecting the dual-task setup and avoiding the performance drop from multi-task prompts~\citep{khot2023decomposedpromptingmodularapproach}. The post-editing prompt is conditioned on the translation output and generated online during training~(Appendix~\ref{sec:appendix-prompts}).  

Thus we perform a hybrid sampling for both tasks. 
At each training step, for a translation pair $(\textit{src}, \textit{tgt})$, following the sampling procedure in Section~\ref{sec:two_stage}, we obtain $N$ translation trajectories $\{\textit{pred}_i\}_{i=1}^{N}$ and $N \times M$ post-editing trajectories $\{\textit{pe}_{i,j}\}_{i=1,j=1}^{N,M}$. In our main experiments, we set $N = M = 8$.

\subsection{Reward and Advantage}
Our reward function consists of three components.
First, the post-editing policy is trained with a quality estimation reward.
Second, the translation policy is optimized using the expected reward $\frac{1}{M}\sum_{j=1}^{M} R(\tau_1^{(i,j)})$ from the post-editing task.
Finally, we introduce a penalty term to discourage degenerate behaviors, such as unbounded or excessively long outputs.

\subsubsection{Reward for Post-editing}

The post-editing objective is defined to encourage quality improvements as
measured by a quality estimation function $f(\cdot)$.
Under the group-relative policy optimization (GRPO) framework, optimizing
improvement-based rewards is equivalent to directly optimizing absolute output
quality after group-advantage normalization.
A formal proof is provided in Appendix~\ref{app:grpo-qe-equivalence}.

To prevent degenerate updates, if the post-edited output does not modify the
initial translation ($\textit{pe}_{i,j} = \textit{pred}_i$) and its estimated
semantic quality falls below a threshold $\alpha$ (e.g., $\alpha = 0.95$,
which is used in all our experiments), we assign a zero reward.
Let $\mathcal{D}(u)$ denote this condition.
For each post-editing instance
$u = (\textit{src}, \textit{pred}_i, \textit{pe}_{i,j}, \textit{tgt})$,
the post-editing reward is defined as
\begin{equation}
R_{\mathrm{pe}}(u) =
\begin{cases}
0, & \mathcal{D}(u), \\
f(\textit{pe}_{i,j} \mid \textit{src}, \textit{tgt}), & \text{otherwise}.
\end{cases}
\label{eq:metric}
\end{equation}
In our subsequent experiments, $f(\cdot)$ is instantiated by COMETKiwi~\citep{rei2023scalingcometkiwiunbabelist2023} together with a surface-level metric, e.g., chrF++~\citep{popovic2017chrf++} or BLEU\citep{post-2018-call}.

\subsubsection{Reward For Translation}
When computing the translation reward, for each translation instance
$v = (\textit{src}, \textit{pred}_i, \textit{tgt})$,
we aggregate the contributions from all associated post-editing trajectories.
Let $\mathcal{C}(v)$ denote the set of post-editing trajectories corresponding
to $v$.
The translation reward is then defined as
\begin{equation}
R_{\text{mt}}(v) = \mathrm{Mean}\big( \{ R_{\text{pe}}(u) \mid u \in \mathcal{C}(v) \} \big).
\end{equation}

This formulation directly corresponds to the average post-editing reward defined in Eq.~\eqref{eq:pe_avg_reward}.

\subsubsection{Penalty Reward}
We disable explicit reasoning in Qwen3~\citep{yang2025qwen3technicalreport}, and thus do not use CoT during trajectory generation~\citep{wei2023chainofthoughtpromptingelicitsreasoning}. To discourage degenerate behaviors such as excessive repetition or unbounded outputs, any such trajectory is assigned a total reward of $-1$.

\begin{table*}[htbp]
    \centering
    \small
    \setlength{\tabcolsep}{4pt}
    \resizebox{\textwidth}{!}{%
    \begin{tabular}{
        p{3.8cm}
        *{3}{c} @{\hspace{6pt}}
        *{3}{c} @{\hspace{10pt}}
        *{3}{c} @{\hspace{6pt}}
        *{3}{c}
    }
    \toprule
    \multirow{1}{*}{\sc Model}
        & \multicolumn{3}{c}{\sc EN--FI (WMT24)}
        & \multicolumn{3}{c}{\sc EN--FI (FLORES200)}
        & \multicolumn{3}{c}{\sc EN--TR (WMT24)}
        & \multicolumn{3}{c}{\sc EN--TR (FLORES200)} \\
    \cmidrule(lr){2-4}
    \cmidrule(lr){5-7}
    \cmidrule(lr){8-10}
    \cmidrule(lr){11-13}
        & chrF++ & Kiwi & xCOM 
        & chrF++ & Kiwi & xCOM
        & chrF++ & Kiwi & xCOM
        & chrF++ & Kiwi & xCOM \\
    \midrule

    \multicolumn{13}{c}{\text{\textbf{Resource-Constrained LLM-based Translation Systems}}} \\
    \multicolumn{13}{@{}l}{\textcolor{lightgray}{\text{General-purpose LLMs}}} \\
    Qwen3-4B
        & 40.74 & 41.27 & 45.86 
        & 36.79 & 46.87 & 48.77 
        & 40.12 & 50.60 & 53.89 
        & 42.34 & 61.61 & 65.91  \\
    Qwen3-8B
        & 45.86 & 51.92 & 58.15 
        & 43.28 & 61.77 & 66.72 
        & 44.82 & 59.25 & 63.04 
        & 47.58 & 70.62 & 76.89  \\
    Qwen3-14B
        & 49.02 & 60.43 & 66.80 
        & 46.48 & 70.06 & 77.34 
        & 47.56 & 63.26 & 67.82 
        & 50.25 & 73.98 & 82.39  \\
    Qwen3-32B
        & 48.69 & 60.54 & 67.34 
        & 46.61 & 71.10 & 78.55 
        & 47.18 & 62.66 & 66.47 
        & 49.46 & 73.28 & 81.32  \\

    \multicolumn{13}{@{}l}{\textcolor{lightgray}{\text{MT-R1-Zero}}} \\
    MT-R1-Zero-4B
        & 43.42 & 56.04 & 61.34  
        & 40.49 & 65.20 & 69.41 
        & 43.25 & 61.57 & 63.85 
        & 45.22 & 72.70 & 78.14  \\
    MT-R1-Zero-8B
        & 47.45 & 62.16 & 69.79 
        & 44.51 & 72.42 & 78.78 
        & 47.10 & 63.96 & 68.98 
        & 48.72 & 75.92 & 83.53   \\

    \multicolumn{13}{@{}l}{\textcolor{lightgray}{\text{Ours}}} \\
    Ours-4B
        & 45.29 & 62.49 & 69.40  
        & 42.65 & 73.78 & 79.99  
        & 45.39 & 65.49 & 69.24  
        & 47.77 & 76.35 & 83.63   \\
    Ours-8B
        & \textbf{49.02 }& \textbf{67.90} & \textbf{76.49} 
        & \textbf{46.62} & \textbf{79.07} & \textbf{86.50}  
        & \textbf{48.04 }& \textbf{68.14} & \textbf{73.51}  
        & \textbf{50.41} & \textbf{78.26}& \textbf{87.25}  \\
    \midrule

\multicolumn{13}{c}{\textbf{LLM-based Translation Systems with Large Models or Extensive Data (only for reference)}} \\

    \multicolumn{13}{@{}l}{\textcolor{lightgray}{\text{General-purpose LLMs}}} \\
    Gemini-2.0-flash
    & 57.93 & 75.74  &  87.09
    &  58.09& 85.72 &   95.83
    & \textbf{57.42} & 68.05 & 77.65
    & 59.46 & 79.48 & 92.10\\
    
    OpenAI GPT-5.2
    & \textbf{59.44} & \textbf{76.26} & \textbf{87.83}
    & 59.56& \textbf{86.53} & \textbf{96.01}
    & 56.14 & \textbf{69.45}  &  \textbf{77.87}
    & 58.68 & \textbf{79.96} & \textbf{92.39}  \\

    DeepSeek-V3.2
        & 57.18 & 74.00  &  85.87
        & 56.53& 84.71 &  94.70
        & 56.21 & 68.13 &  77.84
        & 58.38 & 79.55&  91.70\\

    \multicolumn{13}{@{}l}{\textcolor{lightgray}{\text{Translation-specific LLMs}}} \\
    
    Seed-X-PPO-7B 
        & 57.48 & 74.72 & 86.51 
        & \textbf{62.57 }& 85.53 & 95.32 
        & 54.28 & 67.28 & 75.99 
        & \textbf{62.76} & 78.94 & 91.40  \\

    TowerInstruct-13B-v0.1
        & 44.58  &53.74  &  61.81
        & 43.96 &68.79  &  76.29
        & - & - &  -
        & - & - &  - \\

    \bottomrule
     \end{tabular}}
\caption{Results on translation directions (EN--FI and EN--TR). In the metric columns, \textbf{xCOM denotes xCOMET}. Models are grouped into resource-constrained LLM-based systems and large-scale or data-intensive LLM-based translation systems.
A dash (``--'') indicates that the model does not support the corresponding language direction.
MT-R1-Zero serves as the baseline, and both \emph{Ours} and MT-R1-Zero are trained with the same amount of data.
The best settings within each category are highlighted in \textbf{bold}.}

    \label{tab:main_mt_results}
\end{table*}

\subsubsection{Overall Reward and Advantage Computation}
\label{sec:overall-reward}

Let $x$ denote either a translation or a post-editing instance in our hybrid sampling step. Trajectories exceeding the token budget are penalized with $-1$. Valid trajectories receive task-specific rewards: 
\[
R(x)=
\begin{cases}
-1, & \text{if } x \text{ exceeds token budget},\\
R_{\text{pe}}(x), & \text{if } x = (\textit{src}, \textit{pred}_i, \textit{pe}_{i,j}, \textit{tgt}),\\
R_{\text{mt}}(x), & \text{if } x = (\textit{src}, \textit{pred}_i, \textit{tgt}).
\end{cases}
\]
After reward computation, the translation trajectories $\{\textit{pred}_i\}_{i=1}^{N}$ form a single GRPO group for advantage computation.  
The post-editing trajectories $\{\textit{pe}_{i,j}\}_{i=1,j=1}^{N,M}$ are divided into $N$ GRPO groups, each consisting of $\{\textit{pe}_{i,j}\}_{j=1}^{M}$ with independently computed advantages.  
All advantages are then used to optimize the policy via the GRPO policy gradient.

\subsection{Variance-Aware Gradient Weighting}
\label{sec:problem_weight}

As discussed in Section~\ref{sec:variance_analysis}, conditioning on a fixed draft trajectory $\tau_0$ yields a lower-variance estimator of the expected post-editing return, compared to a translation-level baseline that marginalizes over $\tau_0$.
As a consequence, the post-editing term in the policy gradient is associated with a more stable learning signal, while the translation-level term involves additional variability due to uncertainty over $\tau_0$.

Motivated by this discrepancy in the variance of their underlying return estimates and the different roles played by the two gradient terms, we introduce weighting coefficients to explicitly balance their relative contributions in Eq.~\eqref{eq:master_obj}.
This leads to a biased estimator, but allows for improved stability during optimization:
\begin{align}
	&
	\mathbb{E}_{\tau_0}\Big[
	\lambda_{\text{pe}}\,\mathbb{E}_{\tau_1}
	\big[
	\nabla_\theta \log \pi_\theta(\tau_1 \mid p,\tau_0)\, R(\tau_1)
	\big]
	\Big]
	\nonumber\\
	&\quad
	+ \lambda_{\text{mt}}\,
	\mathbb{E}_{\tau_0}
	\big[
	\nabla_\theta \log \pi_\theta(\tau_0 \mid q)\,
	\mathbb{E}_{\tau_1}[R(\tau_1)]
	\big].
\end{align}

In our main experiments, we set $\lambda_{\text{pe}} = M$ and $\lambda_{\text{mt}} = 1$,
placing greater emphasis on the post-editing signal, whose baseline is more directly aligned with the optimized return.
The effects of different $\lambda_{\text{pe}}$ and $\lambda_{\text{mt}}$ settings are further analyzed in Section~\ref{sec:grad_analysis}.

\section{Experiments}

\subsection{Experimental Setup}
\label{sec:experiments_setup}

\paragraph{Datasets.}
Following the capabilities of the base models and their relative coverage over different language pairs, we conduct experiments on two categories of translation directions:

\begin{itemize}
\setlength{\itemsep}{0.3em}
\setlength{\parsep}{0em}
\setlength{\parskip}{0em}
\setlength{\topsep}{0em}
    \item \textbf{Less-Covered Directions.}
    We conduct experiments with Qwen3-(4B, 8B)~\citep{yang2025qwen3technicalreport} on English$\rightarrow$Finnish (EN$\rightarrow$FI) and English$\rightarrow$Turkish (EN$\rightarrow$TR).
    For EN$\rightarrow$FI, 7K sentence pairs are sampled from the validation and test sets of WMT17--19~\citep{bojar-EtAl:2017:WMT1,bojar-EtAl:2018:WMT1,wmt19translate}, while for EN$\rightarrow$TR, 6K sentence pairs are sampled from the WMT17--18 test sets. For these language directions, the function $f(\cdot)$ in Eq.~\eqref{eq:metric} is defined as the sum of COMETKiwi and chrF++.
    
    \item \textbf{More-Covered Directions.}
    We conduct experiments with the smaller Qwen3-0.6B~\citep{yang2025qwen3technicalreport} on English$\leftrightarrow$Chinese (EN$\leftrightarrow$ZH), where the base model exhibits substantially stronger prior competence.
    The bidirectional parallel data are collected following prior work~\citep{feng2025mtr1zeroadvancingllmbasedmachine}. For these language directions, the function $f(\cdot)$ in Eq.~\eqref{eq:metric} is defined as the sum of COMETKIWI and BLEU.
\end{itemize}

Across all language pairs, evaluation is conducted on the WMT24 test sets~\citep{deutsch2025wmt24expandinglanguagecoverage} and the FLORES-200 benchmark~\citep{costa2022no}. In addition, for EN$\leftrightarrow$ZH, we further report results on a more challenging challenge set collected in prior work~\citep{cheng2025seedxbuildingstrongmultilingual}.

\paragraph{Baselines.}
Our baselines are grouped into two categories.
One category comprises advanced LLM-based translation systems characterized by large model sizes ($\ge$100B parameters) and/or extensive training data, including general-purpose LLMs such as Gemini-2.0-Flash,\footnote{https://ai.google.dev/gemini-api/docs/models} OpenAI GPT-5.2,\footnote{https://platform.openai.com/docs/models/gpt-5.2} and DeepSeek-V3.2~\citep{deepseekai2025deepseekv32pushingfrontieropen}, as well as translation-specialized models Seed-X-PPO-7B~\citep{cheng2025seedxbuildingstrongmultilingual} and TowerInstruct-13B-v0.1~\citep{tower_llm_2024}.
The other category targets resource-constrained settings and includes the Qwen3 family of general-purpose models and our primary comparison method, MT-R1-Zero.
Unlike our hybrid trajectory design that interleaves translation and post-editing, MT-R1-Zero samples trajectories only at the translation stage.
To control variables, we use the same prompts as in MT-R1-Zero and compute its translation quality using the post-editing reward ($R_{\text{pe}}$), reporting results under the \textbf{non-thinking} setting.

\paragraph{Evaluation Metrics.}
We evaluate translation quality along both surface-form and semantic dimensions.
For surface-level evaluation, we use chrF++~\citep{popovic2017chrf++} for Finnish and Turkish, which exhibit rich morphological variation, and BLEU~\citep{post-2018-call} for English and Chinese, where BLEU is well established.
For semantic evaluation, we adopt cost-effective COMET-style models: COMETkiwi~\citep{rei2023scalingcometkiwiunbabelist2023} as a reference-free metric and XCOMET~\citep{guerreiro2023xcomettransparentmachinetranslation} as a reference-based metric.
Both metrics are used in their XL variants.

\paragraph{Training Details.}
We adopt VeRL~\citep{sheng2024hybridflow} as the RL training framework.
During training, the input prompt length is capped at 768 tokens, and the maximum output length is set to 512 tokens.
Gradients are computed with an effective batch size of 128 samples per step using gradient accumulation, and the learning rate is set to $5\times10^{-7}$.

For GRPO sampling, our approach rolls out 8 translation candidates per input and further rolls out 8 post-editing outputs for each translation, resulting in 72 trajectories per data instance.
\textbf{Accordingly, all compared methods are trained with 72 rollouts per example to ensure a fair comparison.}

Main experiments are conducted on 1 × 8 NVIDIA A100 GPUs (80GB) and 4 × 8 NVIDIA H20 GPUs (96GB). Training for a single language direction takes approximately 24 hours, requiring around 400 training steps.

We further extend our experiments to the Ascend platform by training two 4B-scale models on Ascend 910 GPUs. Detailed implementation and experimental settings on Ascend are provided in Section~\ref{sec:plat} and Appendix~\ref{app:ascend}.
\begin{table*}[htbp]
\centering
\small
\setlength{\tabcolsep}{4pt}
\resizebox{\textwidth}{!}{%
\begin{tabular}{
    l
    *{3}{c} @{\hspace{6pt}}
    *{3}{c} @{\hspace{6pt}}
    *{3}{c} @{\hspace{10pt}}
    *{3}{c} @{\hspace{6pt}}
    *{3}{c} @{\hspace{6pt}}
    *{3}{c}
}
\toprule
\multirow{1}{*}{\sc Model}
& \multicolumn{3}{c}{\sc EN--ZH (WMT24)}
& \multicolumn{3}{c}{\sc EN--ZH (FLORES)}
& \multicolumn{3}{c}{\sc EN--ZH (Challenge)}
& \multicolumn{3}{c}{\sc ZH--EN (WMT24)}
& \multicolumn{3}{c}{\sc ZH--EN (FLORES)}
& \multicolumn{3}{c}{\sc ZH--EN (Challenge)} \\
\cmidrule(lr){2-4}
\cmidrule(lr){5-7}
\cmidrule(lr){8-10}
\cmidrule(lr){11-13}
\cmidrule(lr){14-16}
\cmidrule(lr){17-19}
& {\scriptsize BLEU} & {\scriptsize Kiwi} & {\scriptsize xCOM}
& {\scriptsize BLEU} & {\scriptsize Kiwi} & {\scriptsize xCOM}
& {\scriptsize BLEU} & {\scriptsize Kiwi} & {\scriptsize xCOM}
& {\scriptsize BLEU} & {\scriptsize Kiwi} & {\scriptsize xCOM}
& {\scriptsize BLEU} & {\scriptsize Kiwi} & {\scriptsize xCOM}
& {\scriptsize BLEU} & {\scriptsize Kiwi} & {\scriptsize xCOM} \\
\midrule

Qwen3-0.6B
& 26.20 & 58.57 & 64.45
& 30.25 & 70.54 & 77.18
& 21.67 & 64.33 & 63.90
& 15.00 & 63.87 & 75.62
& 19.32 & 72.85 & 88.34
& 15.52 & 58.83 & 62.91 \\

MT-R1-Zero
& 28.23 & 62.96 & 67.16
& 33.24 & 73.78 & 79.83
& 23.00 & 66.89 & 65.28
& 15.97 & \textbf{66.86} & 77.74
& 19.66 & 74.91 & 89.69
& 16.88 & 61.56 & 63.41 \\

Ours
& \textbf{29.23} & \textbf{64.63} & \textbf{68.40}
& \textbf{34.03} & \textbf{74.39} & \textbf{80.89}
& \textbf{24.44} & \textbf{68.89} & \textbf{67.00}
& \textbf{16.26} & 66.69 & \textbf{78.28}
& \textbf{20.68} & \textbf{75.49} & \textbf{90.48}
& \textbf{17.16} & \textbf{62.34} & \textbf{64.63} \\

\bottomrule
\end{tabular}}
\caption{Results on translation directions (EN$\leftrightarrow$ZH). In the metric columns, \textbf{xCOM denotes xCOMET}. Our method consistently outperforms baselines across different language directions and datasets. The best results are highlighted in \textbf{bold}.}
\label{tab:main_ze}
\end{table*}
\begin{figure*}[t]
    \centering
    \setlength{\tabcolsep}{0pt}
    \subfigure{
        \includegraphics[width=0.23\textwidth]{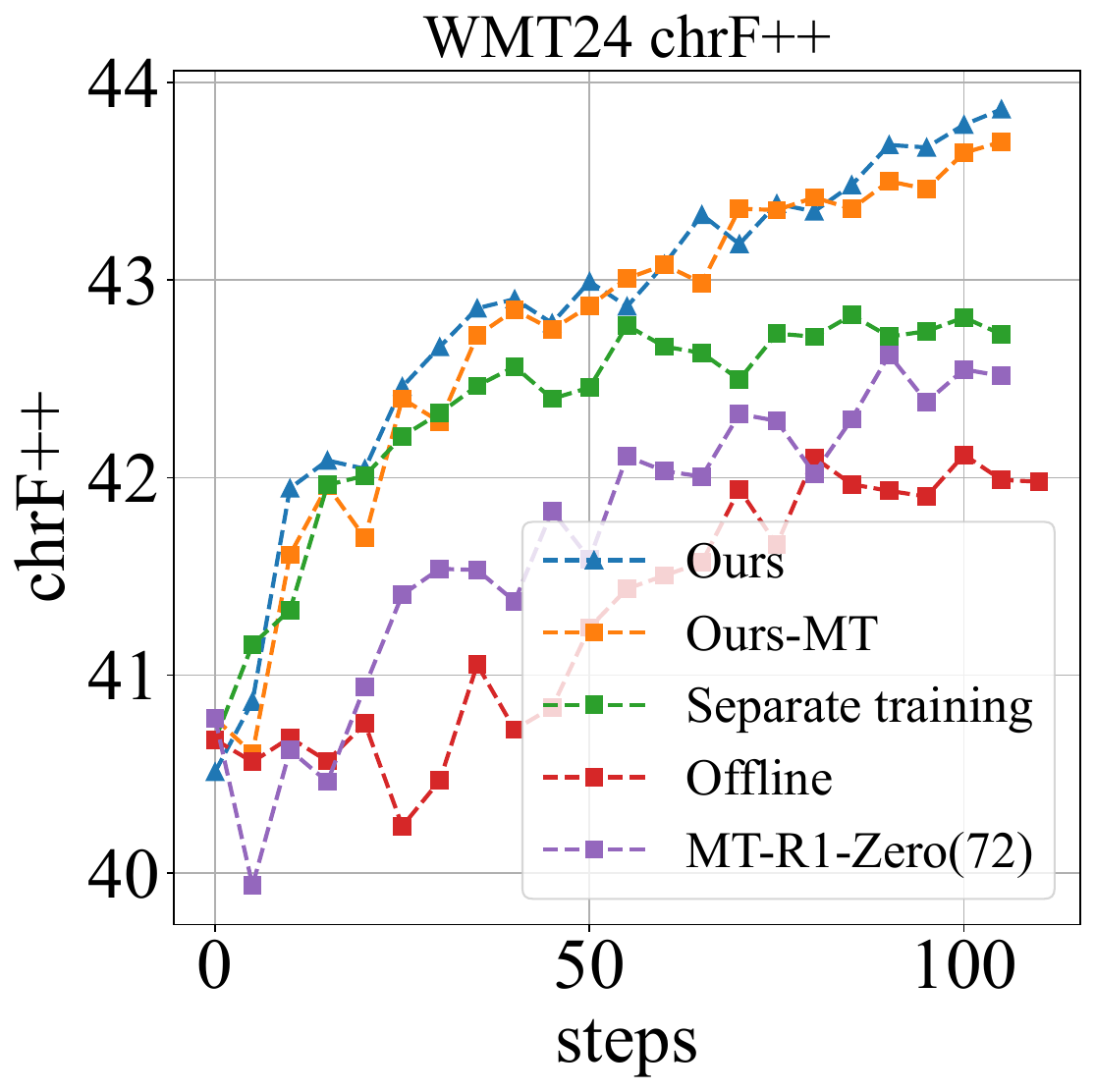}
    }
    \hfill
    \subfigure{
        \includegraphics[width=0.23\textwidth]{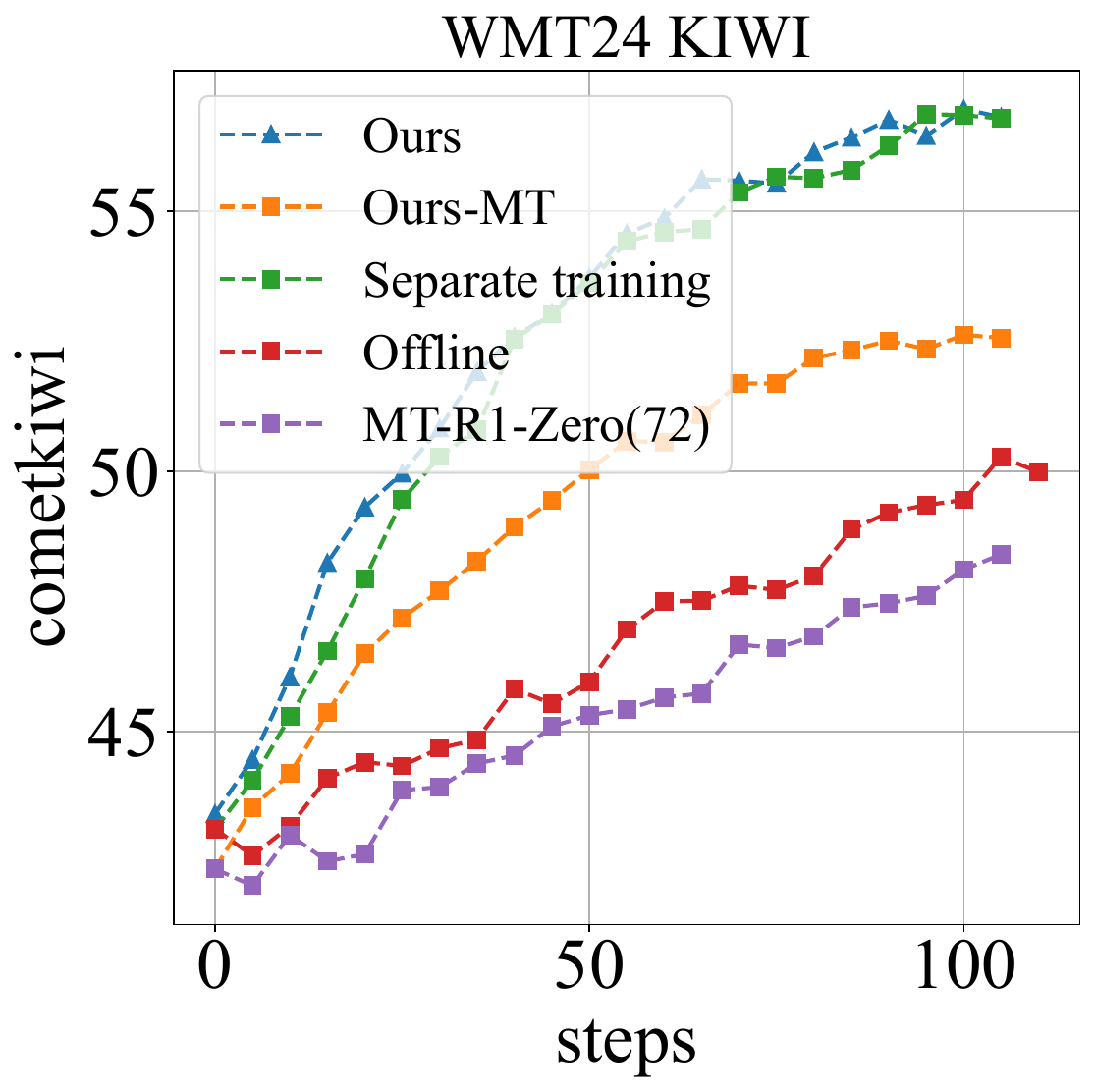}
    }
    \hfill
    \subfigure{
        \includegraphics[width=0.23\textwidth]{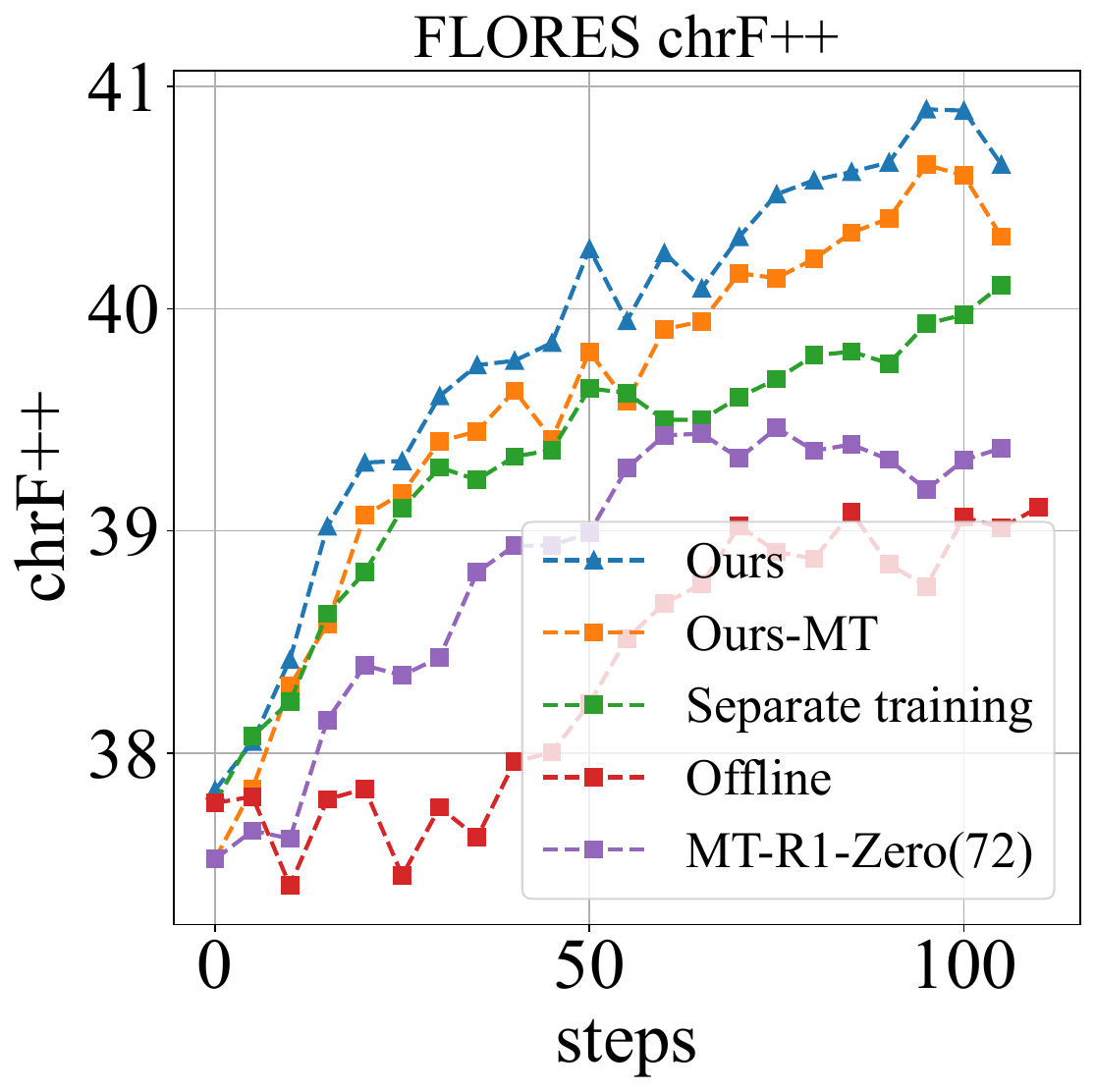}
    }
    \hfill
    \subfigure{
        \includegraphics[width=0.23\textwidth]{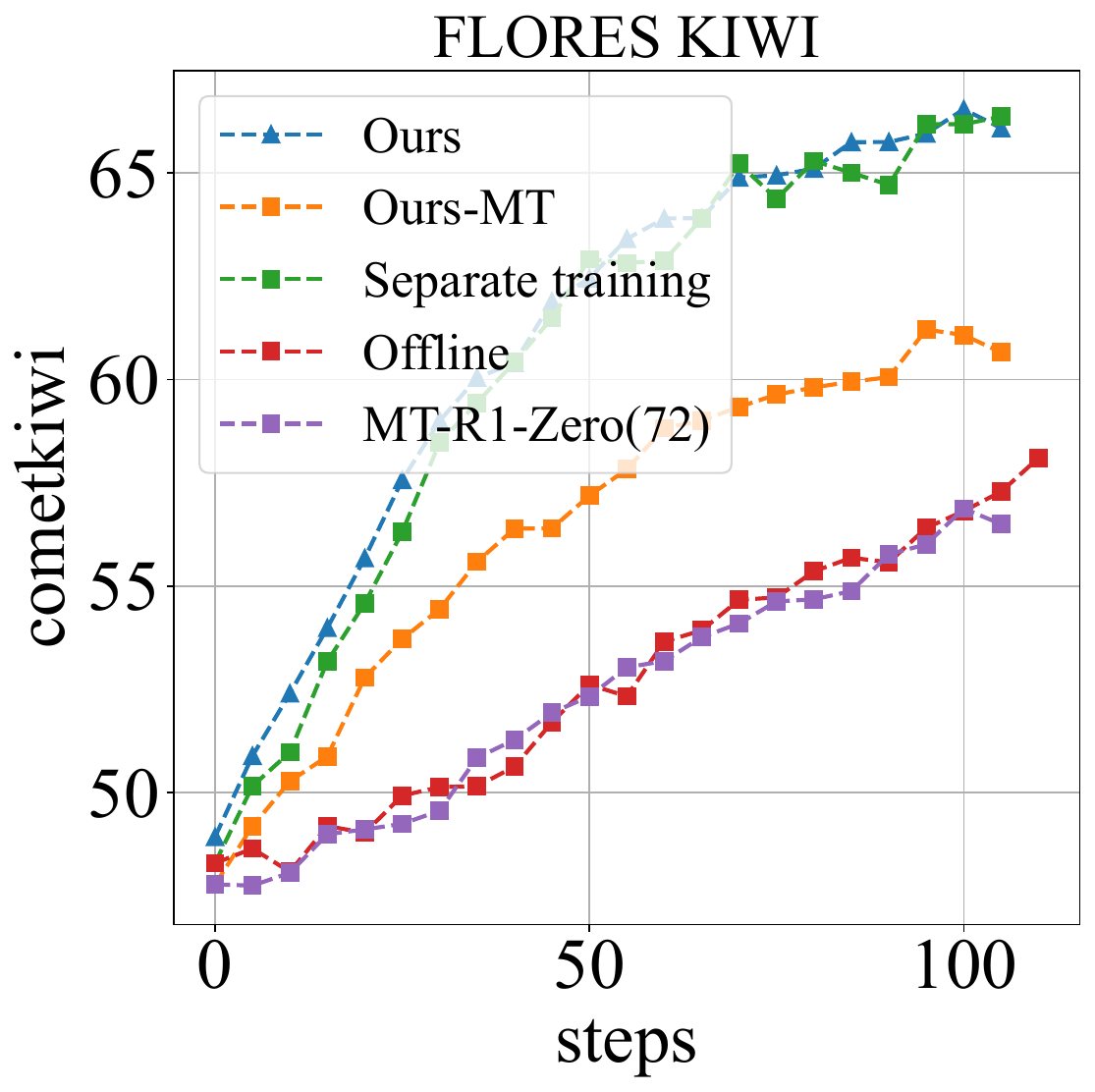}
    }
\caption{Ablation study of our framework components on WMT24 (EN$\rightarrow$FI) and FLORES200 (EN$\rightarrow$FI), evaluated using chrF++ and COMET-KIWI. All experiments are conducted on 1K EN$\rightarrow$FI translation instances sampled from the training set. In the offline setting, an additional 7K post-editing instances are used. Models are trained for 15 epochs; at each training step, 72 trajectories are sampled per instance, and evaluation is performed every 5 steps.}

    \label{fig:reward-ablation}
\end{figure*}

\subsection{Main Results}

\textbf{Our method outperforms pure GRPO under resource constraints.}  
As Table~\ref{tab:main_mt_results} shows, Ours-8B surpasses Qwen3-32B on EN$\rightarrow$FI, achieving COMET-KIWI gains of +7.36 (WMT24) and +7.97 (FLORES), with even larger improvements on XCOMET: +9.15 (WMT24) and +7.95 (FLORES). For EN$\rightarrow$TR, we observe consistent gains of approximately 5--6 points on COMET-KIWI and around 7 points on XCOMET. Ours-4B also outperforms Qwen3-32B on both COMET-KIWI and XCOMET.

Compared to MT-R1-Zero, our approach delivers larger improvements using the same base models. On EN$\rightarrow$FI (WMT24), Ours-4B improves XCOMET by +23.54, compared to +15.48 for MT-R1-Zero-4B, while Ours-8B achieves +18.34 versus +11.64 for MT-R1-Zero-8B. On EN$\leftrightarrow$ZH (Table~\ref{tab:main_ze}), our method consistently outperforms MT-R1-Zero across most metrics, with only a slight drop on COMET-KIWI for ZH$\rightarrow$EN.

\textbf{Our method achieves strong semantic gains with limited resources.}  
Table~\ref{tab:main_mt_results} shows that Ours-8B approaches state-of-the-art COMET-KIWI performance on EN$\rightarrow$TR, closely matching DeepSeek-V3.2 on WMT24 (68.14 vs.\ 68.13) and FLORES (78.26 vs.\ 79.55), despite being trained on only 6K examples with an 8B model, demonstrating the effectiveness of our framework.

\begin{table*}[t]
\footnotesize
\centering
\setlength{\tabcolsep}{6pt}
\begin{tabularx}{\linewidth}{@{}l X r r r r @{}}
\toprule
 &  & \textbf{chrF++} & \textbf{KIWI} & \textbf{SUM} &
$\mathbf{Avg_{pe}}$ \\
\midrule

\textbf{Source} &
``She had a real fear of food waste,'' Mr. Coe said.
& -- & -- & -- & -- \\

\midrule

\textbf{Reference} &
``Hän todellakin pelkäsi ruoan tuhlaamista,'' Coe sanoi.
& -- & -- & -- & -- \\

\midrule

\textbf{Base (T1)} &
\scalebox{0.95}{
{
``Hänellä oli todellinen järkytys ruoan hukkautumisesta,'' Coe hakeutui.
}}
& 0.25 & 0.4775 & 0.7321 & 0.7504 \\
& \multicolumn{5}{l}{\scalebox{0.85}{\textcolor{orange!70!black}{“She had a genuine shock about the causing of food to drown,” Coe hakeutui{\textcolor{red}{($\times$: to apply, to seek, to make one’s way)}}.}}} \\

\midrule

\textbf{Base (T2)} &
\scalebox{0.95}{
{
``Hänellä oli varsin vakava huuhtola ruoasta,'' sanoi herra Coe.
}}
& 0.23 & 0.2132 & 0.4407 & 0.4697 \\
& \multicolumn{5}{l}{\scalebox{0.85}{\textcolor{orange!70!black}{ “She had a rather serious huuhtola{\textcolor{red}{($\times$: possibly huuhtoutuminen ‘wash-away / leaching’)}} about food,” said Mr. Coe.}}} \\

\midrule

\textbf{M-Z (105 s, T1)} &
\scalebox{0.95}{
"Hänellä oli oltu todellinen huolia ruoan hajoamisesta", herra Coe sanoi.
}
& 0.38 & 0.5243 & 0.9004 & 1.0608 \\
& \multicolumn{4}{l}{\scalebox{0.85}{\textcolor{orange!70!black}{“She had been had real worries about the decomposition of food,” Mr. Coe said.}}}   \\

\midrule

\textbf{M-Z (105 s, T2)} &
\scalebox{0.95}{
"Hänellä oli todellinen korko ruoan hukkumisesta", herra Coe sanoi.
}
& 0.37 & 0.3219 & 0.6929 & 0.9250 \\
& \multicolumn{4}{l}{\scalebox{0.85}{\textcolor{orange!70!black}{“She had a genuine korko{\textcolor{red}{($\times$: interest rate / heel)}} about the drowning of food,” Mr. Coe said.}}}   \\

\midrule

\textbf{Ours (105 s, T1)} &
\scalebox{0.95}{
``Hänellä oli todellinen huoli ruoan häviöstä'', Coe sanoi.
}
& 0.40 & 0.8849 & 1.2826 & 1.2287 \\
& \multicolumn{4}{l}{\scalebox{0.85}{\textcolor{orange!70!black}{“She had a genuine concern about the loss of food,” Coe said.}}}   \\
$\hookrightarrow$\textbf{post-edit(T1):} &
\scalebox{0.95}{
``Hänellä oli todellinen huoli ruoan häviöstä'', Coe sanoi.
}
& 0.40 & 0.8849 & 1.2826 & -- \\
& \multicolumn{4}{l}{\scalebox{0.85}{\textcolor{orange!70!black}{ “She had a genuine concern about the loss of food,” Coe said. }}} \\
$\hookrightarrow$\textbf{post-edit(T1):} &
\scalebox{0.95}{
``Hänellä oli todellinen huoli ruoan käyttöstä'', Coe sanoi.
}
& 0.40 & 0.4765 & 0.8806 & -- \\
& \multicolumn{4}{l}{\scalebox{0.85}{\textcolor{orange!70!black}{ “She had a genuine concern about food käyttöstä{\textcolor{red}{($\times$: usage / use / utilization)}},” Coe said.}}} \\
\midrule

\textbf{Ours (105 s, T2)} &
\scalebox{0.95}{
``Hänellä oli todellinen huoli ruokaan menettymästä,'' Coe sanoi.
}
& 0.33 & 0.5793 & {0.9083} & {1.1774} \\
& \multicolumn{4}{l}{\scalebox{0.85}{\textcolor{orange!70!black}{“She had a genuine concern about ruokaan menettymästä{\textcolor{red}{($\times$: the loss of food)}},” Coe said.}}} \\
$\hookrightarrow$\textbf{post-edit(T2):} &
\scalebox{0.95}{
``Hänellä oli todellinen huoli ruoan häviämisestä,'' Coe sanoi.
}
& 0.36 & 0.8591 & 1.2229 & -- \\
& \multicolumn{4}{l}{\scalebox{0.85}{\textcolor{orange!70!black}{ “She had a genuine concern about food disappearing,” Coe said.}}} \\
$\hookrightarrow$\textbf{post-edit(T2):} &
\scalebox{0.95}{
``Hänellä oli todellinen huoli ruoan menetystä,'' Coe sanoi.
}
& 0.36 & 0.7472 & 1.1067 & -- \\
& \multicolumn{4}{l}{\scalebox{0.85}{\textcolor{orange!70!black}{“He had a genuine concern about food menetystä{\textcolor{red}{($\times$: loss / losing)}},” Coe said.}}} \\
\bottomrule
\end{tabularx}

\caption{Case study of model generation behavior. \textbf{Base} (T1/T2) denotes two translation trajectories sampled from the base model. \textbf{M-Z} (105s, T1/T2) refers to two trajectories produced by MT-R1-Zero after 105 training steps, while \textbf{Ours} (105s, T1/T2) are generated by our method. Each trajectory is followed by its \textbf{post-editing variants} ($\hookrightarrow$ post-edit). We analyze one training-set example using MT-R1-Zero and our 105-step checkpoint, selecting two representative trajectories from eight sampled translations. Scores are chrF++ and COMETKIWI (\textsc{Sum}); $\mathrm{Avg}_{\mathrm{pe}}$ denotes the average over post-edits. English translations are shown beneath each Finnish output. Misspelled Finnish words are left untranslated and annotated as \textcolor{red}{($\times$: text)}, where \emph{text} indicates the intended meaning (e.g., \emph{menetystä} \textcolor{red}{($\times$: loss / losing)}, denotes a misspelled form of a word meaning ``loss'' or ``losing''.).
}

\label{tab:case_study}
\end{table*}

\section{Analysis}

\subsection{Hybrid Sampling and Reward Analysis}
\label{sec:ablation}

This subsection examines the contribution of each component under different settings.

\begin{itemize}
\setlength{\itemsep}{0.3em}     
\setlength{\parsep}{0em}      
\setlength{\parskip}{0em}     
\setlength{\topsep}{0em} 
    \item \textbf{Ours}: Post-editing is trained with online-generated data. The translation trajectories are optimized using rewards derived from post-editing feedback, while the post-editing trajectories are optimized with $R_{\text{pe}}(x)$.

    \item \textbf{Ours-MT}: Trained the same with \textbf{Ours}. Evaluation is performed using only the first-stage draft translations, without applying post-editing.
    
    \item \textbf{Separate training}: Post-editing relies solely on online-generated data. Unlike \textbf{Ours}, the translation stage is trained only with the sum of COMETKIWI and chrF++.

    \item \textbf{Offline}: Post-editing is trained on static, pre-collected data, and both translation and post-editing models optimize only sum of COMETKIWI and chrF++. 

    \item \textbf{MT-R1-Zero(72)}: Used for comparison with \textbf{Ours-MT}, where the number 72 indicates that it uses 72 translation rollouts for gradient updates.

\end{itemize}
\textbf{Online generation of post-editing data is effective.}
As shown in Figure~\ref{fig:reward-ablation}, the \emph{Separate training} setting significantly outperforms its offline counterpart on the \text{COMETkiwi} metric, and also achieves a marginal improvement on \text{chrF++}.
This indicates that our framework does \textbf{not} simply optimize two independent tasks.

\textbf{Stage-1 translation reward aligns better with final post-edited quality.}  
Compared with \emph{Separate training}, our method differs only in the Stage~1 reward, defined as $\mathbb{E}_{\tau_1}[R(\tau_1)]$, which accounts for downstream post-editing. This yields an ~1-point improvement in chrF++ on the final outputs, while COMETKIWI remains comparable.

\textbf{Despite a smaller token budget, first-stage drafts from our framework outperform MT-R1-Zero.}
As shown in Figure~\ref{fig:reward-ablation}, \emph{Ours-MT} outperforms \emph{MT-R1-Zero(72)} on chrF++ and COMET-KIWI. Although each sample yields 8 translation and 64 post-editing trajectories, only the 8 drafts contribute to the policy gradient, compared to 72 translation trajectories in MT-R1-Zero. This indicates that post-editing enables fine-grained local exploration that guides translation toward higher-quality regions and indirectly promotes global exploration.

\begin{figure}[htbp]
    \centering
    \begin{minipage}[b]{0.23\textwidth}
        \includegraphics[width=\textwidth]{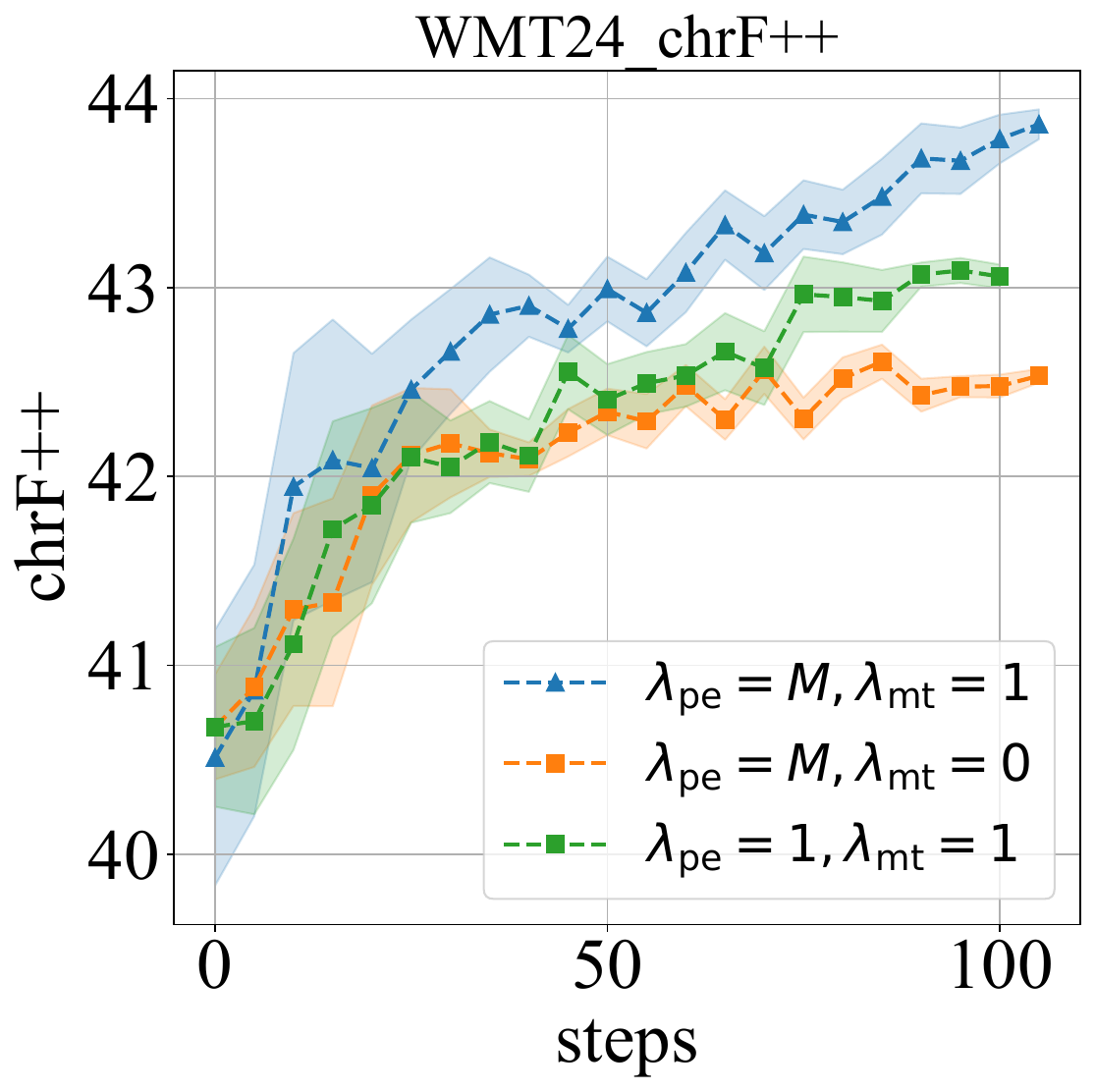}
        \centering
    \end{minipage}
    \hfill
    \begin{minipage}[b]{0.23\textwidth}
        \includegraphics[width=\textwidth]{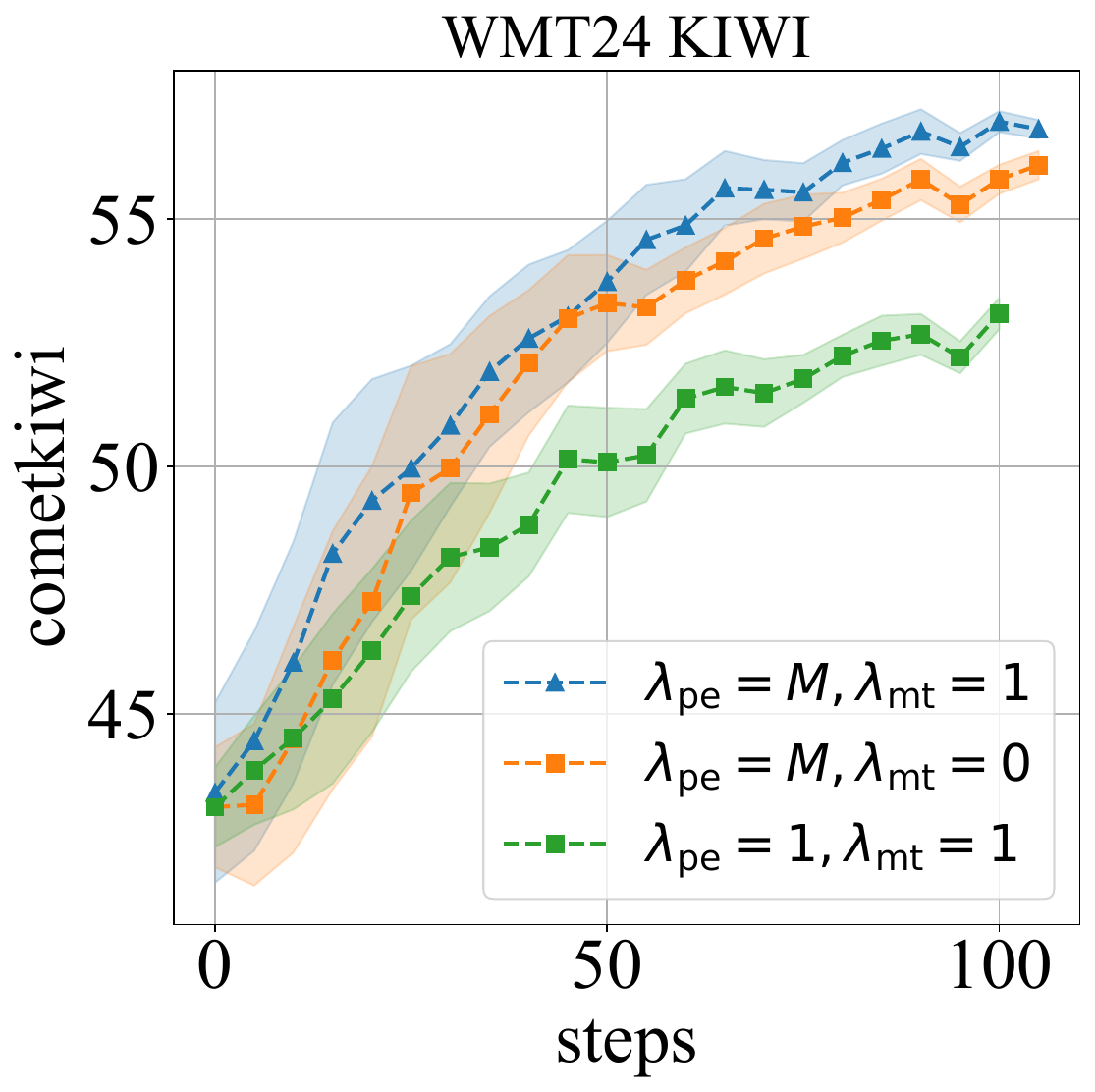}
        \centering
    \end{minipage}
\caption{Gradient Weight Analysis. Experimental settings are identical to those in Section~\ref{sec:ablation}.}

    \label{fig:grad_fig}
\end{figure}

\subsection{Gradient Weight Analysis}
\label{sec:grad_analysis}

As discussed in Section~\ref{sec:problem_weight}, the post-editing and translation gradient terms differ in their noise characteristics due to the variance of their underlying return estimators.
In this subsection, we analyze the impact of the scaling factors $\lambda_{\text{pe}}$ (for the post-editing policy gradient) and $\lambda_{\text{mt}}$ (for the translation policy gradient), which control the relative contributions of these two learning signals.

We consider the following experimental settings:
\begin{itemize}
	\setlength{\itemsep}{0.3em}
	\setlength{\parsep}{0em}
	\setlength{\parskip}{0em}
	\setlength{\topsep}{0em}
	\item $\lambda_{\text{pe}} = M$, $\lambda_{\text{mt}} = 1$: Places greater emphasis on the post-editing signal, whose baseline provides a more stable estimate of the optimized return, while keeping the number of trajectories balanced per step.
	\item $\lambda_{\text{pe}} = 1$, $\lambda_{\text{mt}} = 1$: Treats the two gradient terms equally, yielding an unbiased estimator but with increased sensitivity to noise from the translation-level return estimation.
	\item $\lambda_{\text{pe}} = M$, $\lambda_{\text{mt}} = 0$: Removes the translation-level term entirely, isolating its contribution to overall performance.
\end{itemize}

Figure~\ref{fig:grad_fig} and Table~\ref{tab:grad_tab} show that $\lambda_{\text{pe}} = M$ and $\lambda_{\text{mt}} = 1$ consistently achieve the best performance on WMT24, yielding the largest gains in chrF++ and improved COMET-KIWI.
Accordingly, we adopt this configuration as the default setting in all subsequent experiments.

\subsection{LLM-based Evaluation}

To assess whether the observed improvements extend beyond standard automatic metrics, we complement our evaluation with an LLM-as-a-judge analysis against MT-R1-Zero. We employ two independent LLM judges, both of which consistently prefer our method on the test set. 

We use the following experimental setup:
\begin{itemize}
	\setlength{\itemsep}{0.3em}
	\setlength{\parsep}{0em}
	\setlength{\parskip}{0em}
	\setlength{\topsep}{0em}
    \item \textit{Models.} We compare Ours-8B vs.\ MT-R1-Zero-8B on en$\to$fi and en$\to$tr, and Ours-0.6B vs.\ MT-R1-Zero-0.6B on en$\leftrightarrow$zh. 
    \item \textit{Data \& Judges.} Evaluation is conducted on WMT24, using two LLM judges: GPT-5.2 and Gemini-3-Pro-Preview.  
    \item \textit{Prompt Design.} To reduce positional bias, we randomly swap the order of the two candidate translations with equal probability.
\end{itemize}

As shown in Table~\ref{tab:llm_judge}, LLM preference results consistently favor our method over MT-R1-Zero across all directions and judges, in line with our main findings. While not fully ruling out metric-specific effects, the consistent preferences from two independent LLM judges, together with improvements on X-COMET—an evaluation metric not optimized during training—provide evidence that our method achieves stronger performance than the baseline across diverse evaluation signals, rather than merely better alignment with specific metrics.

\begin{table}[t]
\centering
\small
\setlength{\tabcolsep}{5pt}
\renewcommand{\arraystretch}{1.2}
\definecolor{tigray}{RGB}{160,160,160}
\begin{tabular}{@{}ll rr@{}}
\toprule
\textbf{Dir.} & \textbf{Method} & \textbf{Gemini} & \textbf{GPT} \\
\midrule

\multirow{3}{*}{en$\to$fi}
 & \textbf{Ours}           & \textbf{540/998} & \textbf{627/998} \\
 & MT-R1-Zero              & 310/998          & 361/998          \\
 & \textcolor{tigray}{T/I} & \textcolor{tigray}{148/998} & \textcolor{tigray}{10/998} \\
\midrule

\multirow{3}{*}{en$\to$tr}
 & \textbf{Ours}           & \textbf{570/998} & \textbf{592/998} \\
 & MT-R1-Zero              & 364/998          & 388/998          \\
 & \textcolor{tigray}{T/I} & \textcolor{tigray}{64/998}  & \textcolor{tigray}{18/998} \\
\midrule

\multirow{3}{*}{en$\to$zh}
 & \textbf{Ours}           & \textbf{495/998} & \textbf{540/998} \\
 & MT-R1-Zero              & 394/998          & 449/998          \\
 & \textcolor{tigray}{T/I} & \textcolor{tigray}{109/998} & \textcolor{tigray}{9/998}  \\
\midrule

\multirow{3}{*}{zh$\to$en}
 & \textbf{Ours}           & \textbf{459/998} & \textbf{517/998} \\
 & MT-R1-Zero              & 453/998          & 468/998          \\
 & \textcolor{tigray}{T/I} & \textcolor{tigray}{86/998}  & \textcolor{tigray}{13/998} \\

\bottomrule
\end{tabular}
\caption{LLM-based pairwise evaluation. Each entry is reported as $a/b$, where $a$ denotes the number of LLM-preferred samples and $b$ the total number of evaluated samples; ties or invalid judgments (T/I) may occur.}
\label{tab:llm_judge}
\end{table}

\subsection{Case Study}

\textbf{Base translation explores broadly.}  
Table~\ref{tab:case_study} illustrates two sampled translation trajectories (T1, T2). The base model generates outputs differing in lexical choice and structure, reflecting broad but unstable exploration. 

\textbf{Compared to MT-R1-Zero, our method yields draft translations that are semantically closer to the source and achieves higher average quality after post-editing.}
Ours (T1/T2) correctly captures the meaning of \emph{food waste} at the draft stage and achieves higher average final output quality (1.2287/1.1774) than MT-R1-Zero (1.0608/0.9250).

\subsection{Platform Generalization}
\label{sec:plat}
To improve the accessibility and reproducibility of our method, we further implement PEGRL on Ascend hardware, with full support for end-to-end training.

We compare the results obtained on the Ascend platform (Atlas 800I A3) at 400 training steps with those from our main experiments, and observe comparable performance, suggesting that our approach generalizes well across different hardware environments. Detailed comparisons are provided in Appendix~\ref{app:ascend}.

\section{Conclusion}

We present a two-stage RL framework for machine translation, which models translation and post-editing as sequential actions and enables both global and local RL exploration. 
By exploiting more stable learning signals derived from conditional return estimation in the post-editing stage, our framework supports more stable policy optimization. 
Furthermore, a task-specific weighting scheme balances the contributions of translation and post-editing objectives, improving sample efficiency under a fixed token budget. 
Our results highlight the importance of accounting for variance in return estimation when designing RL objectives, which may be critical for more complex tasks.


\section{Limitations}

While our framework demonstrates strong performance in translation experiments, its theoretical foundation relies on a task with a relatively small effective sampling space. We have only verified that post-editing can stabilize learning and improve convergence for translation; it remains unclear whether similar auxiliary tasks exist or provide comparable benefits in other domains, such as verifiable-reward tasks, mathematical reasoning, or code generation. Additionally, the reward density of auxiliary tasks in these domains may differ from translation, potentially limiting their impact. In terms of performance, our method still falls short of state-of-the-art LLM-based translation systems, particularly on surface-level metrics, as post-editing often involves minimal changes that are difficult to capture with such metrics. Moreover, due to limited resources, our experiments are restricted to low-resource scenarios and small models; the behavior in high-resource settings remains unexplored.

\section*{Acknowledgments}
We would like to thank the anonymous reviewers for their insightful comments. Shujian Huang is the corresponding author. This work is supported by National Science Foundation of China (No. 62376116), research project of Nanjing University-China Mobile Joint Institute (NJ20250038), the Fundamental Research Funds for the Central Universities (No. 2024300507), Fundamental and Interdisciplinary Disciplines Breakthrough Plan of the Ministry of Education of China (No. JYB2025XDXM118).

\bibliography{custom.bib}
\appendix

\section{Policy Gradient Derivation}
\label{app:pg}

We provide the detailed derivation of the policy gradient used in the main text.
We first review the log-derivative trick and then apply it to our two-stage trajectory objective.

\subsection{Log-Derivative Trick}
\label{app:logderiv}

For a parameterized distribution $p_\theta(x)$ and a scalar function $f(x)$, the gradient of the expectation can be written as:
\begin{align}
\nabla_\theta \mathbb{E}_{x \sim p_\theta}[f(x)]
&= \nabla_\theta \int p_\theta(x)\, f(x)\, dx \notag\\
&= \int \nabla_\theta p_\theta(x)\, f(x)\, dx \notag \\
&= \int p_\theta(x)\, \nabla_\theta \log p_\theta(x)\, f(x)\, dx \notag \\
&= \mathbb{E}_{x \sim p_\theta}
\big[
\nabla_\theta \log p_\theta(x)\, f(x)
\big].\notag
\end{align}

\subsection{Two-Stage Trajectory Objective}
\label{app:objective}



\section{Variance Analysis of Monte Carlo Estimators}
\label{sec:appendix_variance}

\subsection{Variance of Monte Carlo Estimation}
\label{sec:appendix_mc}
Let $Z \sim P$ and $\mu = \mathbb{E}_{Z \sim P}[f(Z)]$.
Given $N$ i.i.d.\ samples $\{Z_i\}_{i=1}^N$, the Monte Carlo estimator
\begin{equation}
\hat{\mu}_N = \frac{1}{N} \sum_{i=1}^N f(Z_i)
\end{equation}
has variance
\begin{equation}
\mathrm{Var}(\hat{\mu}_N)
= \frac{1}{N}\,\mathrm{Var}_{Z \sim P}\!\left[f(Z)\right].
\label{eq:mc_var}
\end{equation}
Thus, for fixed $N$, a larger population variance $\mathrm{Var}[f(Z)]$ results in a higher-variance estimator.

\subsection{Law of Total Variance}

Let $x \sim p(x)$ and $y \sim q(y \mid x)$.
For any function $f(y)$,
\begin{align}
\mathrm{Var}_{x,y}&\!\left[f(y)\right]
=
\mathbb{E}_{x}\!\left[
    \mathrm{Var}_{y \mid x}\!\left(f(y)\right)
\right]\\
&+
\mathrm{Var}_{x}\!\left(
    \mathbb{E}_{y \mid x}[f(y)]
\right).
\label{eq:total_variance}
\end{align}
The first term captures within-$x$ variability, while the second term reflects variability across different $x$.

\subsection{Variance Ordering of Nested Monte Carlo Estimators}

Consider the expectations
\begin{align}
\mu_0
&=
\mathbb{E}_{\tau_0 \sim \pi_\theta(\cdot \mid q)}\!\left[R(\tau_0)\right], \\
\mu_1(\tau_0)
&=
\mathbb{E}_{\tau_1 \sim \pi_\theta(\cdot \mid \tau_0, p)}\!\left[R(\tau_1)\right], \\
\mu
&=
\mathbb{E}_{\tau_0 \sim \pi_\theta(\cdot \mid q),\;
            \tau_1 \sim \pi_\theta(\cdot \mid \tau_0, p)}
\!\left[R(\tau_1)\right].
\end{align}

\paragraph{Estimators.}
Define the Monte Carlo estimators
\begin{align}
\hat{\mu}_0
&=
\frac{1}{N} \sum_{i=1}^N R(\tau_0^{(i)}),
\\
\tau_0^{(i)} &\sim \pi_\theta(\cdot \mid q), \\
\hat{\mu}_1(\tau_0)
&=
\frac{1}{M} \sum_{j=1}^M R(\tau_1^{(j)}),
\\
\tau_1^{(j)} &\sim \pi_\theta(\cdot \mid \tau_0, p), \\
\hat{\mu}
&=
\frac{1}{NM} \sum_{i=1}^N \sum_{j=1}^M
R(\tau_1^{(i,j)}),
\\
\tau_1^{(i,j)} &\sim \pi_\theta(\cdot \mid \tau_0^{(i)}, p).
\end{align}

\paragraph{Variance comparison.}
Applying Eq.~\eqref{eq:total_variance} with $x=\tau_0$ and $y=\tau_1$,
\begin{align}
\mathrm{Var}_{\tau_0,\tau_1}&\!\left[R(\tau_1)\right]
=
\mathbb{E}_{\tau_0}\!\left[
    \mathrm{Var}_{\tau_1 \mid \tau_0}\!\left(R(\tau_1)\right)
\right]\\
&+
\mathrm{Var}_{\tau_0}\!\left(
    \mathbb{E}_{\tau_1 \mid \tau_0}[R(\tau_1)]
\right).
\label{eq:nested_variance}
\end{align}
Since the second term is non-negative,
\begin{equation}
\mathrm{Var}_{\tau_0,\tau_1}\!\left[R(\tau_1)\right]
\;\ge\;
\mathbb{E}_{\tau_0}\!\left[
    \mathrm{Var}_{\tau_1 \mid \tau_0}\!\left(R(\tau_1)\right)
\right].\label{eq:com_varience}
\end{equation}
Therefore, under the same sampling budget,
\begin{equation}
\mathrm{Var}(\hat{\mu})
\;\ge\;
\mathrm{Var}\!\left(\hat{\mu}_1(\tau_0)\right),
\end{equation}
indicating that conditioning on a fixed $\tau_0$ yields a lower-variance Monte Carlo estimator.

\subsection{Other Supporting Evidence}
We also empirically approximate that the baseline of post-editing gradients is smaller than that of the MT policy gradients in our framework, as shown in Figure~\ref{fig:extra_baseline}.
\begin{figure}[htbp]
    \centering
    \includegraphics[width=0.43\textwidth]{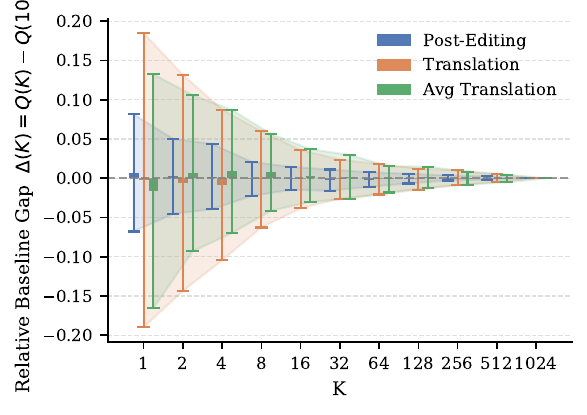}
\caption{
Convergence of the GRPO baseline estimation with respect to the number of sampled trajectories $K$ for post-editing, translation, and average translation (baseline estimation for the translation task in our framework).  
For each of 100 sampled instances, 1024 trajectories are rolled out and the resulting baseline is used as a reference.  
The figure reports the mean and standard deviation (error bars) of the relative baseline gap 
$\Delta(K)=Q(K)-Q(1024)$ computed from the first $K$ trajectories.  
Smaller error bars indicate lower variance in baseline estimation across instances, corresponding to more stable policy gradient estimates.
}

    \label{fig:extra_baseline}
\end{figure}

\section{Equivalence Between Absolute and Relative Rewards}
\label{app:grpo-qe-equivalence}

\begin{theorem}
\label{thm:grpo-qe-equivalence}
Under GRPO group-advantage normalization, optimizing post-editing rewards
defined by absolute quality scores is equivalent to optimizing rewards
defined by quality improvements.
\end{theorem}

\begin{proof}
Let $\mathrm{QE}(\mathrm{pe}_j)$ denote the quality score of the $j$-th
post-editing output, and define the quality improvement
$\Delta \mathrm{QE}(\mathrm{pe}_j) = \mathrm{QE}(\mathrm{pe}_j) - C$,
where $C$ is a constant baseline shared across all samples in the group.

For a group of $M$ post-editing outputs, the GRPO-normalized advantage is
\begin{equation}
A_j =
\frac{\mathrm{QE}(\mathrm{pe}_j)
      - \mathrm{Mean}(\{\mathrm{QE}(\mathrm{pe}_j)\}_{j=1}^{M})}
     {\mathrm{Std}(\{\mathrm{QE}(\mathrm{pe}_j)\}_{j=1}^{M})}.\notag
\end{equation}
Since subtracting a constant does not affect either the mean or the standard
deviation within a group, we equivalently obtain
\begin{equation}
A_j =
\frac{\Delta \mathrm{QE}(\mathrm{pe}_j)
      - \mathrm{Mean}(\{\Delta \mathrm{QE}(\mathrm{pe}_j)\}_{j=1}^{M})}
     {\mathrm{Std}(\{\Delta \mathrm{QE}(\mathrm{pe}_j)\}_{j=1}^{M})}.\notag
\end{equation}
Therefore, maximizing the GRPO objective based on absolute quality scores
is equivalent to maximizing the objective based on quality improvements.
\end{proof}

\section{Prompt Templates}
\label{sec:appendix-prompts}

\begin{tcolorbox}[
    title=Translation Prompt Template,
    boxrule=0.4pt,
    sharp corners,
]
Translate the following text into \{tgt\_lang\} without additional explanations:\\
$\{\textit{src}\}$
\end{tcolorbox}
\label{app:mt-prompt}

\begin{tcolorbox}[
    title=Post-editing Prompt Template,
    boxrule=0.4pt,
    sharp corners,
]
Given the source text:\\
$\{\textit{src}\}$\\
Improve the following draft \{tgt\_lang\} translation into a high-quality
\{tgt\_lang\} version, without explanations:\\
$\{\textit{pred}_i\}$
\end{tcolorbox}
\label{app:pe-prompt}

\section{Extended Results}

\subsection{Main Experiment}
\subsubsection{Evaluation}
\paragraph{Large Models.}  
For large-scale models such as Gemini-2.0-flash, DeepSeek-V3.2-Exp, and OpenAI GPT-5.2, we use the official APIs for evaluation. Only the prompt templates from Appendix~\ref{sec:appendix-prompts} are used, with the maximum output length set to 512 tokens. All other generation parameters are left at their default settings.

\paragraph{Small Models.}  
For smaller models, if an official translation prompt is available (e.g., Seed-X-PPO-7B, TowerInstruct-13B-v0.1), we use it; otherwise, we fall back to the prompt templates in Appendix~\ref{sec:appendix-prompts}. During evaluation, sampling parameters are set to recommended defaults, as summarized in Table~\ref{tab:sampling-params}.

\begin{table}[htbp]
\centering
\scriptsize
\setlength{\tabcolsep}{2pt}  
\begin{tabular}{lcccc}
\toprule
Model & Temp & Top-p & Top-k & Rep Pen \\
\midrule
Seed-X-PPO-7B         & 0.0 & -    & -  & - \\
TowerInstruct-13B-v0.1 & 0.0 & -    & -  & - \\
Qwen3                  & 0.6 & 0.95 & 20 & 1.05 \\
MT-R1-Zero             & 0.6 & 0.95 & 20 & 1.05 \\
Ours                   & 0.6 & 0.95 & 20 & 1.05 \\
\bottomrule
\end{tabular}
\caption{Sampling parameters for small models in translation experiments.}
\label{tab:sampling-params}
\end{table}

\subsubsection{Training Dynamics}
\label{app:training-dynamics}

As RL training exhibits non-monotonic convergence, we report the performance trajectories underlying the main experimental results.
Each training step processes 128 samples, and models are trained for 400 steps in total.
Evaluation is performed on the \textbf{test set} every 20 steps, and the corresponding metrics are plotted to illustrate training dynamics over time, as shown in Figures~\ref{fig:appendix_main_en2fi}--\ref{fig:appendix_main_zh2en}.

\subsection{Gradient Weight Analysis}
We report the metric values at step 100 for the three experimental settings (Figure~\ref{fig:grad_fig}) in a table for clarity (Table~\ref{tab:grad_tab}).

\subsection{Platform Generalization}
\label{app:ascend}

\begin{table*}[htbp]
\centering
\small
\setlength{\tabcolsep}{6pt}
\renewcommand{\arraystretch}{1.15}
\begin{tabular}{lcccccccc}
\toprule
\multirow{2}{*}{\textbf{Model}} 
& \multicolumn{2}{c}{\textbf{EN--FI (WMT24)}} 
& \multicolumn{2}{c}{\textbf{EN--FI (FLORES)}} 
& \multicolumn{2}{c}{\textbf{EN--TR (WMT24)}} 
& \multicolumn{2}{c}{\textbf{EN--TR (FLORES)}} \\
\cmidrule(lr){2-3}
\cmidrule(lr){4-5}
\cmidrule(lr){6-7}
\cmidrule(lr){8-9}
& chrF++ & Kiwi 
& chrF++ & Kiwi 
& chrF++ & Kiwi 
& chrF++ & Kiwi \\
\midrule

\textbf{Ours-4B (Nvidia)} 
& 45.29 & 62.49 
& 42.65 & 73.78 
& 45.39 & 65.49 
& 47.77 & 76.35 \\

\textbf{Ours-4B (Ascend)} 
& 45.15 & 63.07 
& 42.32 & 73.32 
& 45.06 & 65.41 
& 47.29 & 76.09 \\

\bottomrule
\end{tabular}
\caption{Comparison of Ours-4B trained on different hardware platforms. 
We report chrF++ and Kiwi on EN--FI and EN--TR translation tasks. 
Results show that the performance remains consistent across NVIDIA and Ascend platforms under comparable training settings.}
\label{tab:platform_generalization}
\end{table*}

Building upon verl, we extend pytorch\_lightning, which is used by COMET models, with NPU support to enable the entire pipeline to run on NPU servers. We conduct experiments on a server equipped with 8× Ascend 910 NPUs.

Under comparable training budgets, we evaluate the performance against models trained on NVIDIA H20 (96GB) GPUs, as shown in Table~\ref{tab:platform_generalization}. The results demonstrate that our approach maintains consistent performance across different hardware platforms.

We release the weights of our 4B model at \url{https://huggingface.co/collections/DGME/pegrl}. The NVIDIA-based model is re-trained on A6000 GPUs, while the Ascend-based model is trained on the aforementioned NPU server.

\begin{table*}[htbp]
    \centering
    \small
    \begin{tabular}{cccc}
        \toprule
        $\lambda_{\text{pe}}$ & $\lambda_{\text{mt}}$ & \multicolumn{1}{c}{chrF++} & \multicolumn{1}{c}{COMETKIWI} \\
        \midrule
$M$ & $1$ & 43.79 & 56.96 \\
        $M$ & $0$ & 42.48 (\textcolor{blue}{\(\downarrow 1.31\)}) & 55.80 (\textcolor{blue}{\(\downarrow 1.16\)}) \\
        $1$ & $1$ & 43.06 (\textcolor{blue}{\(\downarrow 0.73\)}) & 53.09 (\textcolor{blue}{\(\downarrow 3.87\)}) \\
        \bottomrule
    \end{tabular}
\caption{Performance at step 100 (corresponding to Figure~\ref{fig:grad_fig}). Values in subsequent rows are compared to the first row.}

    \label{tab:grad_tab}
\end{table*}

\begin{figure*}[htbp]
    \centering
    \setlength{\tabcolsep}{0pt}

    \subfigure{
        \includegraphics[width=0.29\textwidth]{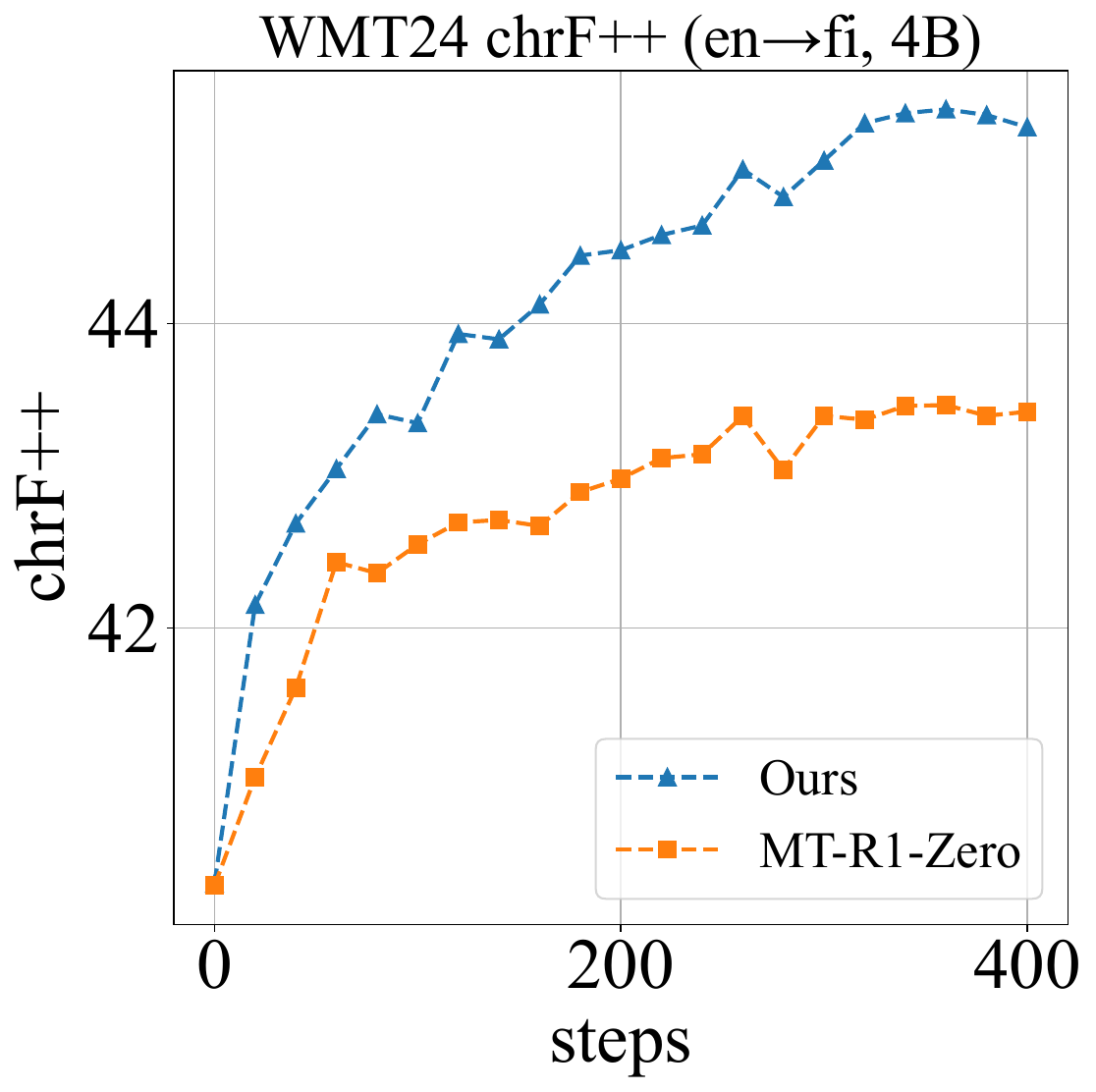}
    }
    \hfill
    \subfigure{
        \includegraphics[width=0.29\textwidth]{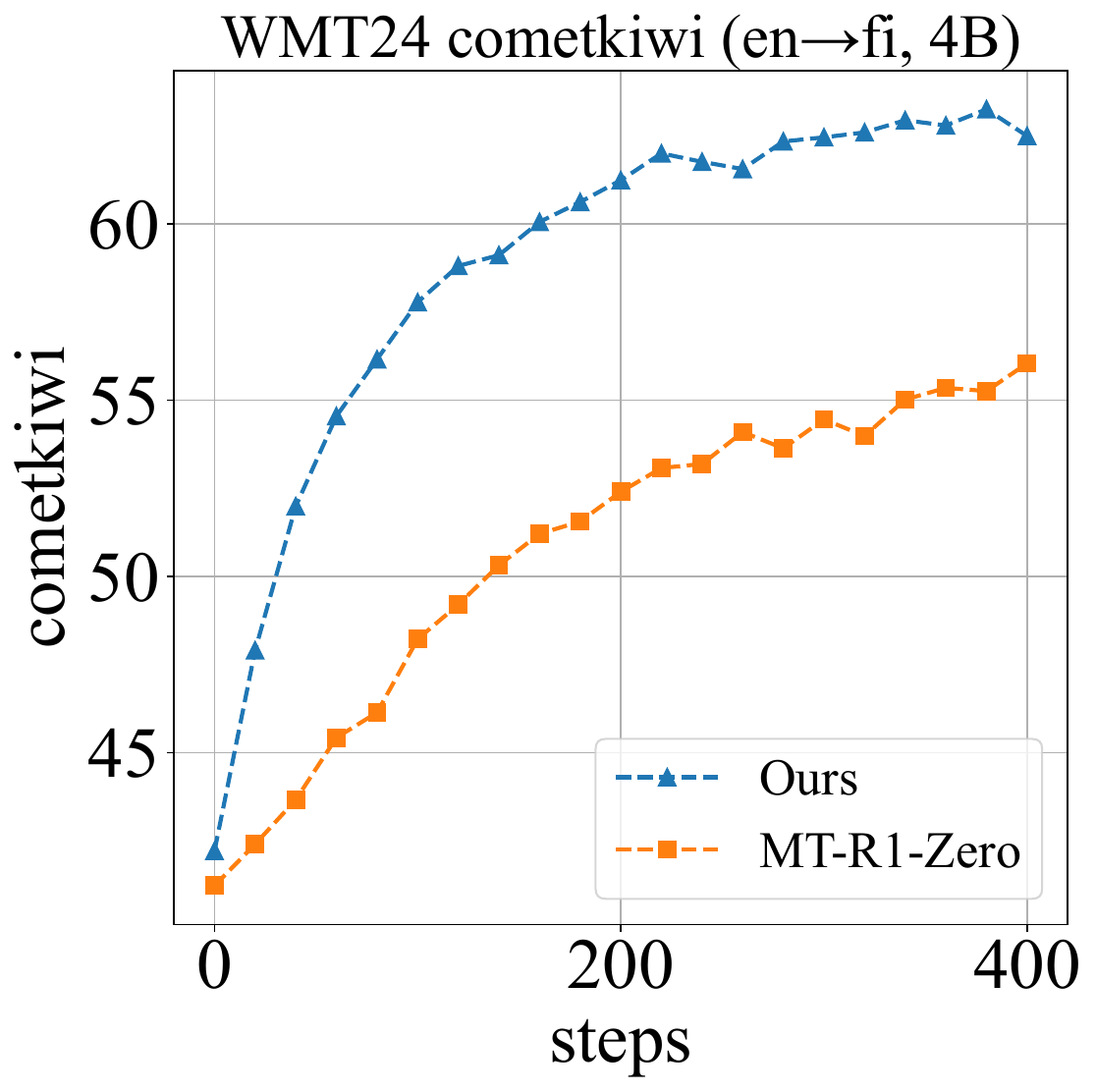}
    }
    \hfill
    \subfigure{
        \includegraphics[width=0.29\textwidth]{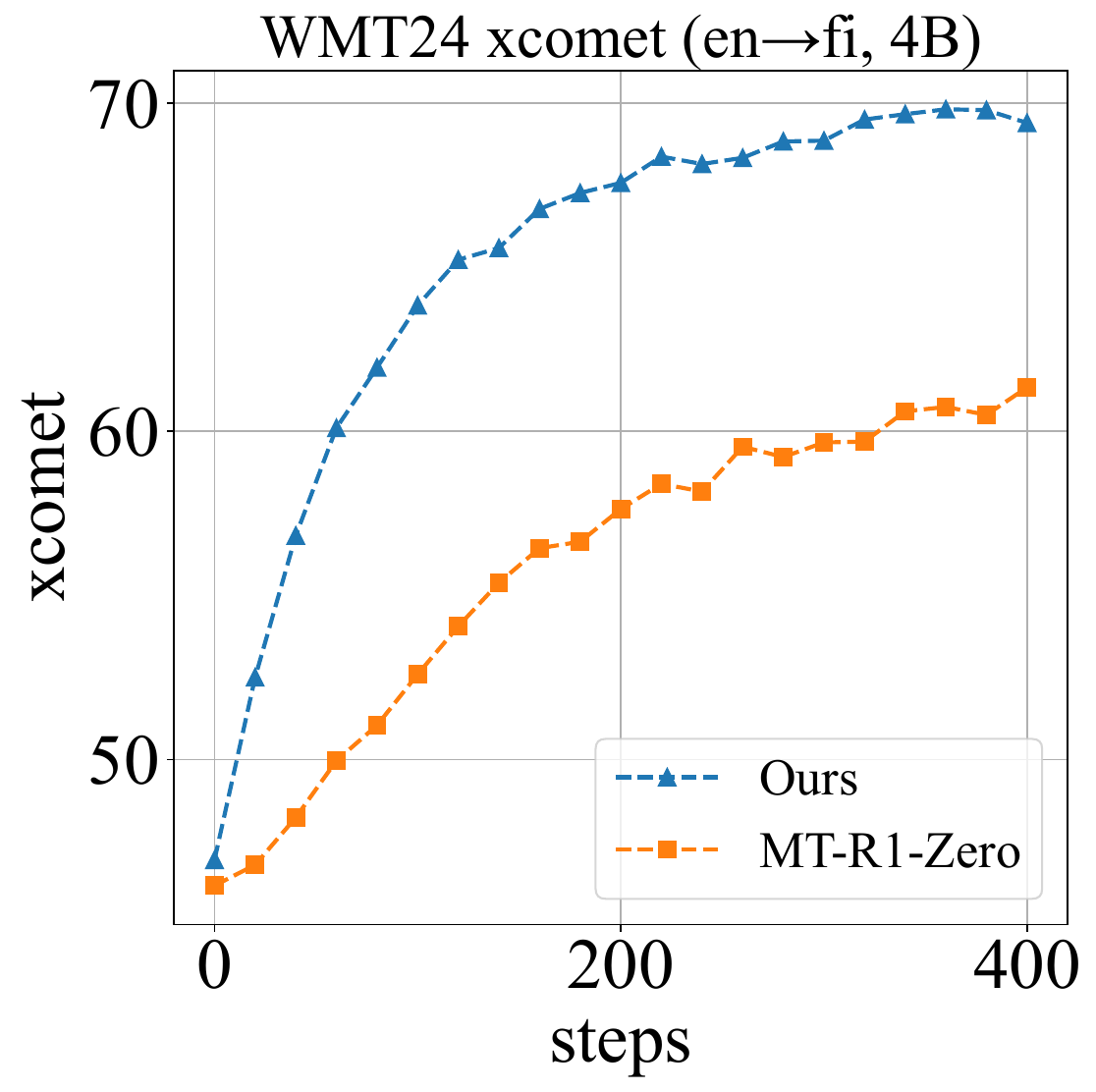}
    }

    \vspace{2mm}

    \subfigure{
        \includegraphics[width=0.29\textwidth]{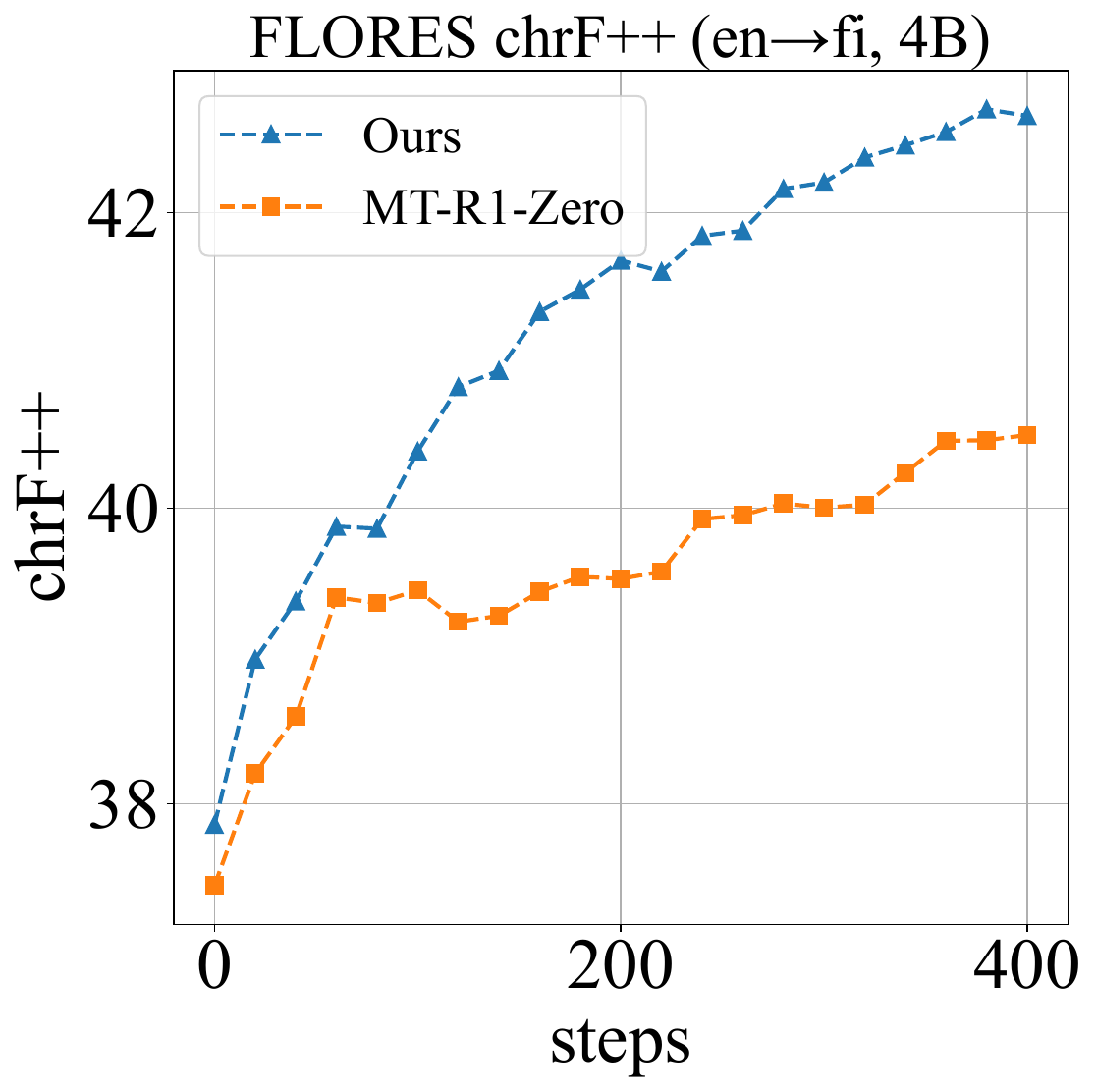}
    }
    \hfill
    \subfigure{
        \includegraphics[width=0.29\textwidth]{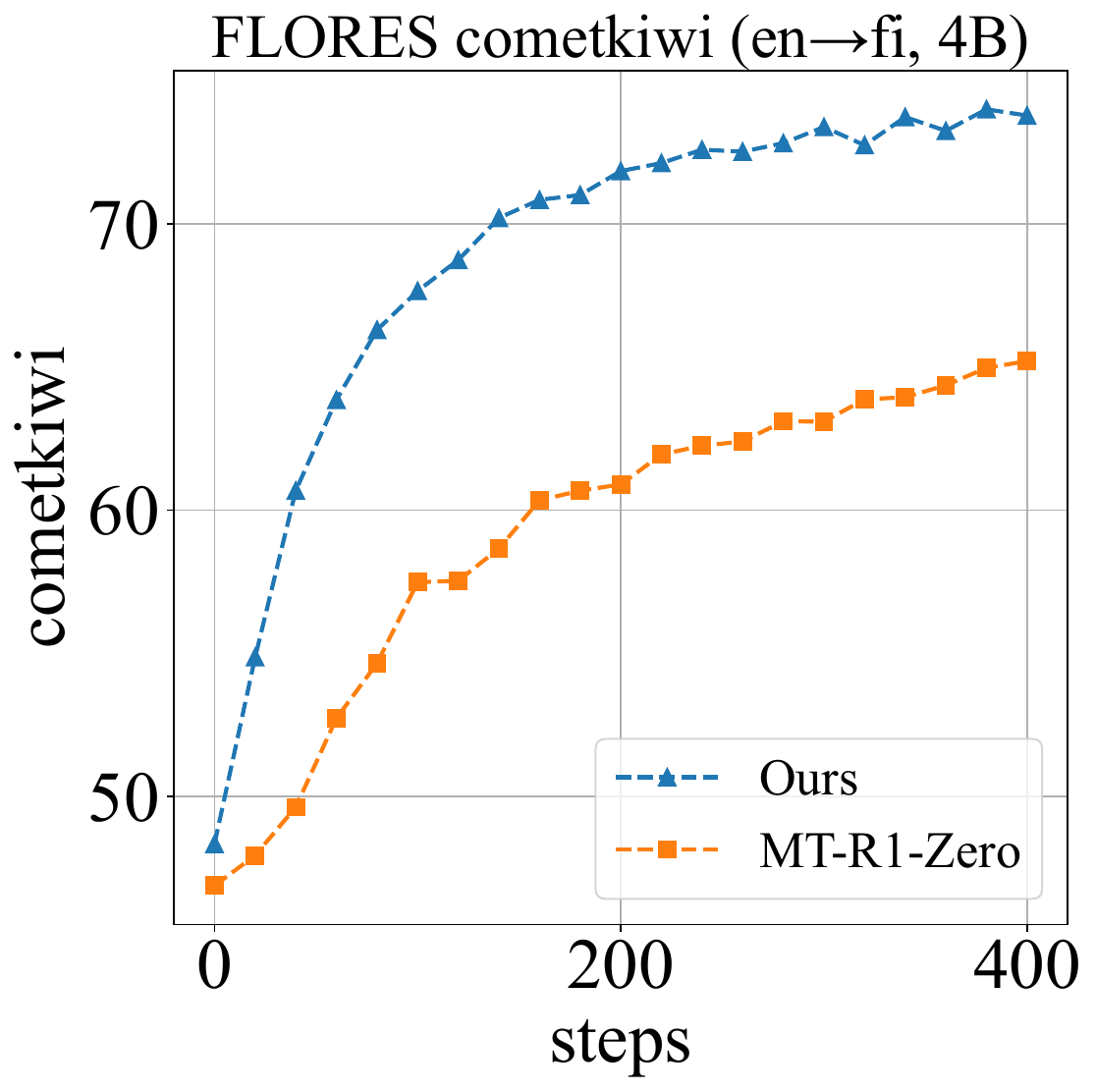}
    }
    \hfill
    \subfigure{
        \includegraphics[width=0.29\textwidth]{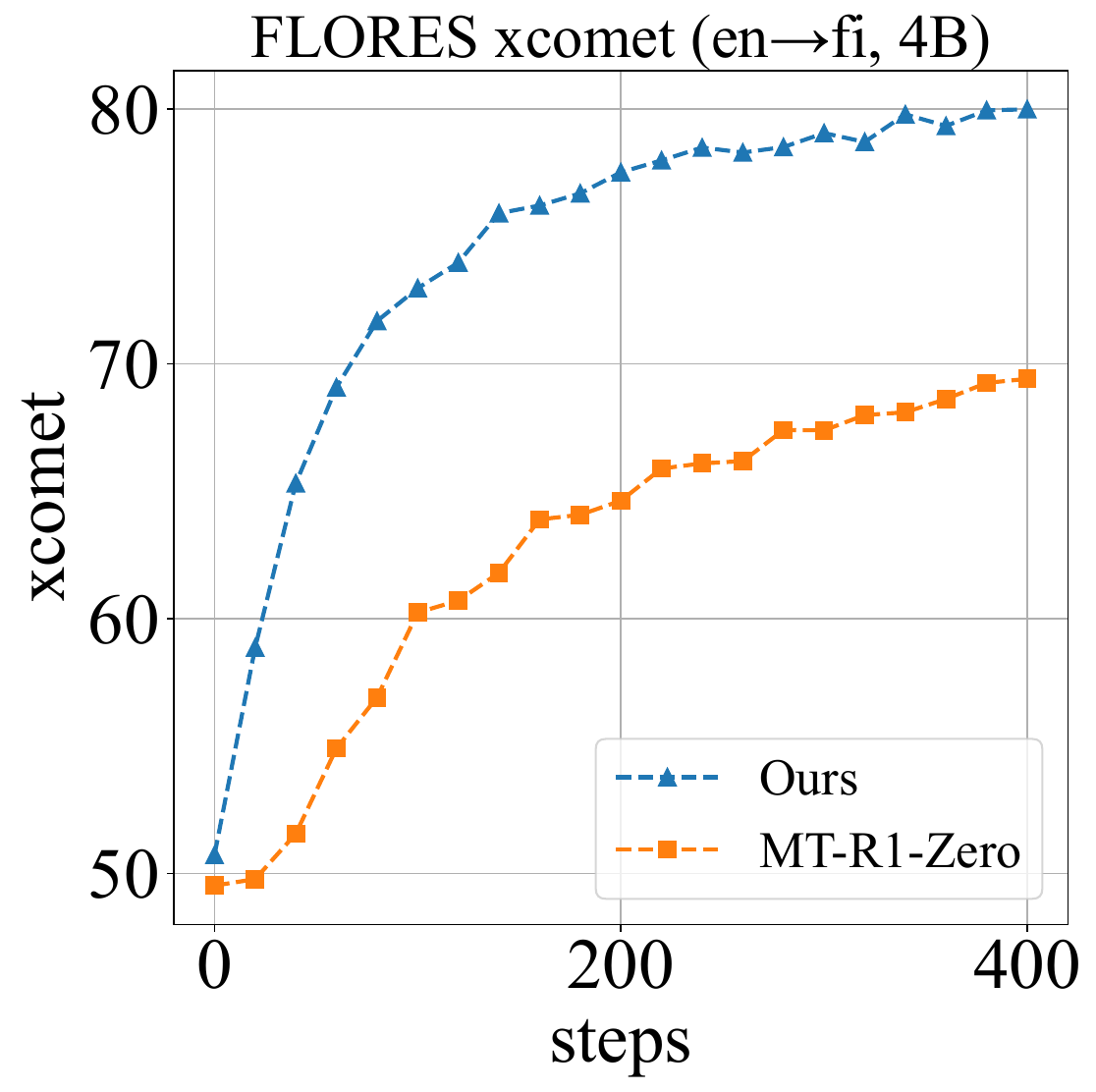}
    }

    \vspace{4mm}

    \subfigure{
        \includegraphics[width=0.29\textwidth]{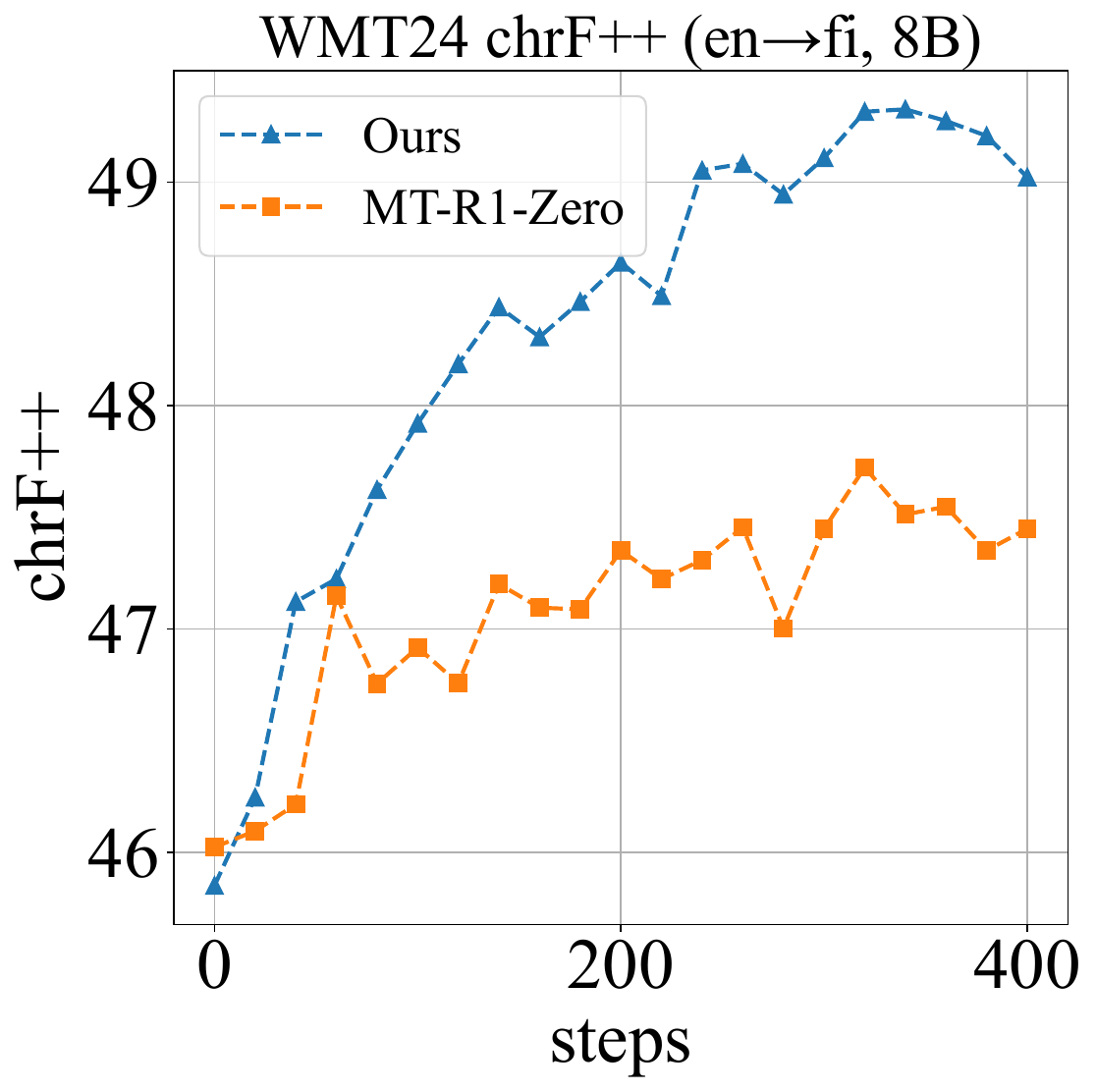}
    }
    \hfill
    \subfigure{
        \includegraphics[width=0.29\textwidth]{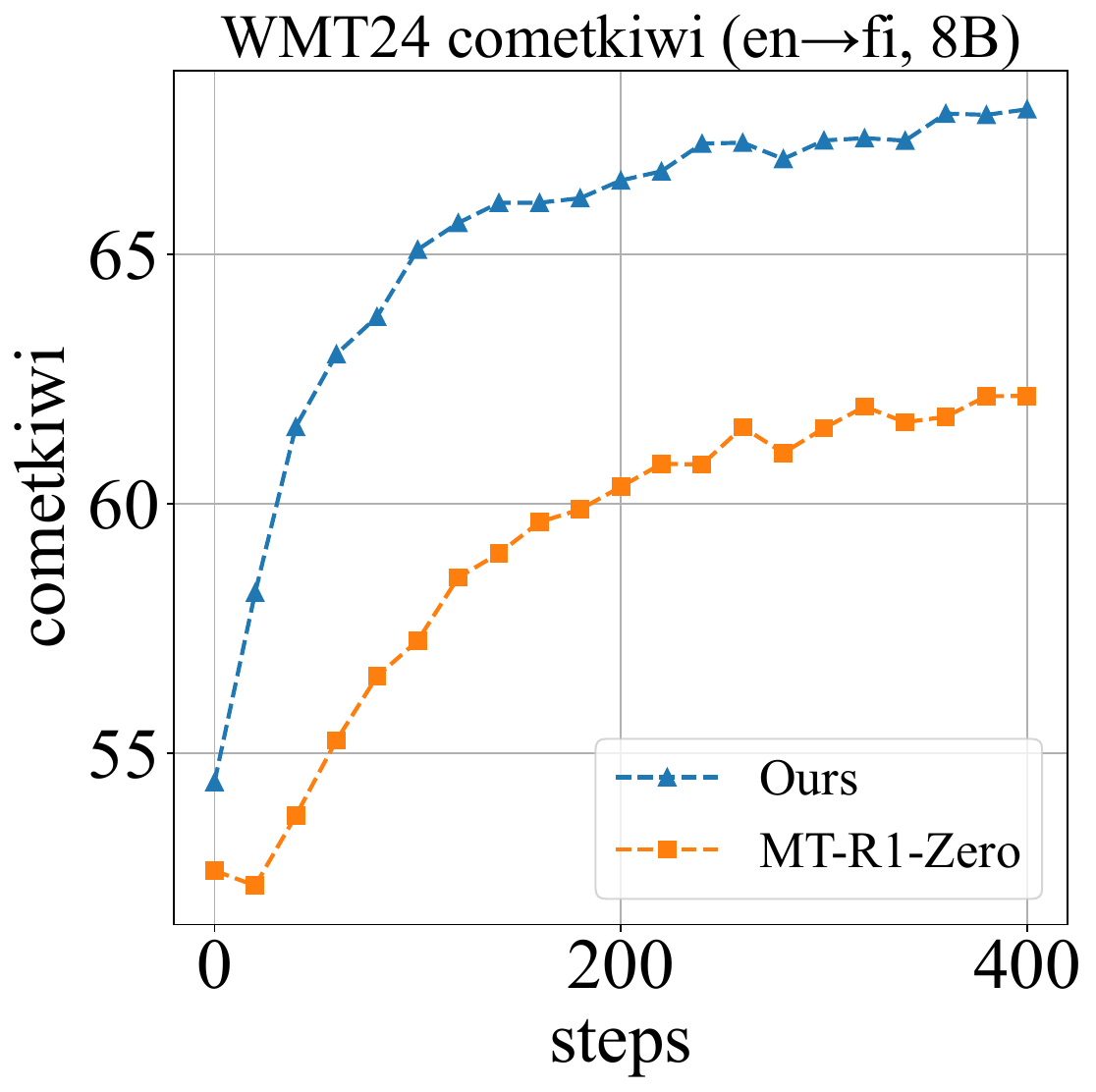}
    }
    \hfill
    \subfigure{
        \includegraphics[width=0.29\textwidth]{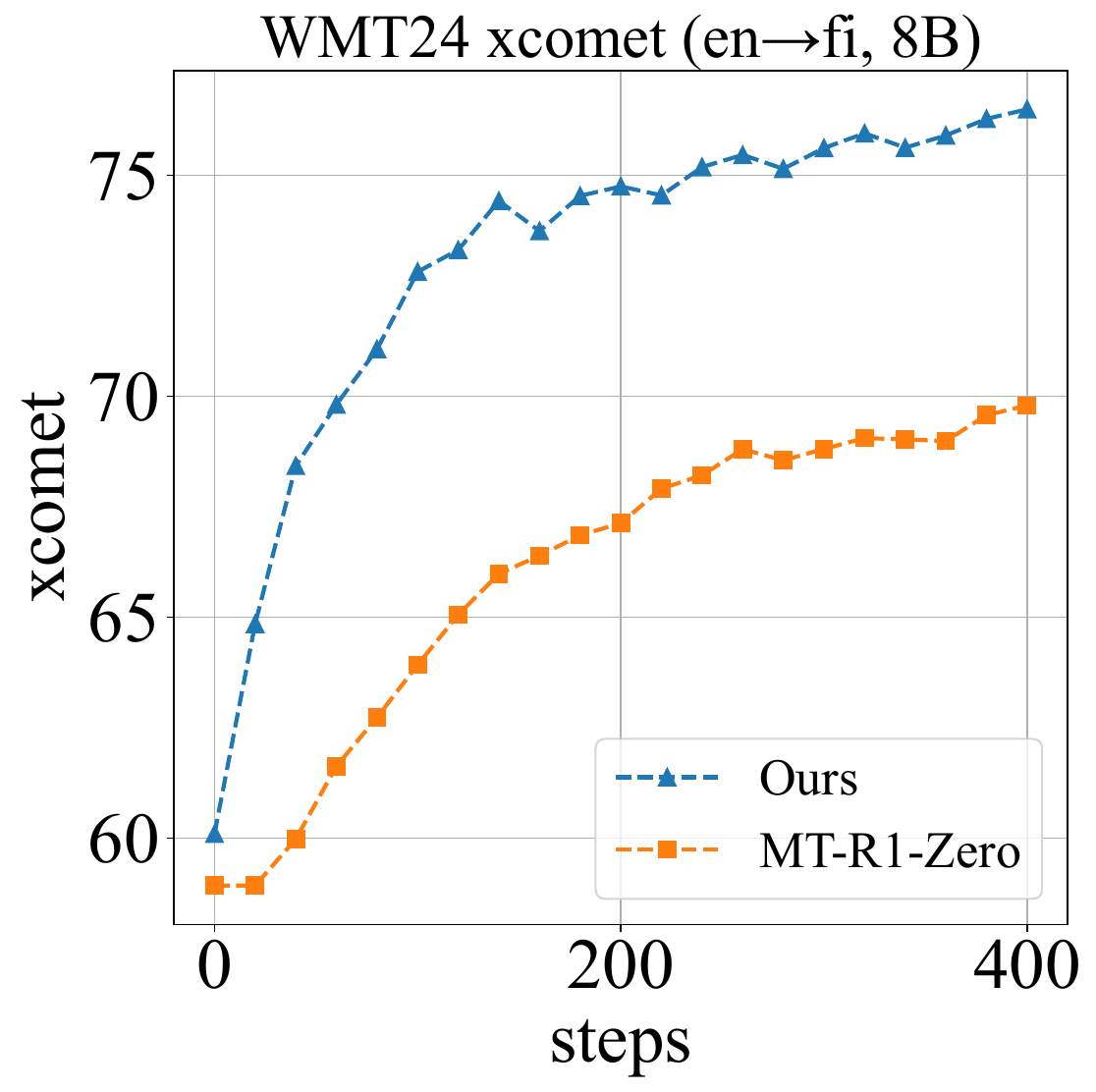}
    }

    \vspace{2mm}

    \subfigure{
        \includegraphics[width=0.29\textwidth]{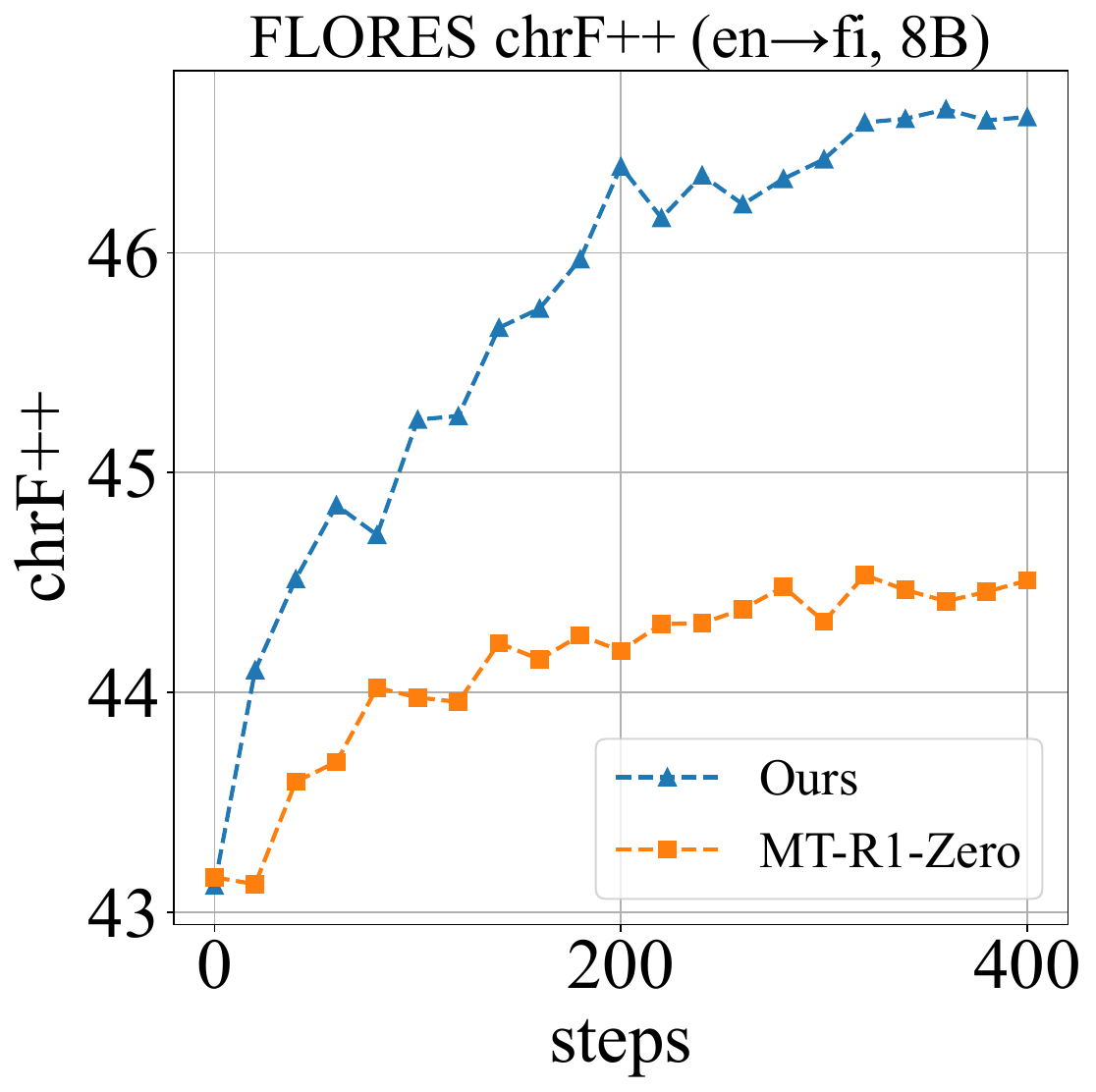}
    }
    \hfill
    \subfigure{
        \includegraphics[width=0.29\textwidth]{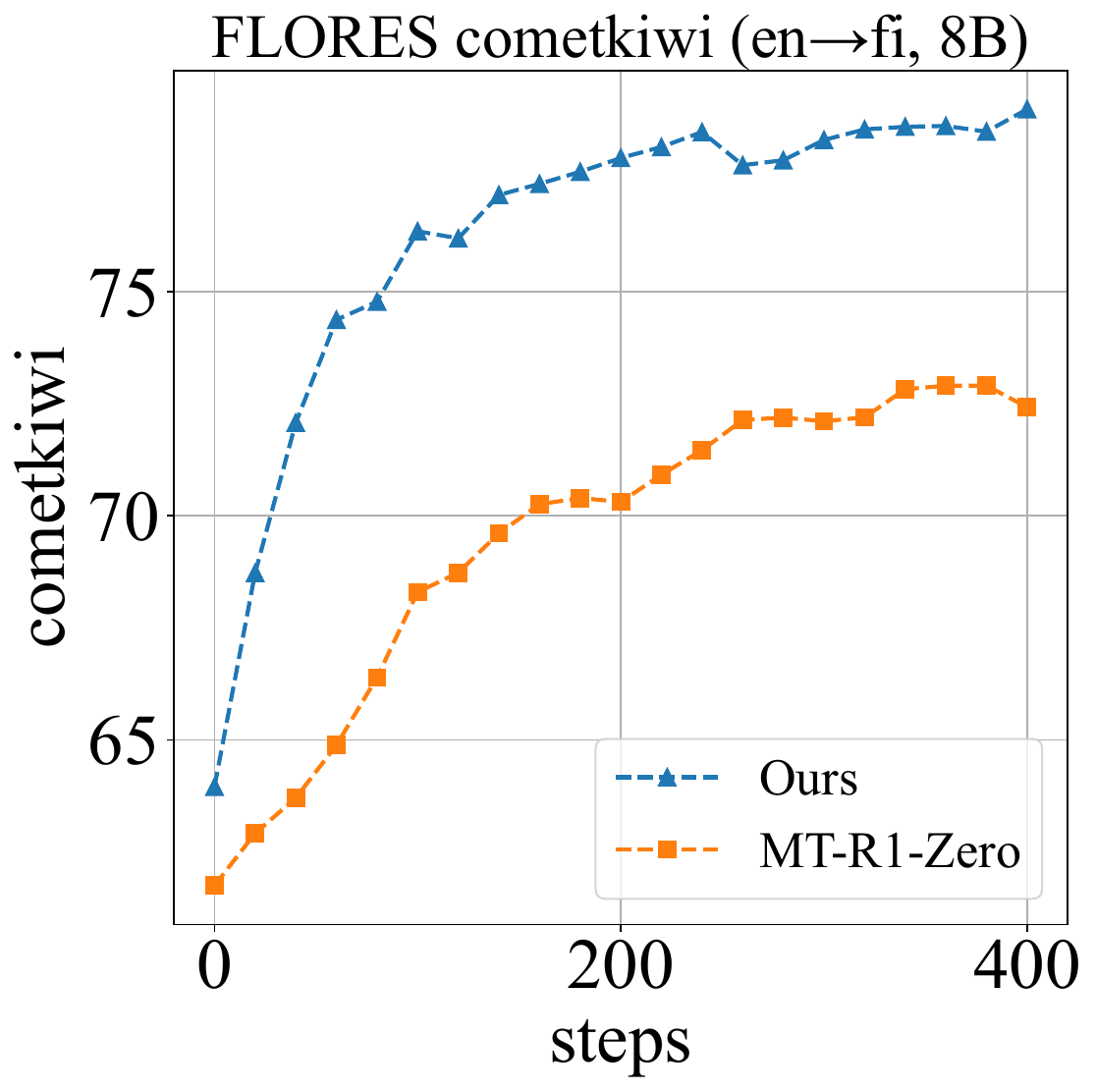}
    }
    \hfill
    \subfigure{
        \includegraphics[width=0.29\textwidth]{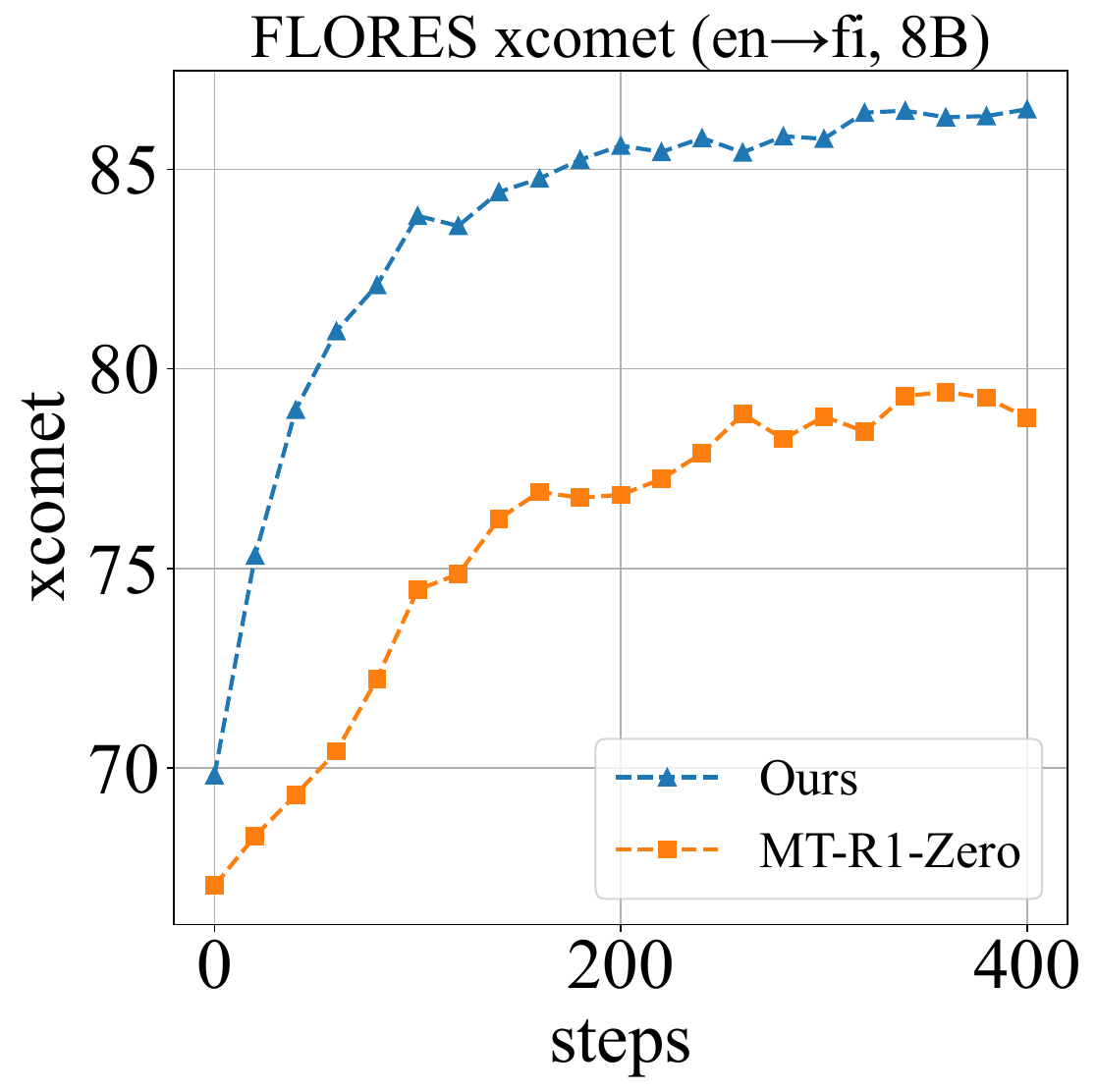}
    }

    \vspace{6mm}

    \caption{
        Training dynamics on FLORES and WMT24 for EN$\rightarrow$FI
        under different model scales (4B, 8B), evaluated by chrF++, COMET-Kiwi, and XCOMET.
    }
    \label{fig:appendix_main_en2fi}
\end{figure*}

\begin{figure*}[htbp]
    \centering
    \setlength{\tabcolsep}{0pt}

    \subfigure{
        \includegraphics[width=0.29\textwidth]{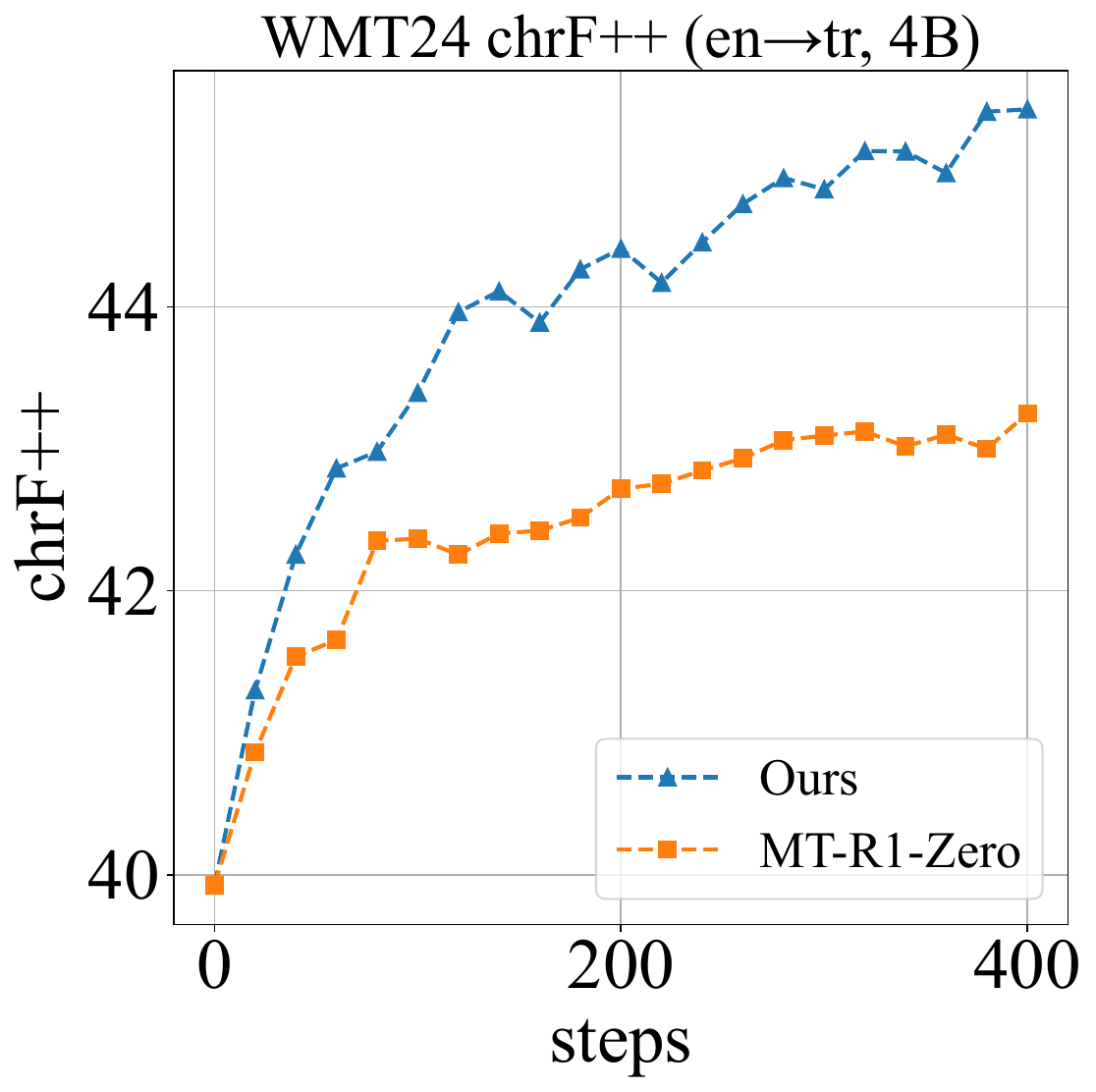}
    }
    \hfill
    \subfigure{
        \includegraphics[width=0.29\textwidth]{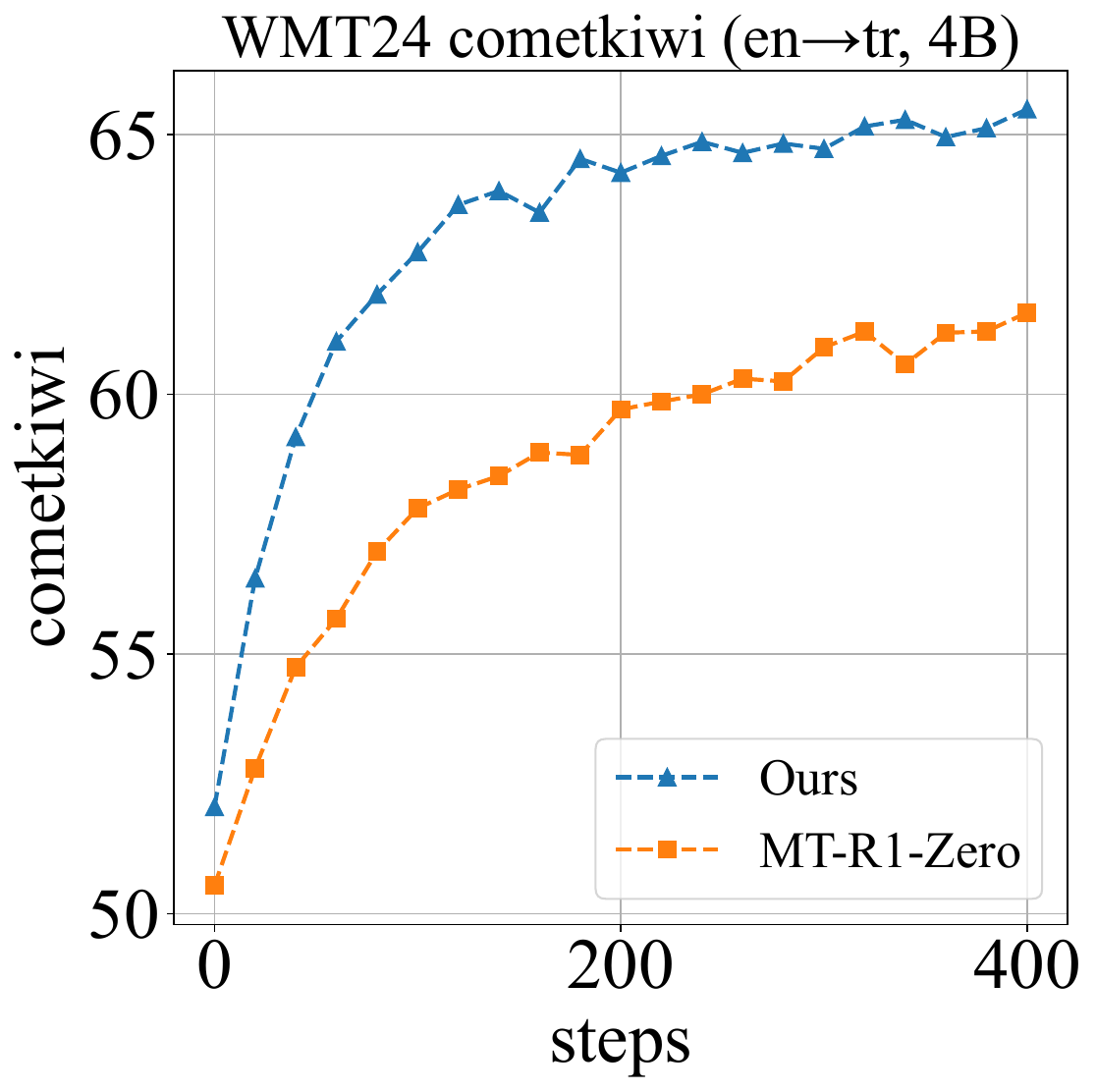}
    }
    \hfill
    \subfigure{
        \includegraphics[width=0.29\textwidth]{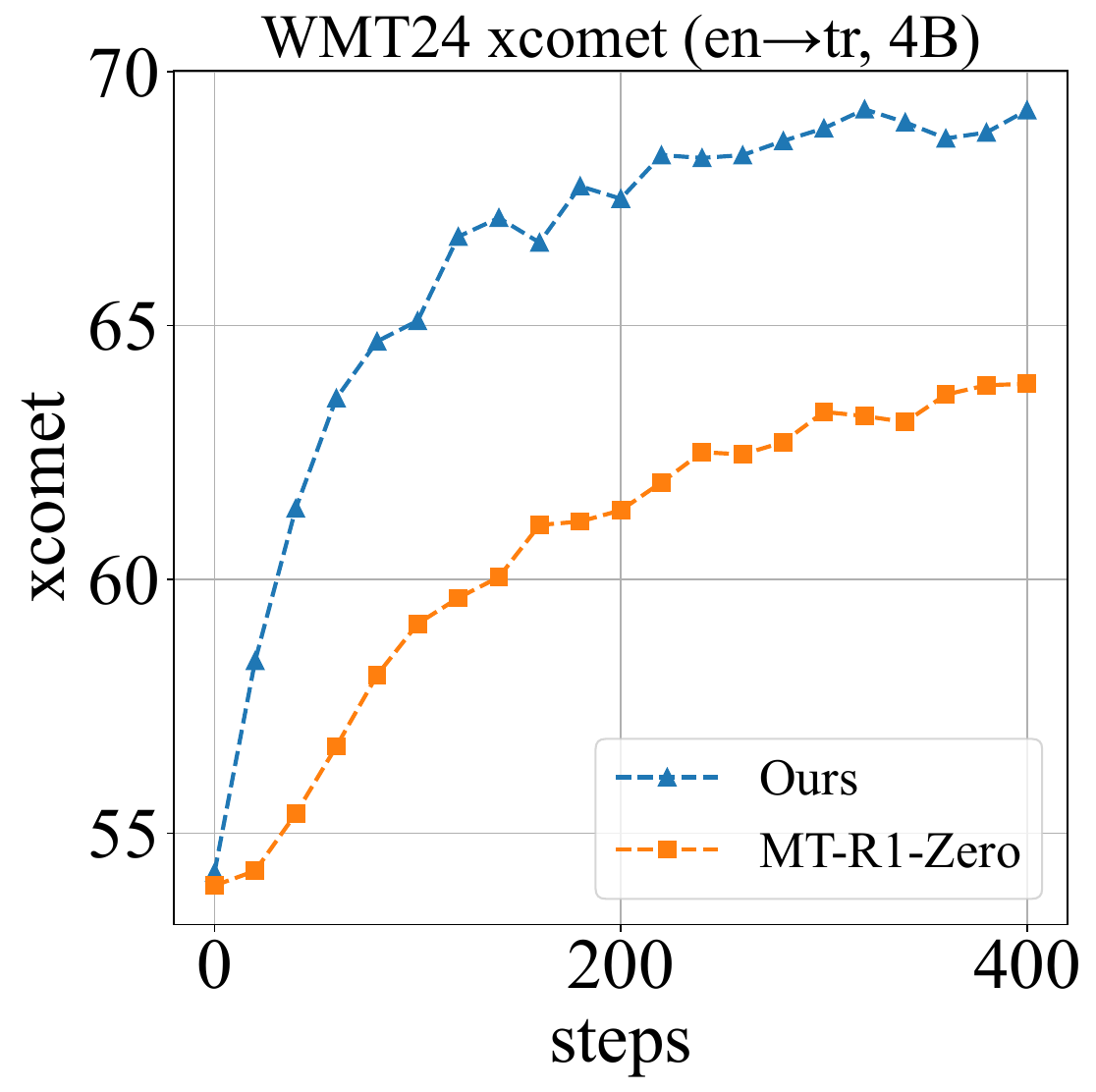}
    }

    \vspace{2mm}

    \subfigure{
        \includegraphics[width=0.29\textwidth]{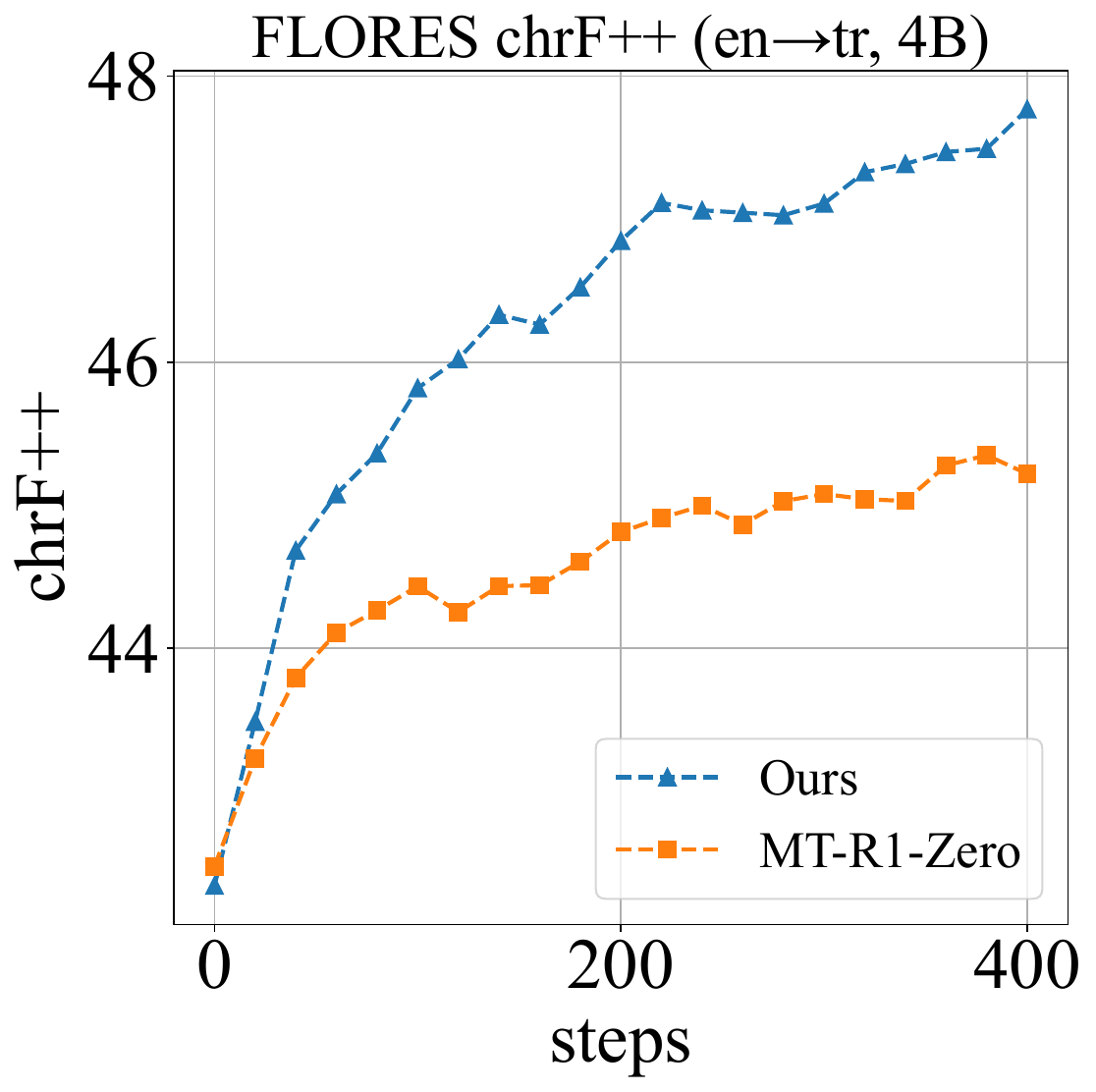}
    }
    \hfill
    \subfigure{
        \includegraphics[width=0.29\textwidth]{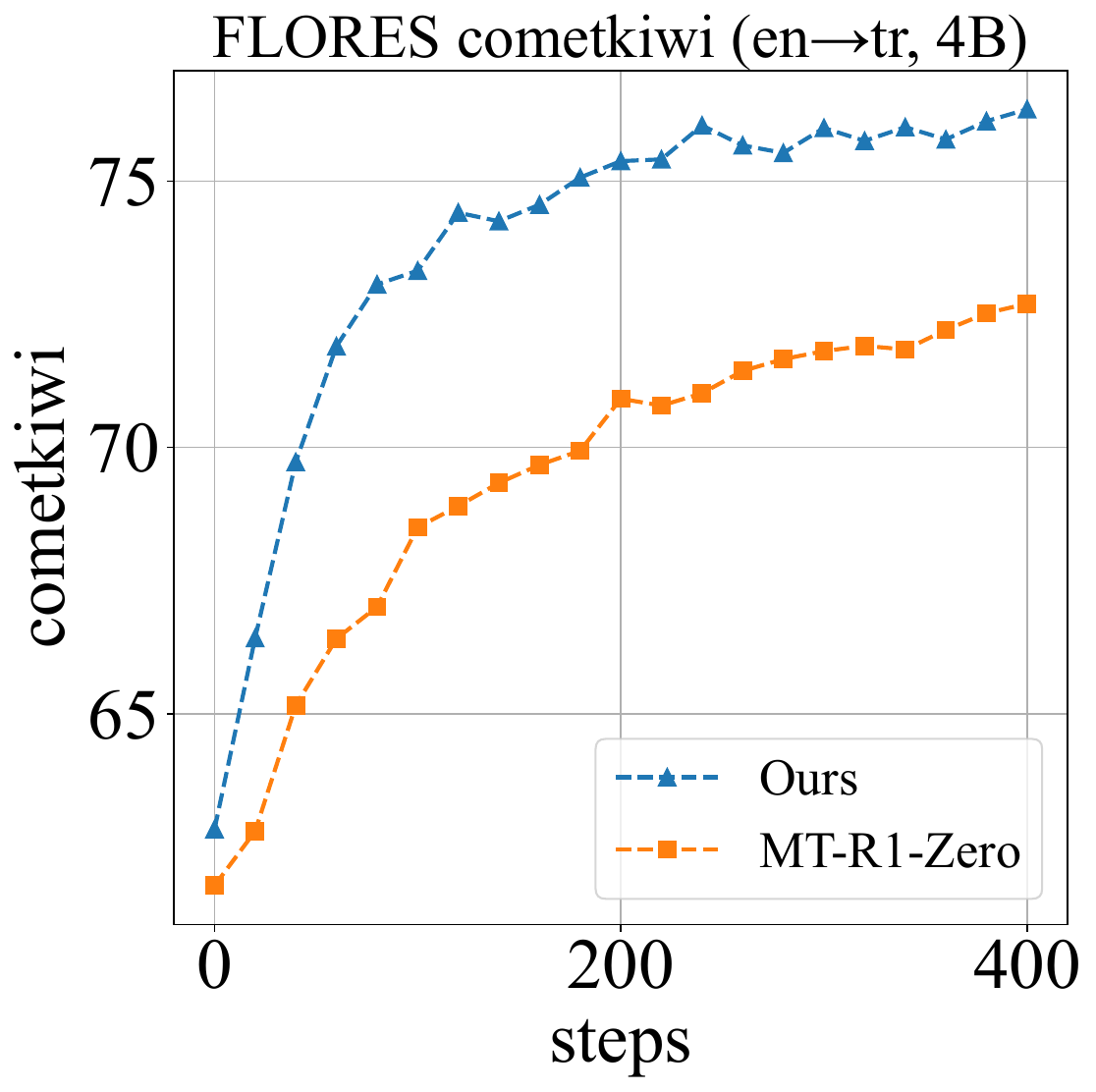}
    }
    \hfill
    \subfigure{
        \includegraphics[width=0.29\textwidth]{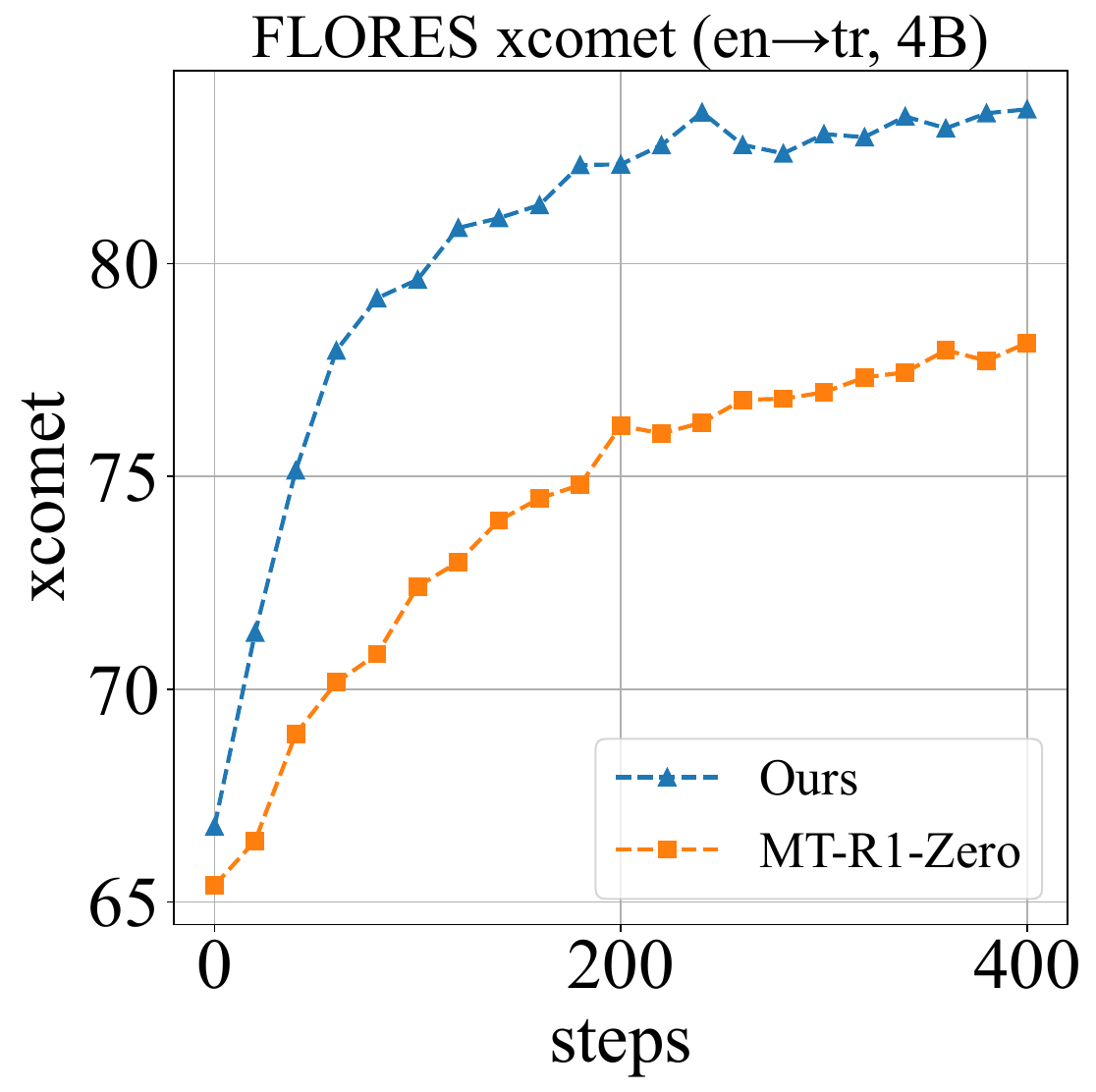}
    }

    \vspace{4mm}

    \subfigure{
        \includegraphics[width=0.29\textwidth]{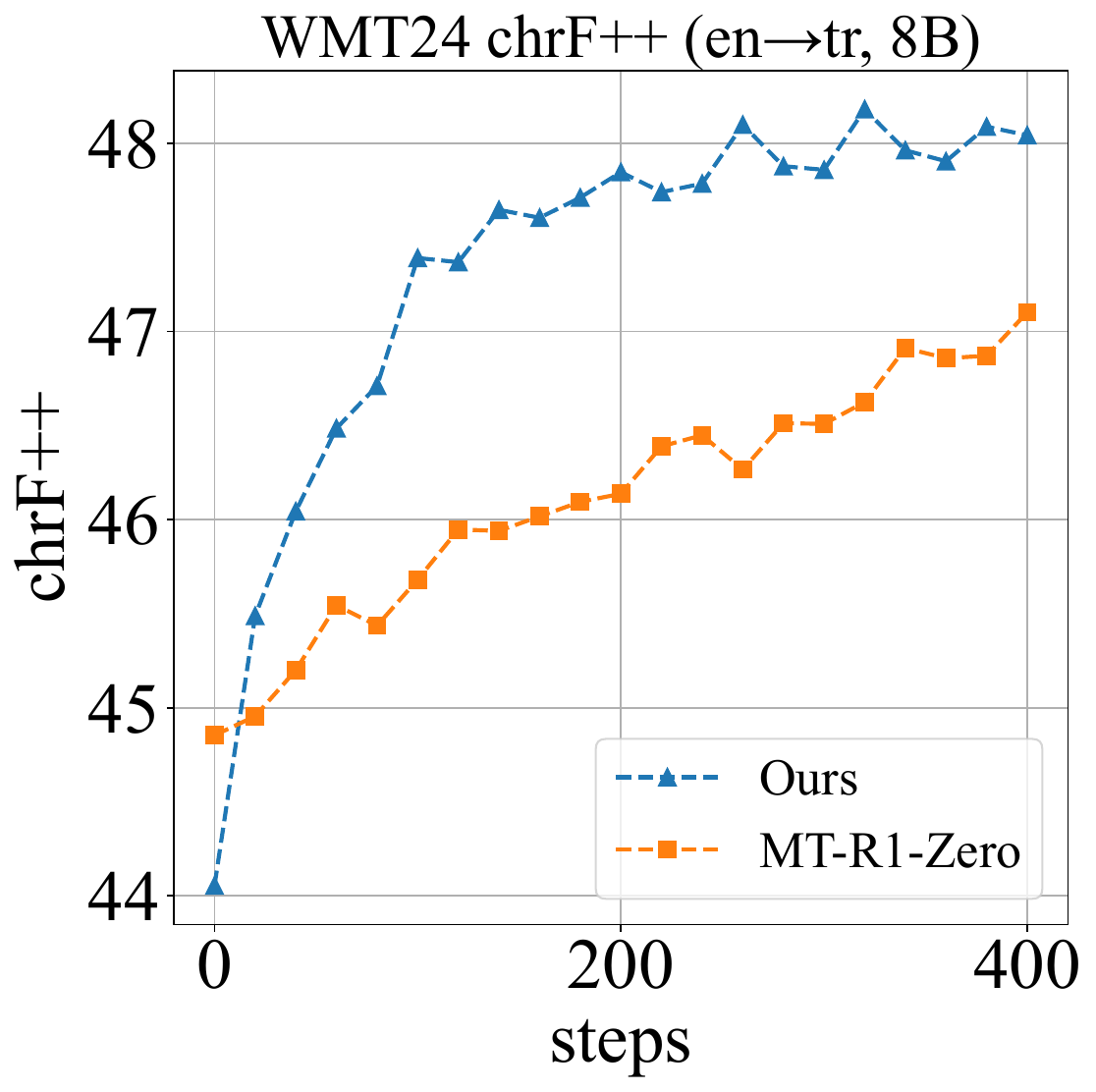}
    }
    \hfill
    \subfigure{
        \includegraphics[width=0.29\textwidth]{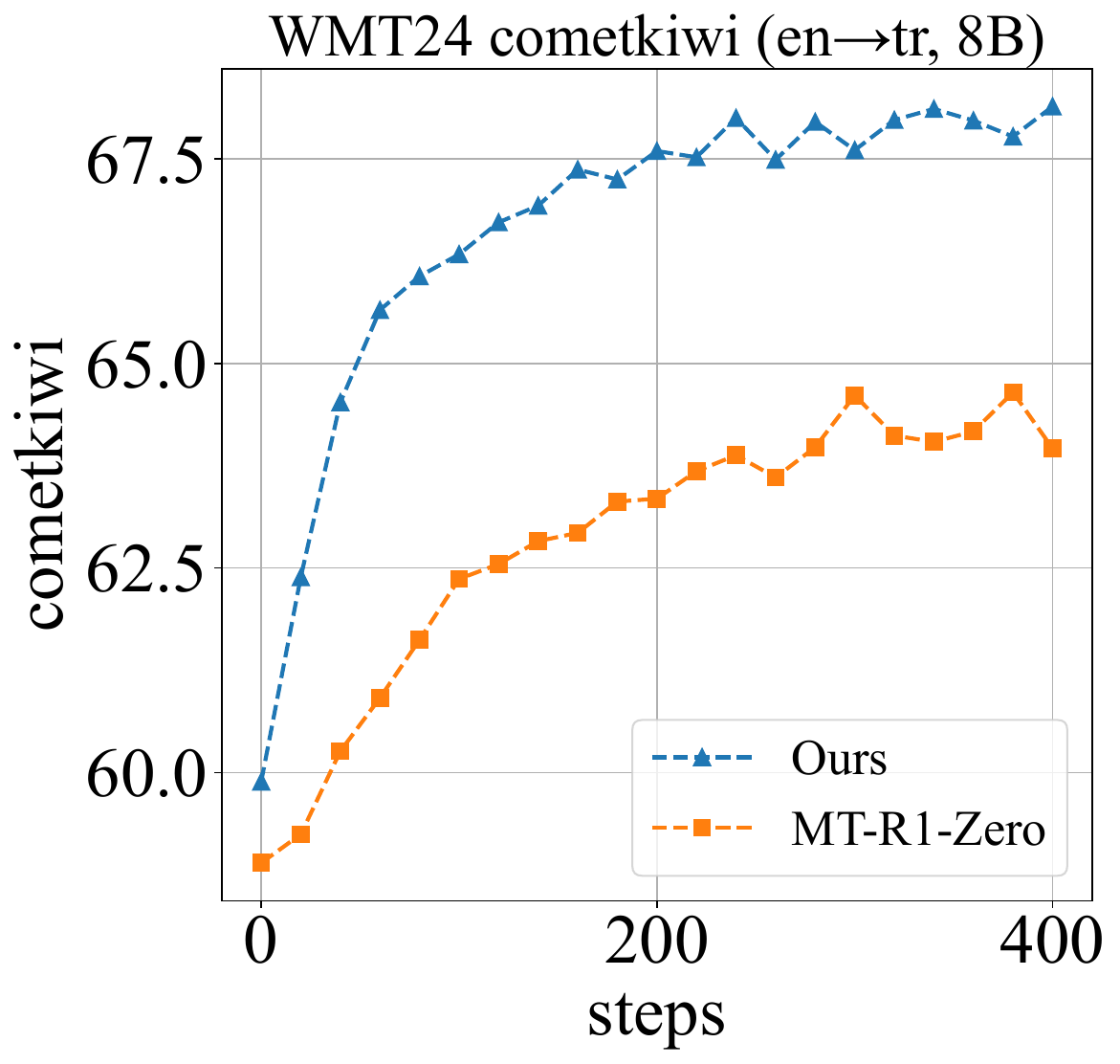}
    }
    \hfill
    \subfigure{
        \includegraphics[width=0.29\textwidth]{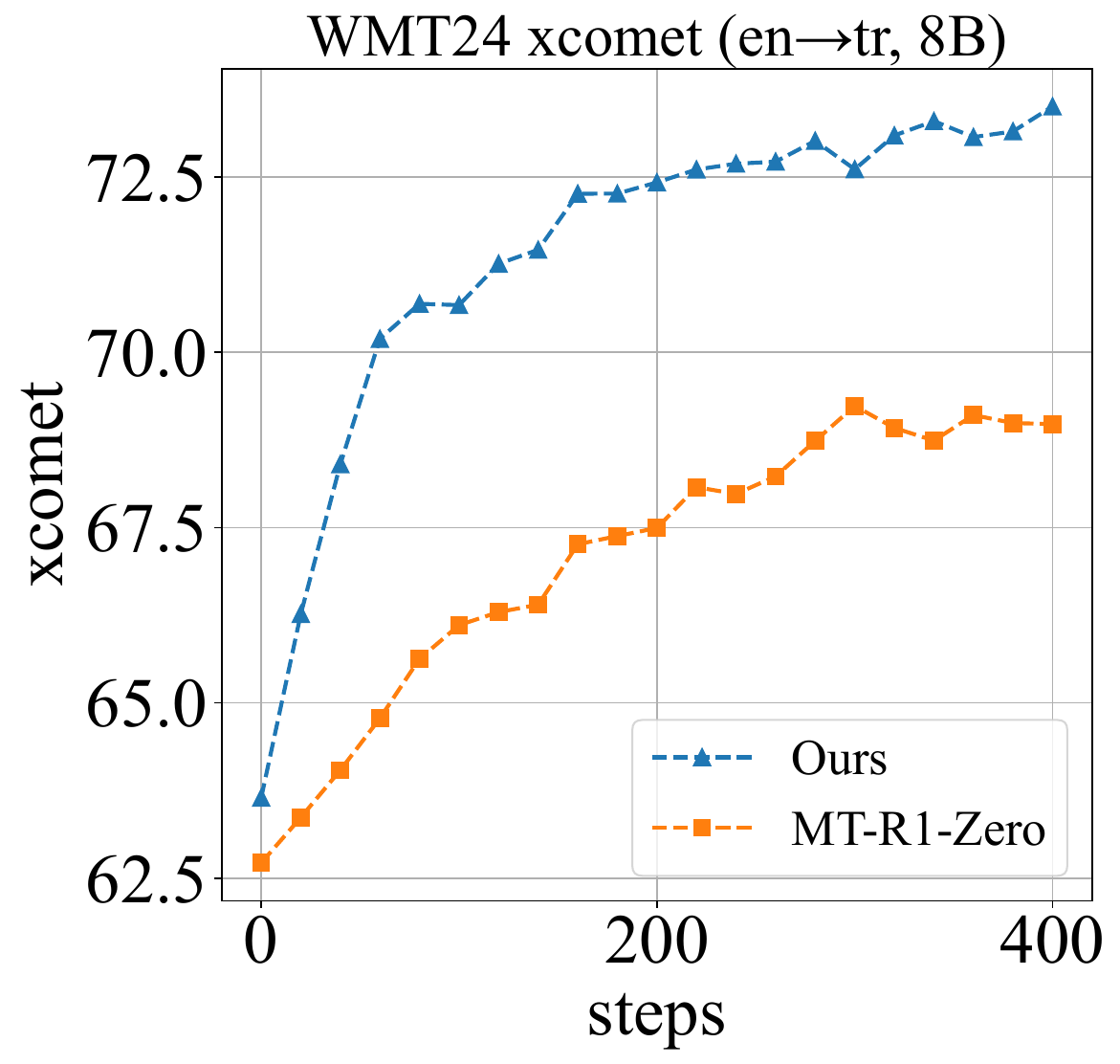}
    }

    \vspace{2mm}

    \subfigure{
        \includegraphics[width=0.29\textwidth]{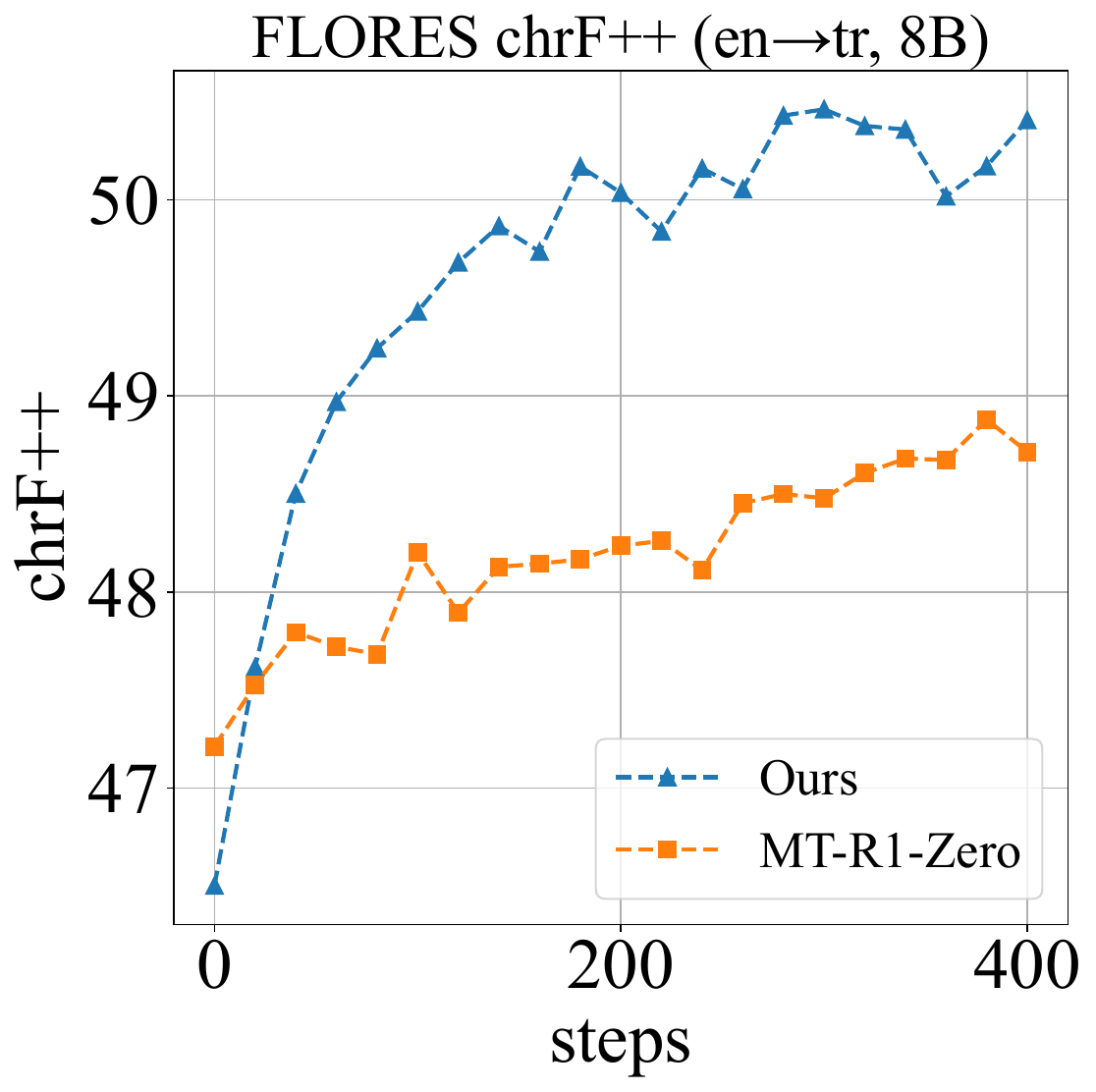}
    }
    \hfill
    \subfigure{
        \includegraphics[width=0.29\textwidth]{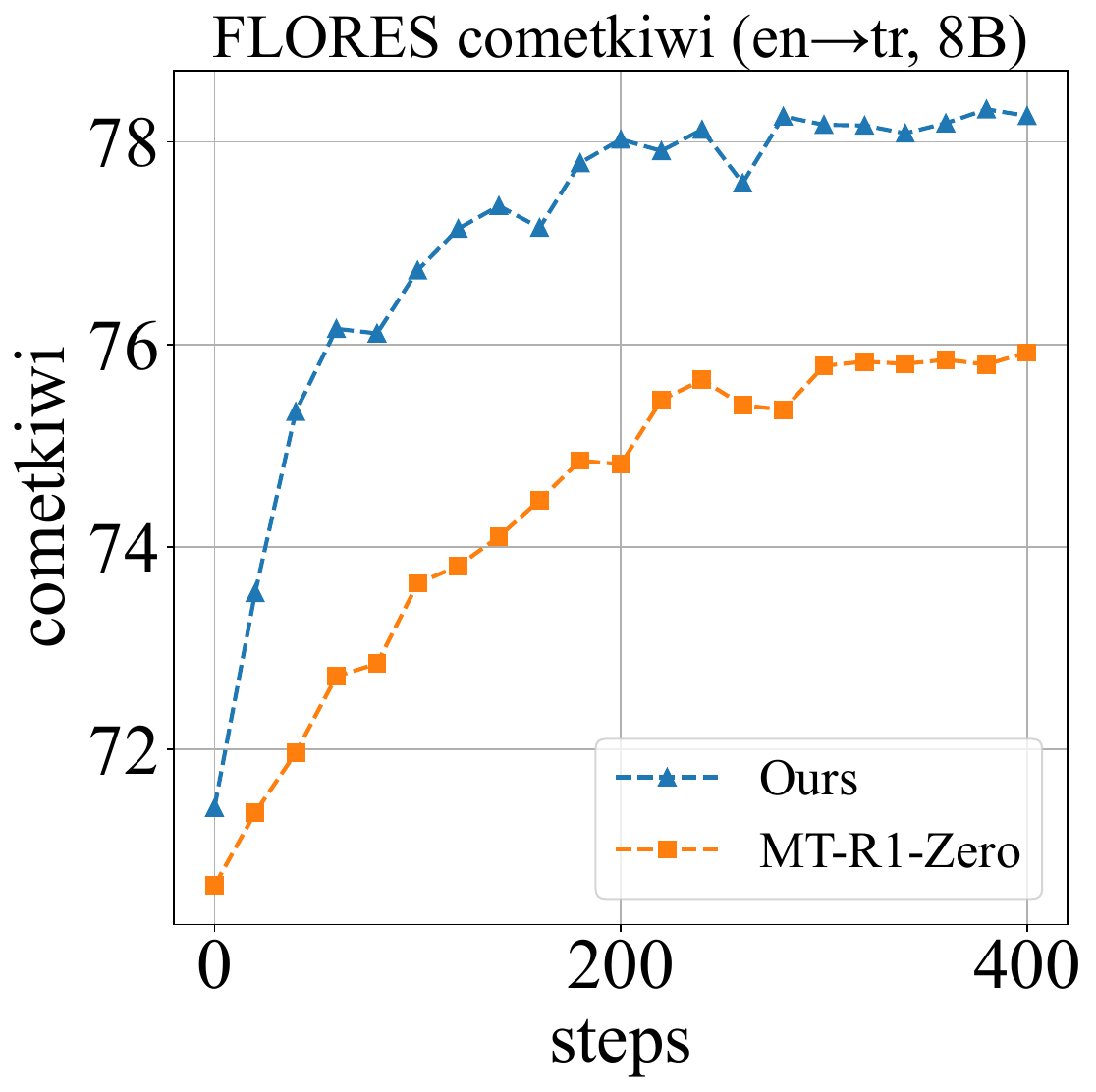}
    }
    \hfill
    \subfigure{
        \includegraphics[width=0.29\textwidth]{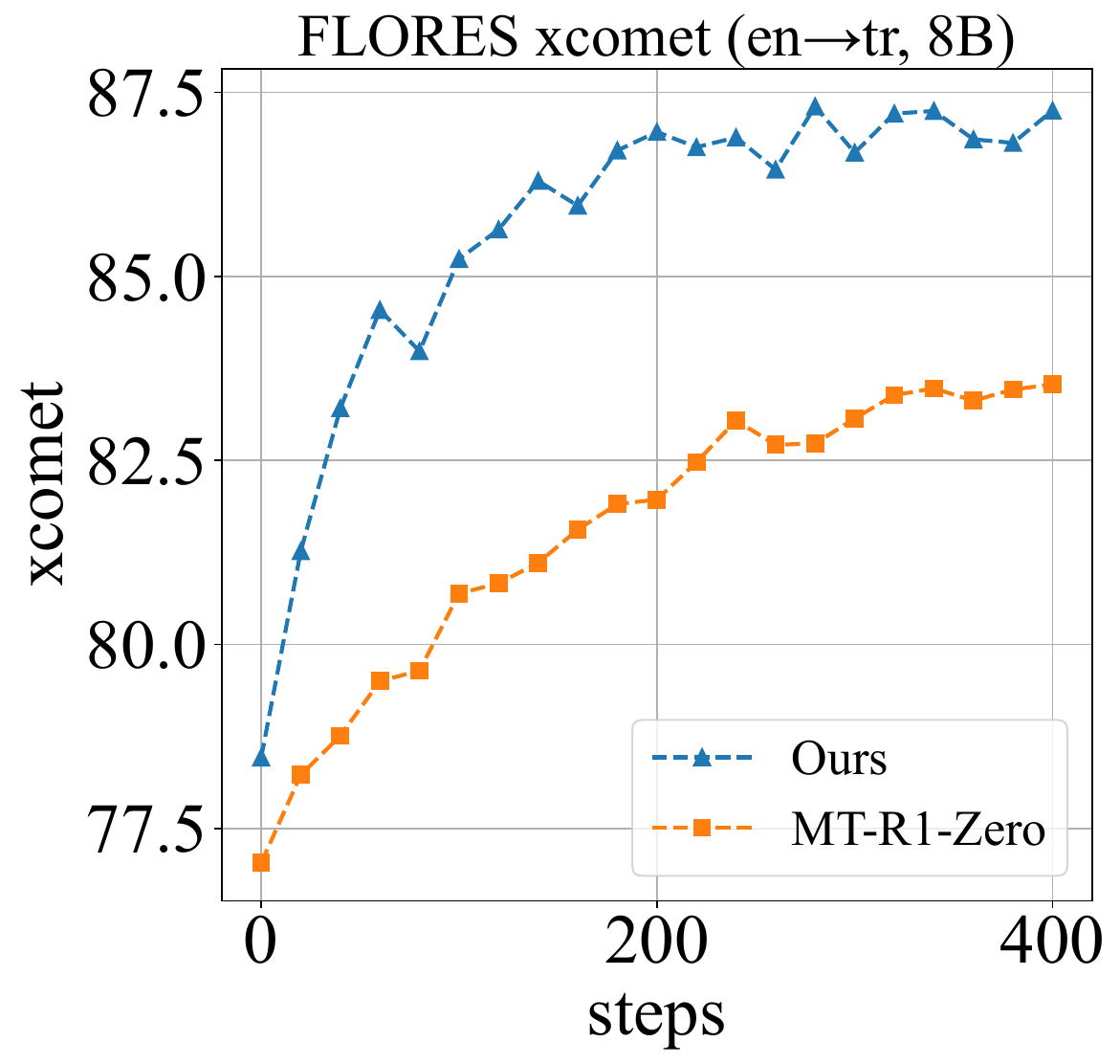}
    }

    \caption{
        Training dynamics on FLORES and WMT24 for EN$\rightarrow$TR
        under different model scales (4B, 8B), evaluated by chrF++, COMET-Kiwi, and XCOMET.
    }
    \label{fig:appendix_main_en2tr}
\end{figure*}

\begin{figure*}[htbp]
    \centering
    \setlength{\tabcolsep}{0pt}

    \subfigure{
        \includegraphics[width=0.29\textwidth]{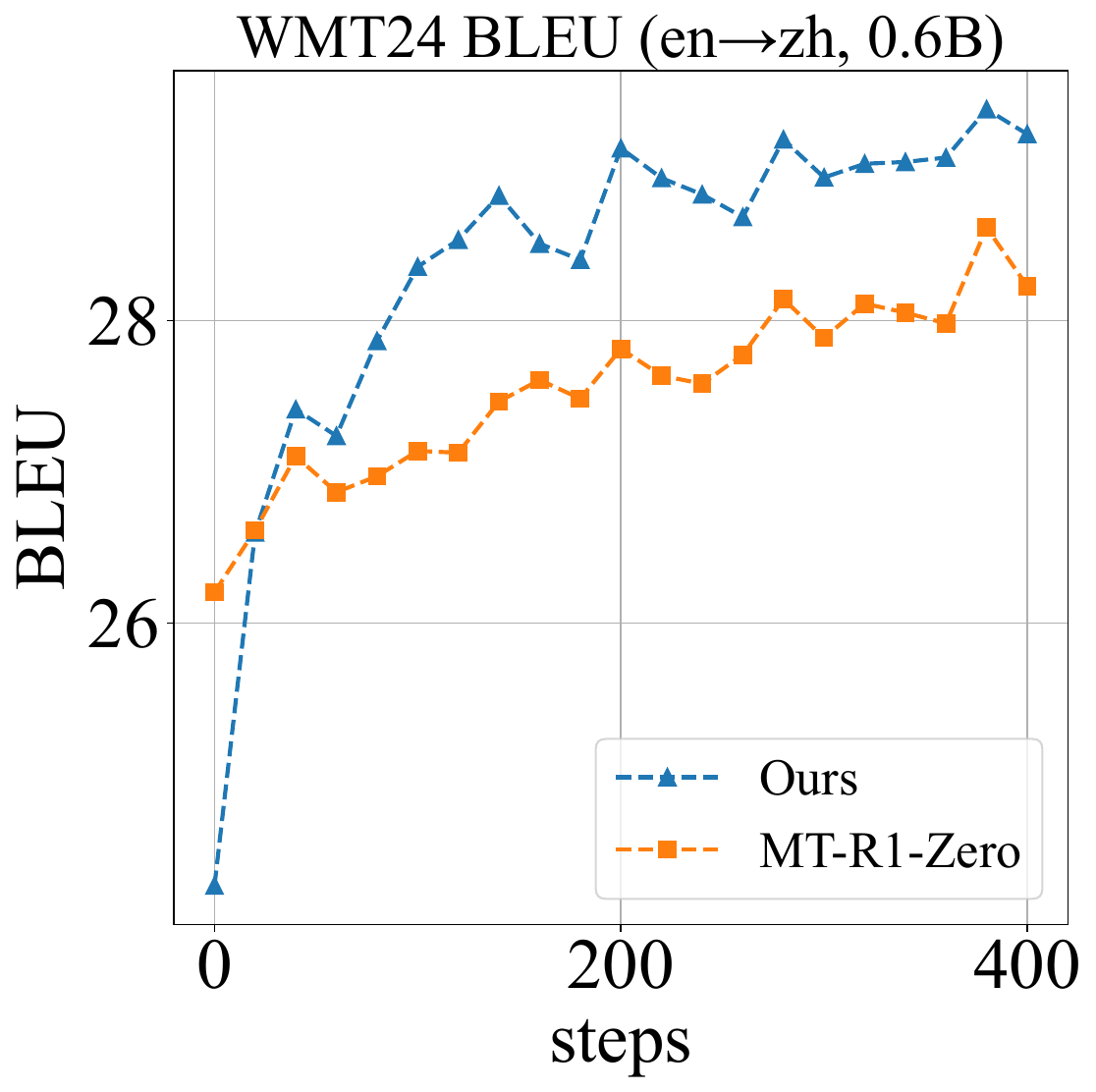}
    }
    \hfill
    \subfigure{
        \includegraphics[width=0.29\textwidth]{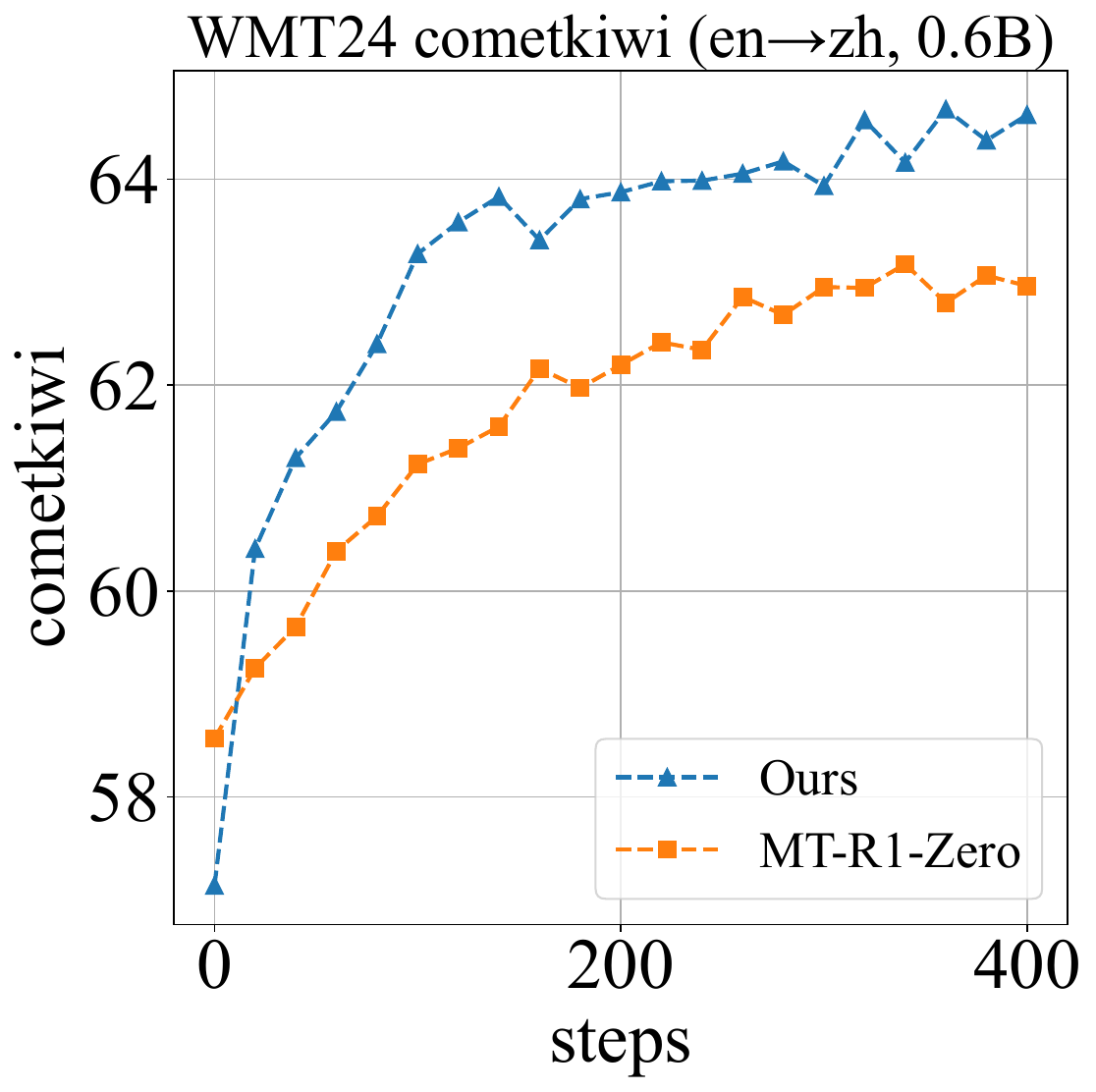}
    }
    \hfill
    \subfigure{
        \includegraphics[width=0.29\textwidth]{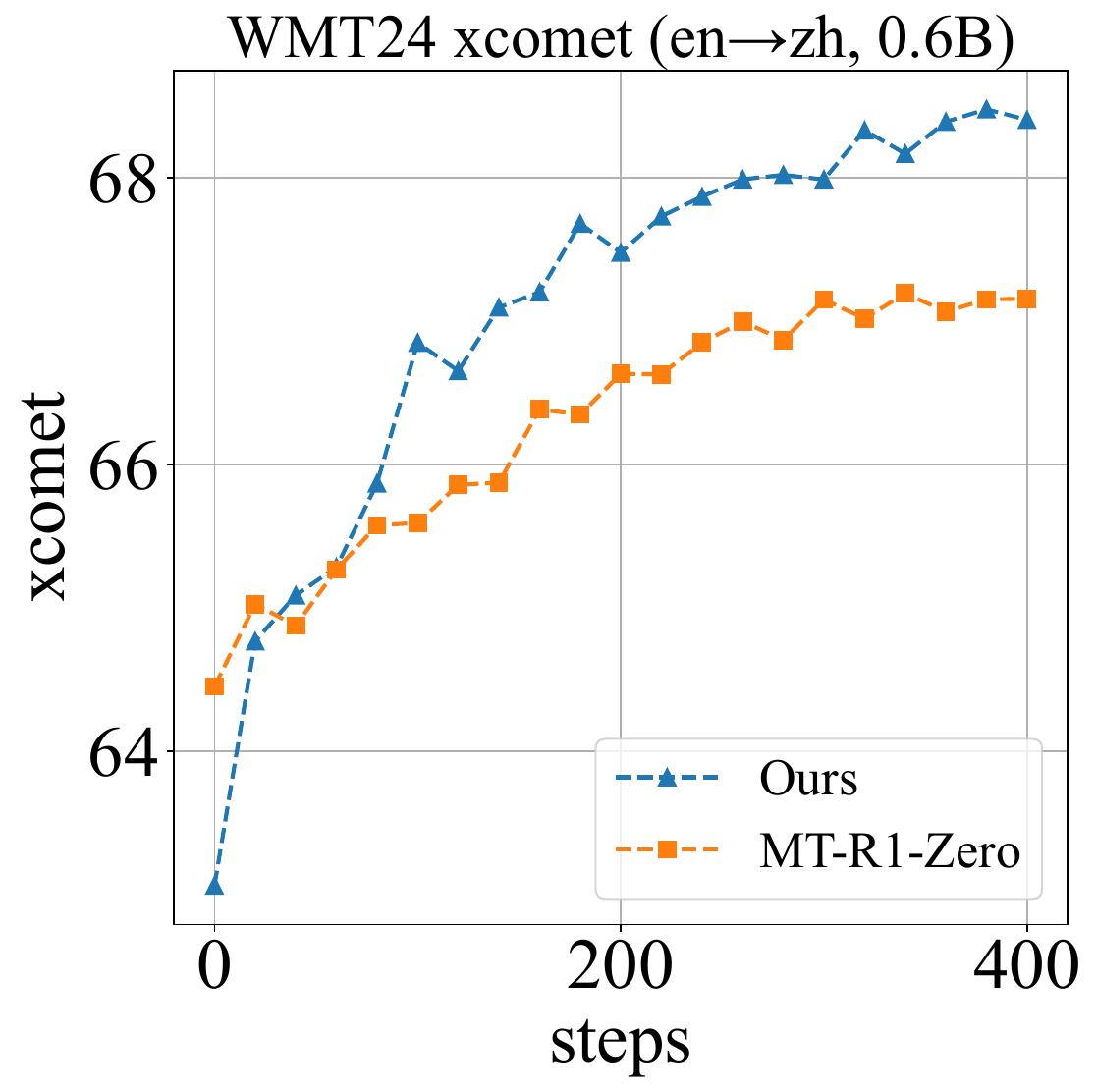}
    }

    \vspace{2mm}

    \subfigure{
        \includegraphics[width=0.29\textwidth]{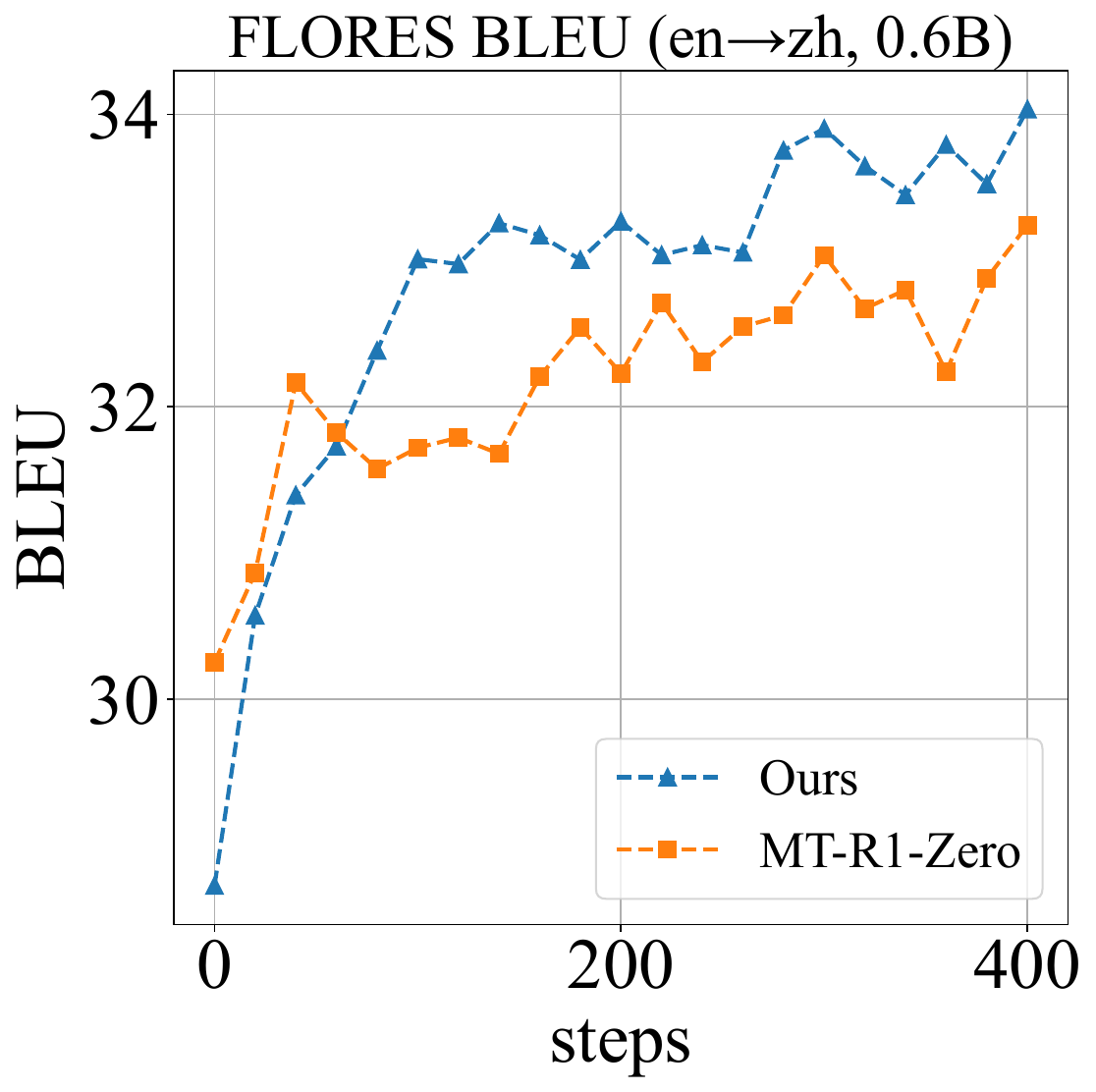}
    }
    \hfill
    \subfigure{
        \includegraphics[width=0.29\textwidth]{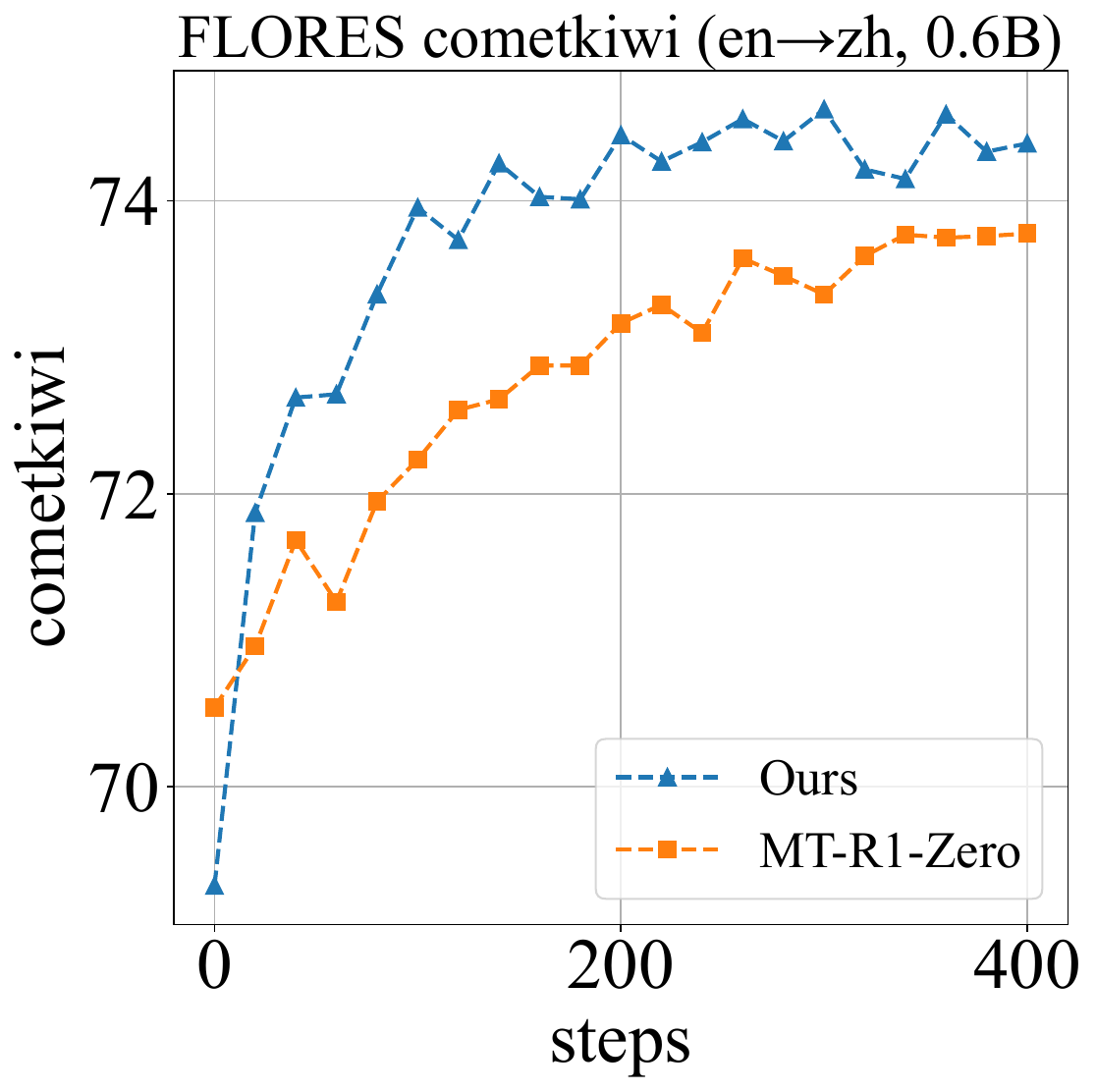}
    }
    \hfill
    \subfigure{
        \includegraphics[width=0.29\textwidth]{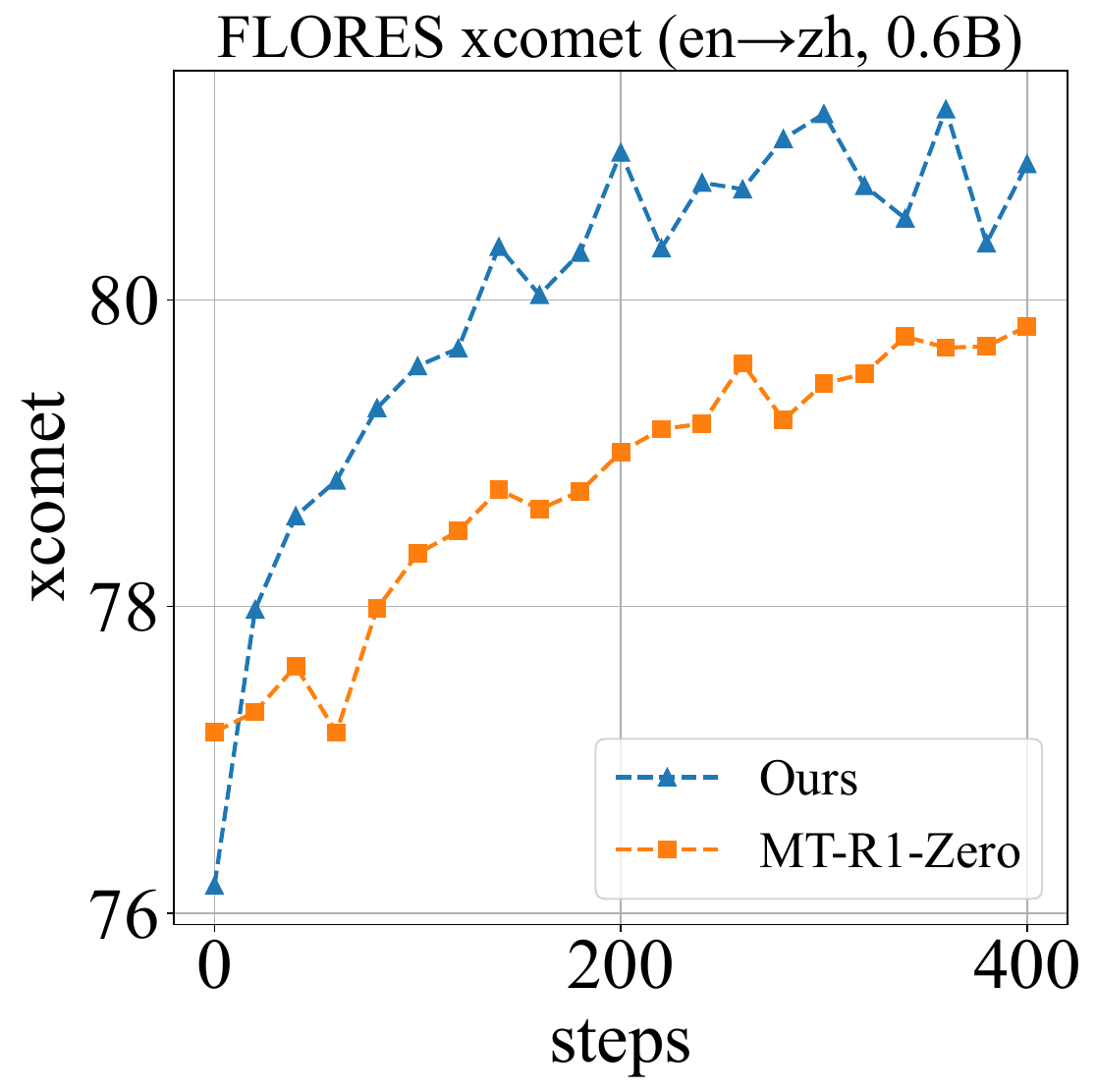}
    }

    \vspace{2mm}

    \subfigure{
        \includegraphics[width=0.29\textwidth]{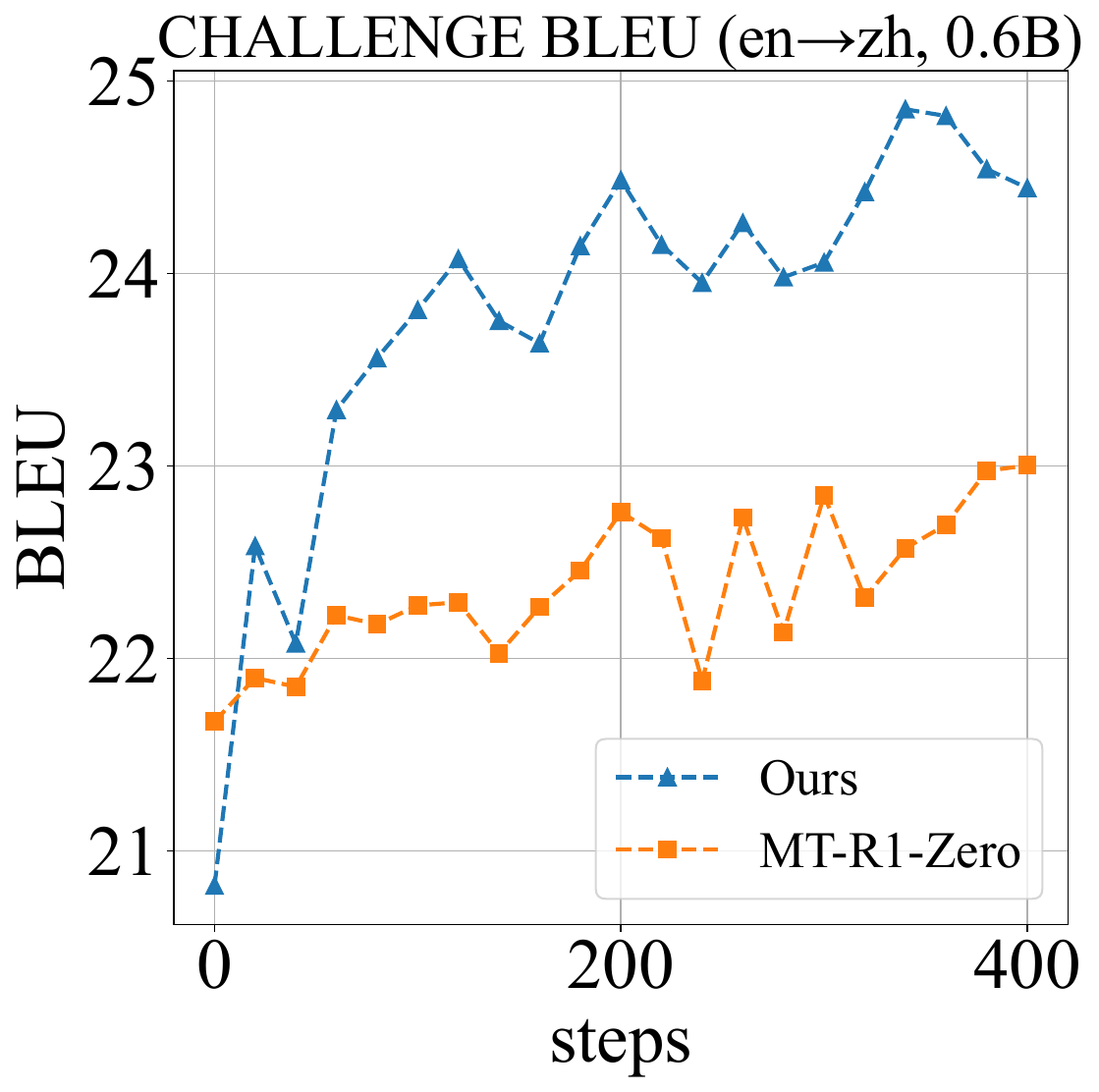}
    }
    \hfill
    \subfigure{
        \includegraphics[width=0.29\textwidth]{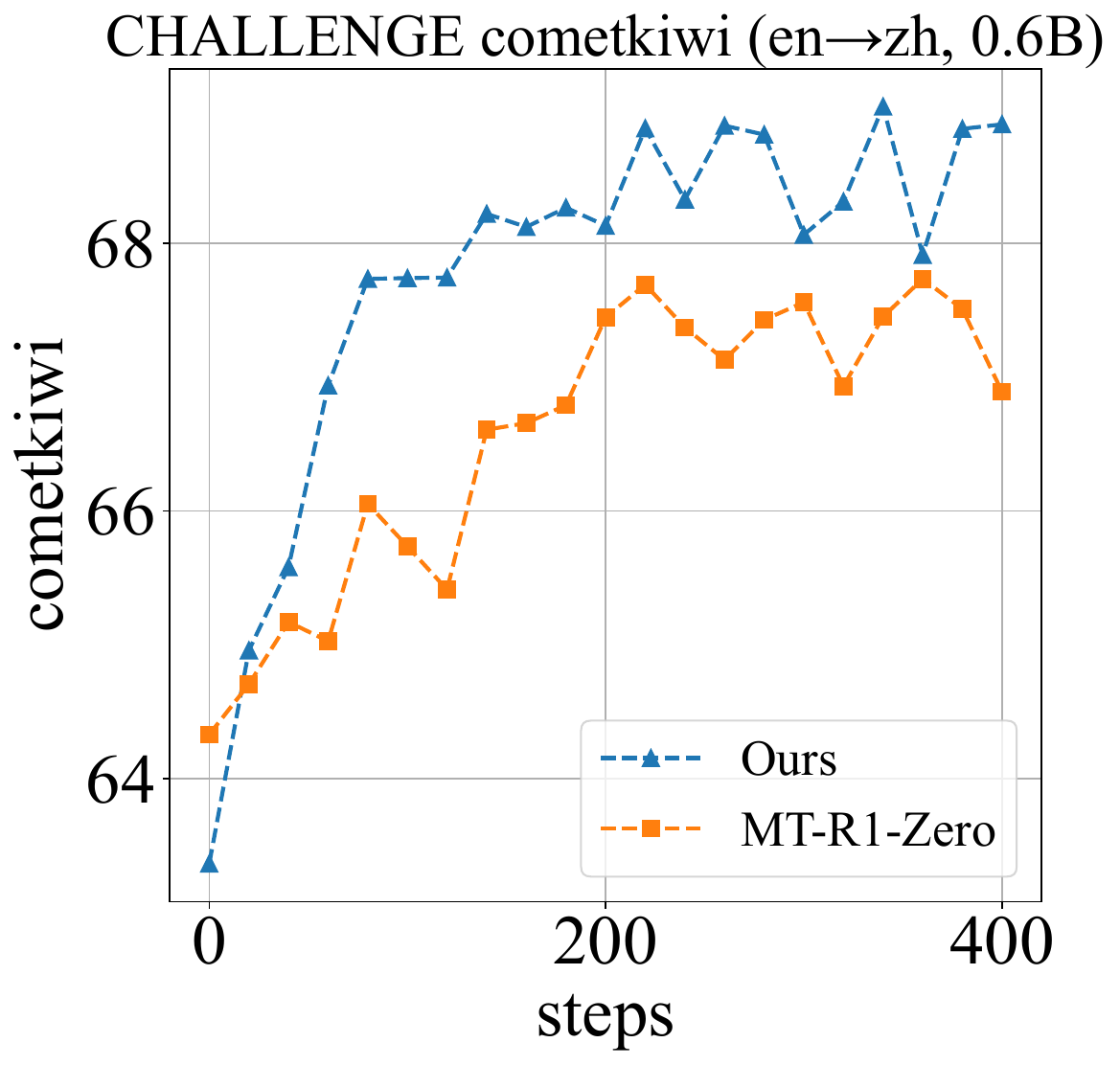}
    }
    \hfill
    \subfigure{
        \includegraphics[width=0.29\textwidth]{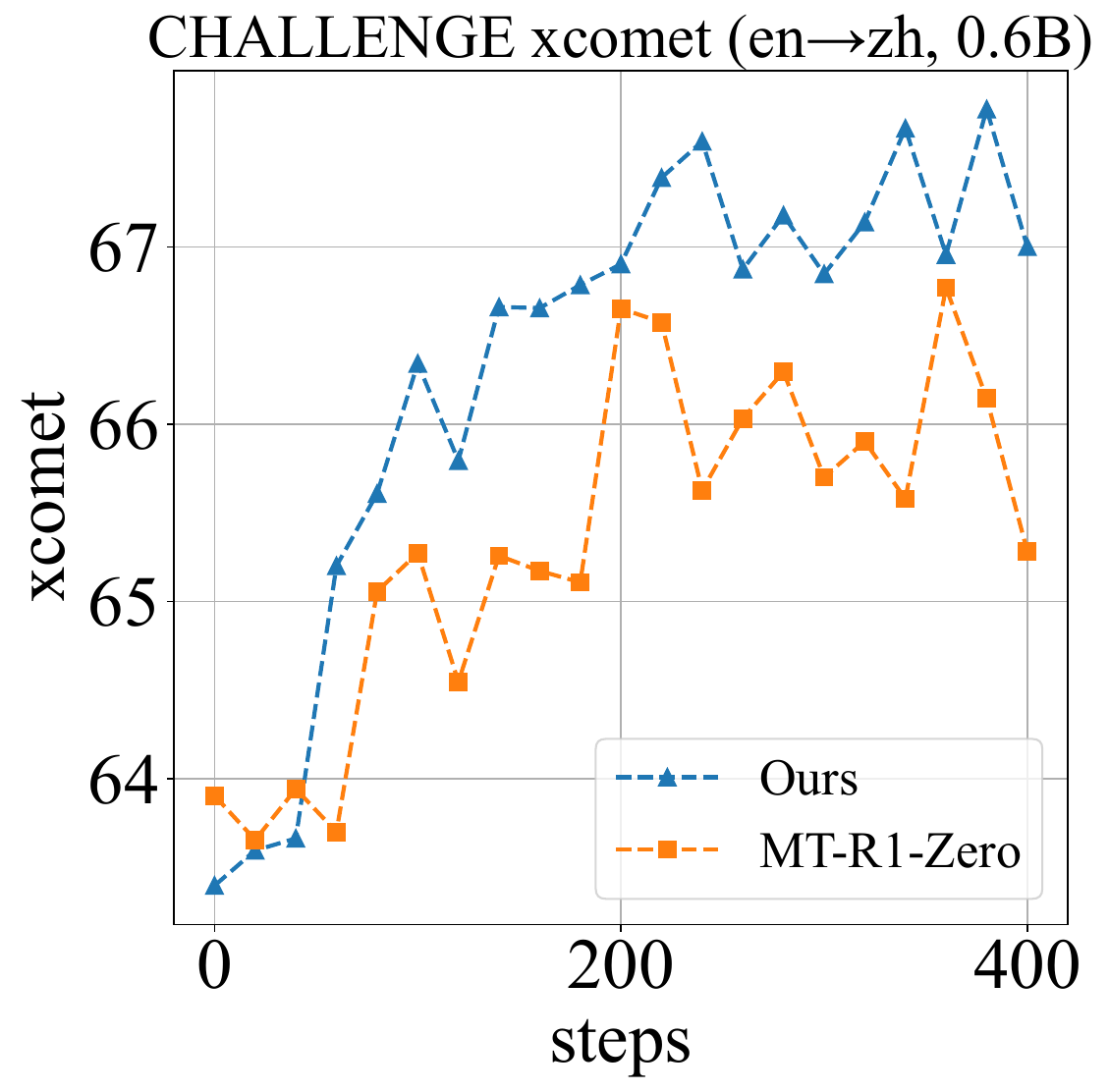}
    }

    \caption{
        Training dynamics on FLORES, WMT24, and Challenge for EN$\rightarrow$ZH
        with a 0.6B model, evaluated by BLEU, COMET-Kiwi, and XCOMET.
    }
    \label{fig:appendix_main_en2zh}
\end{figure*}

\begin{figure*}[htbp]
    \centering
    \setlength{\tabcolsep}{0pt}

    \subfigure{
        \includegraphics[width=0.29\textwidth]{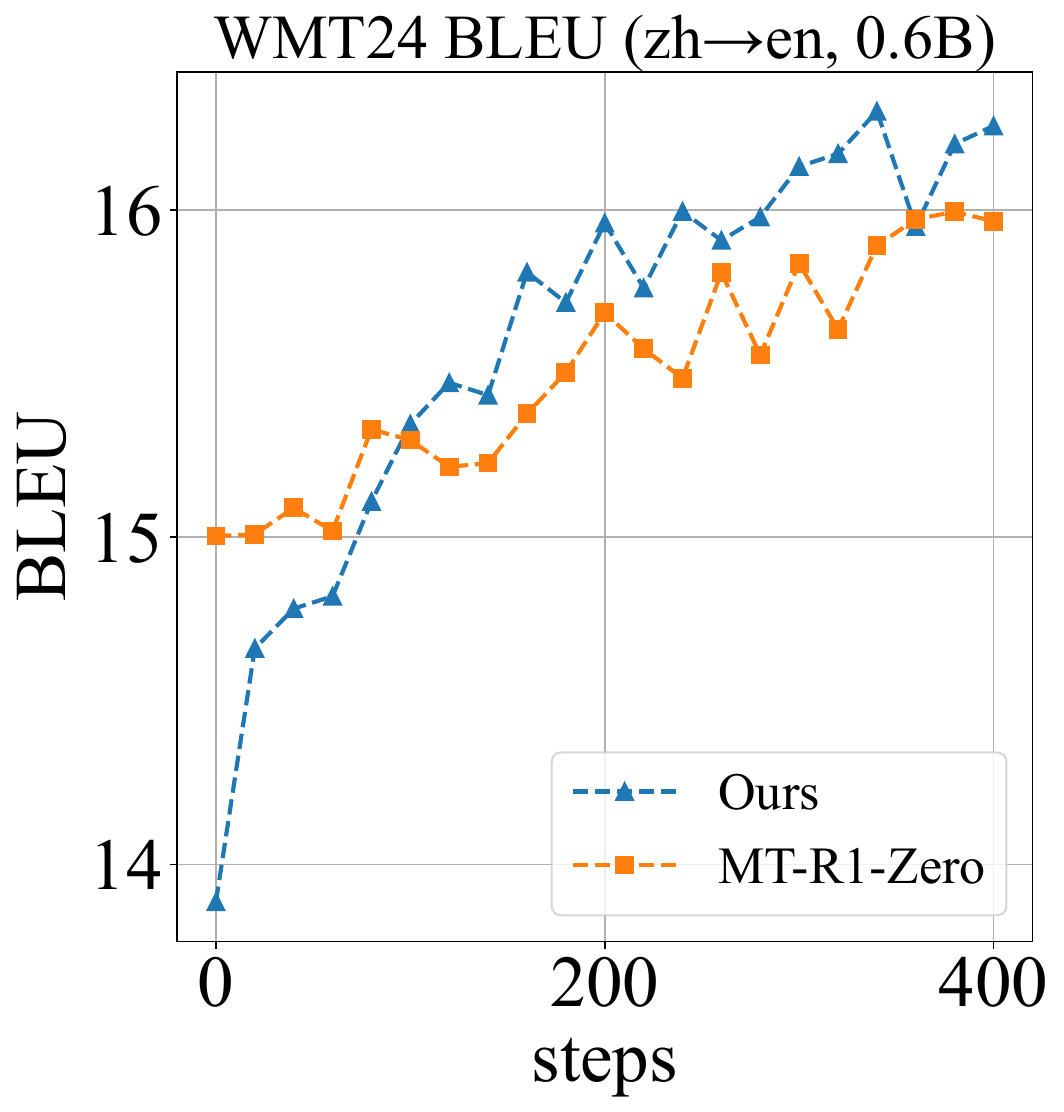}
    }
    \hfill
    \subfigure{
        \includegraphics[width=0.29\textwidth]{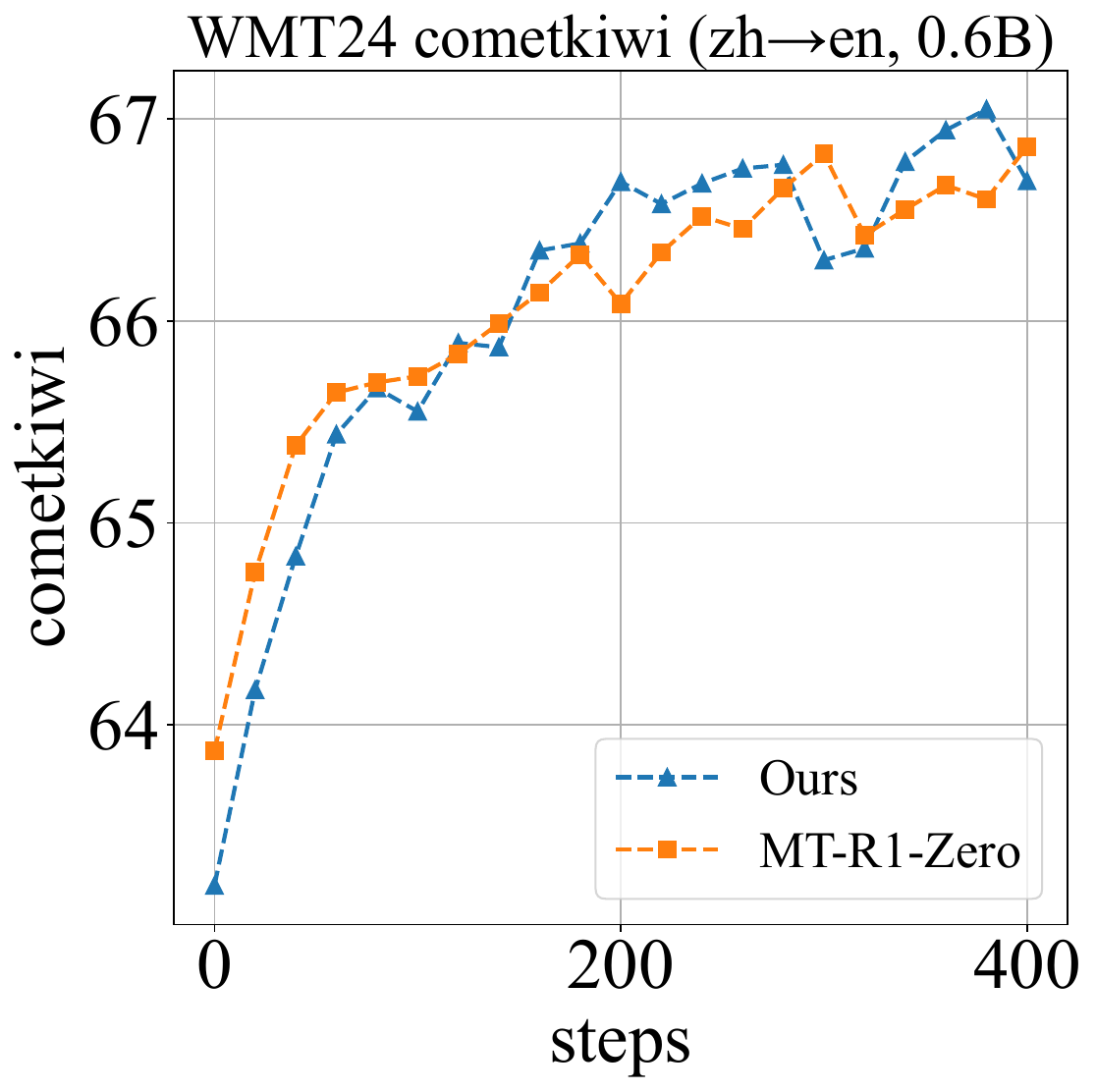}
    }
    \hfill
    \subfigure{
        \includegraphics[width=0.29\textwidth]{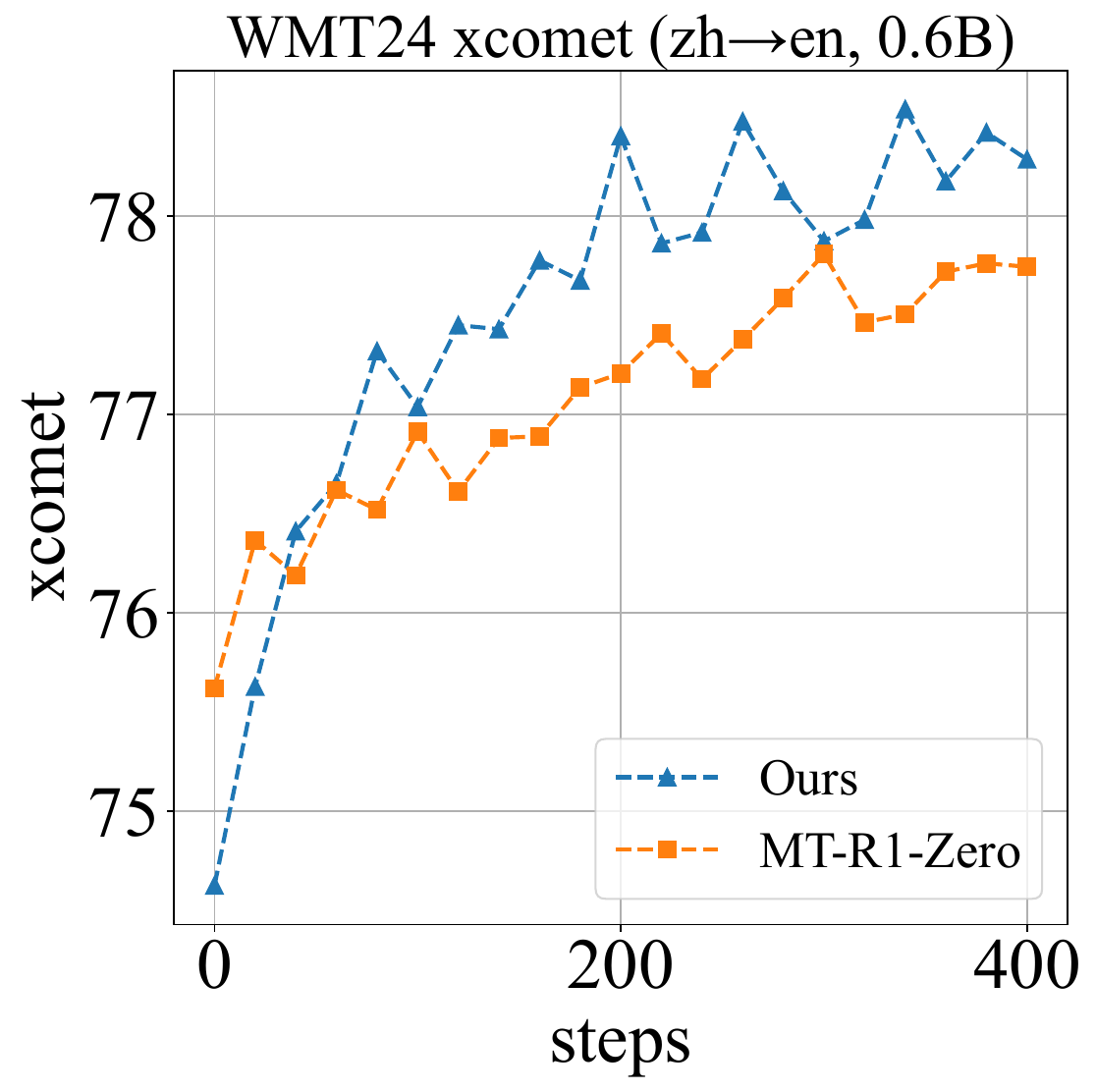}
    }

    \vspace{2mm}

    \subfigure{
        \includegraphics[width=0.29\textwidth]{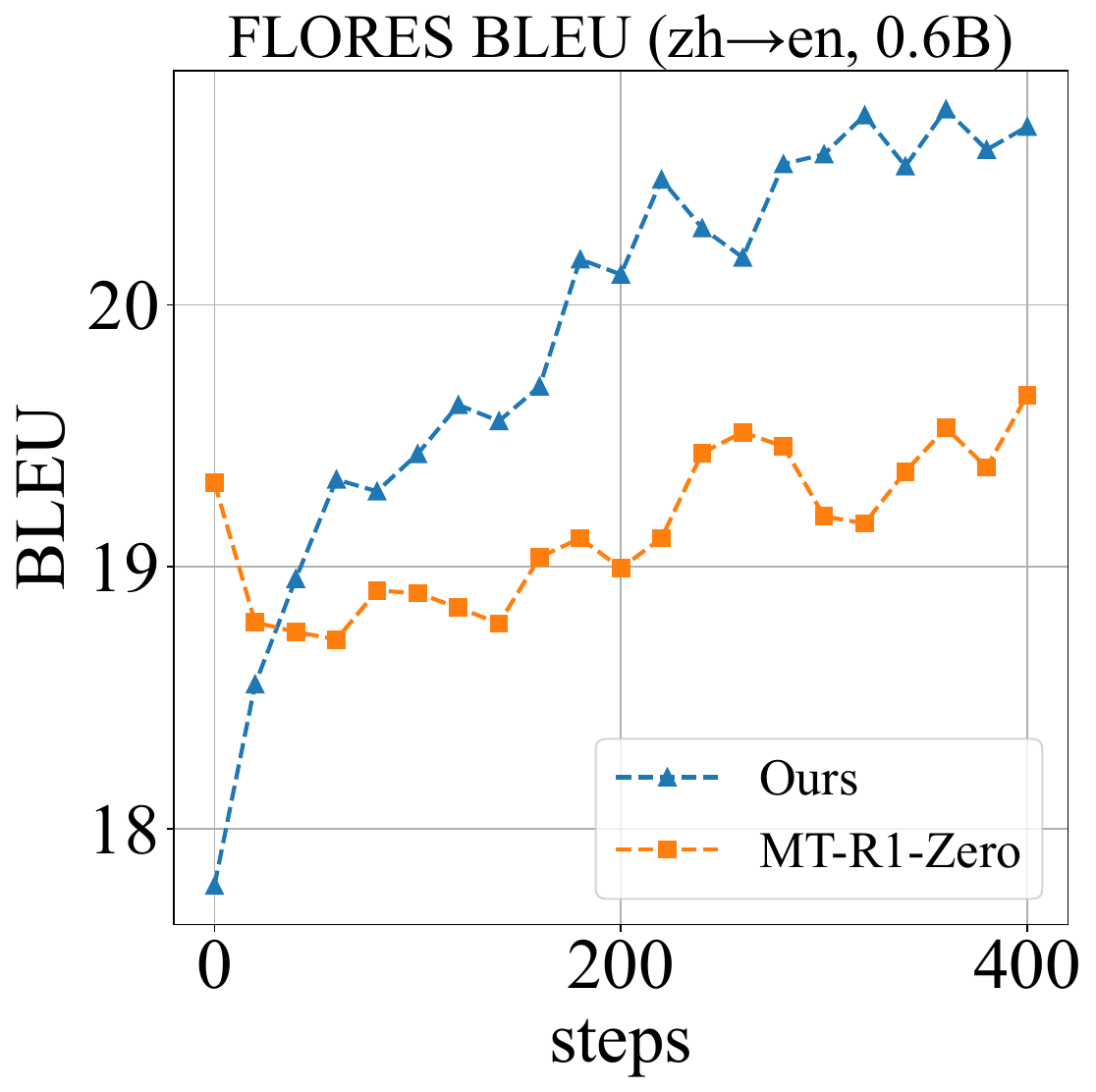}
    }
    \hfill
    \subfigure{
        \includegraphics[width=0.29\textwidth]{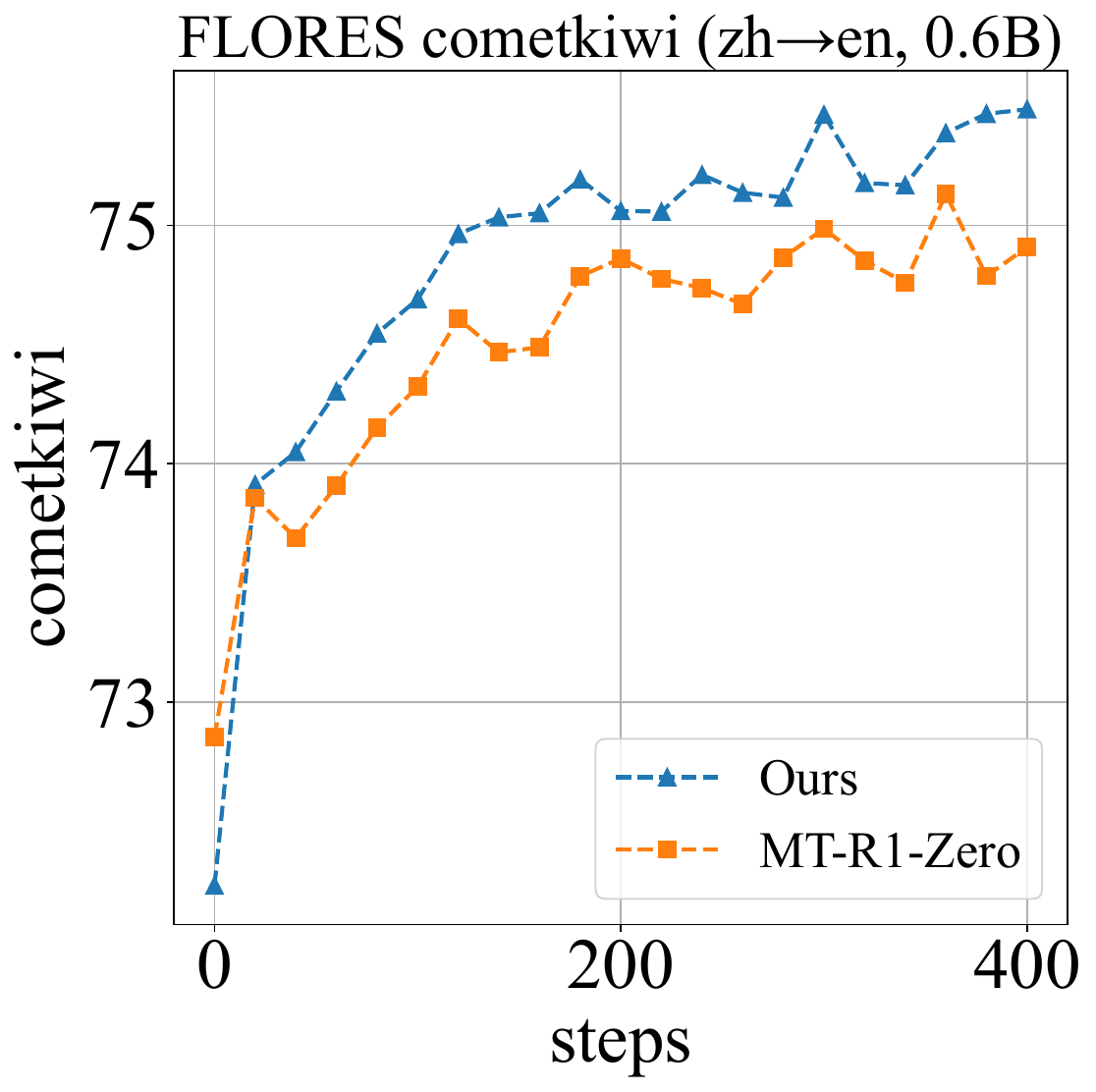}
    }
    \hfill
    \subfigure{
        \includegraphics[width=0.29\textwidth]{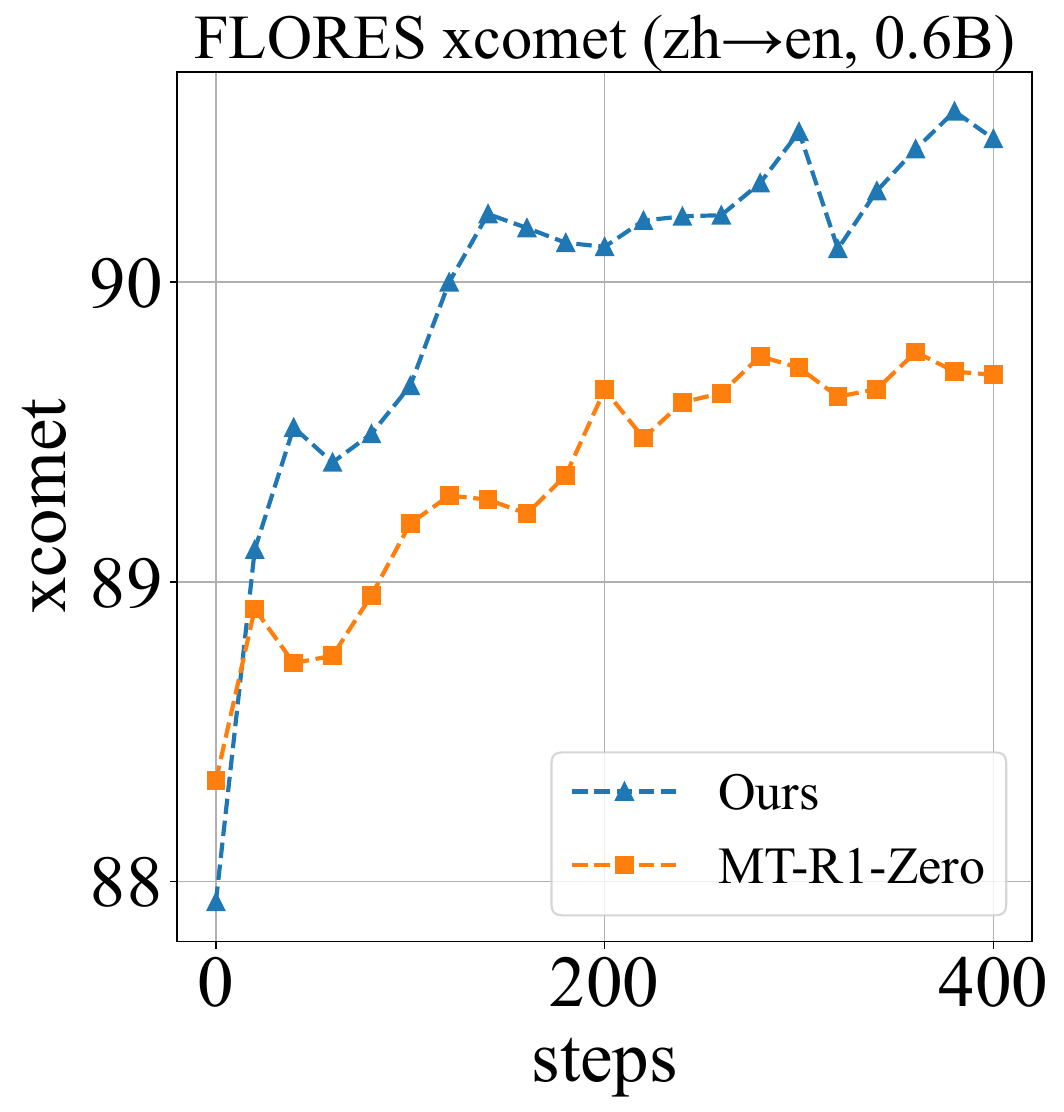}
    }

    \vspace{2mm}

    \subfigure{
        \includegraphics[width=0.29\textwidth]{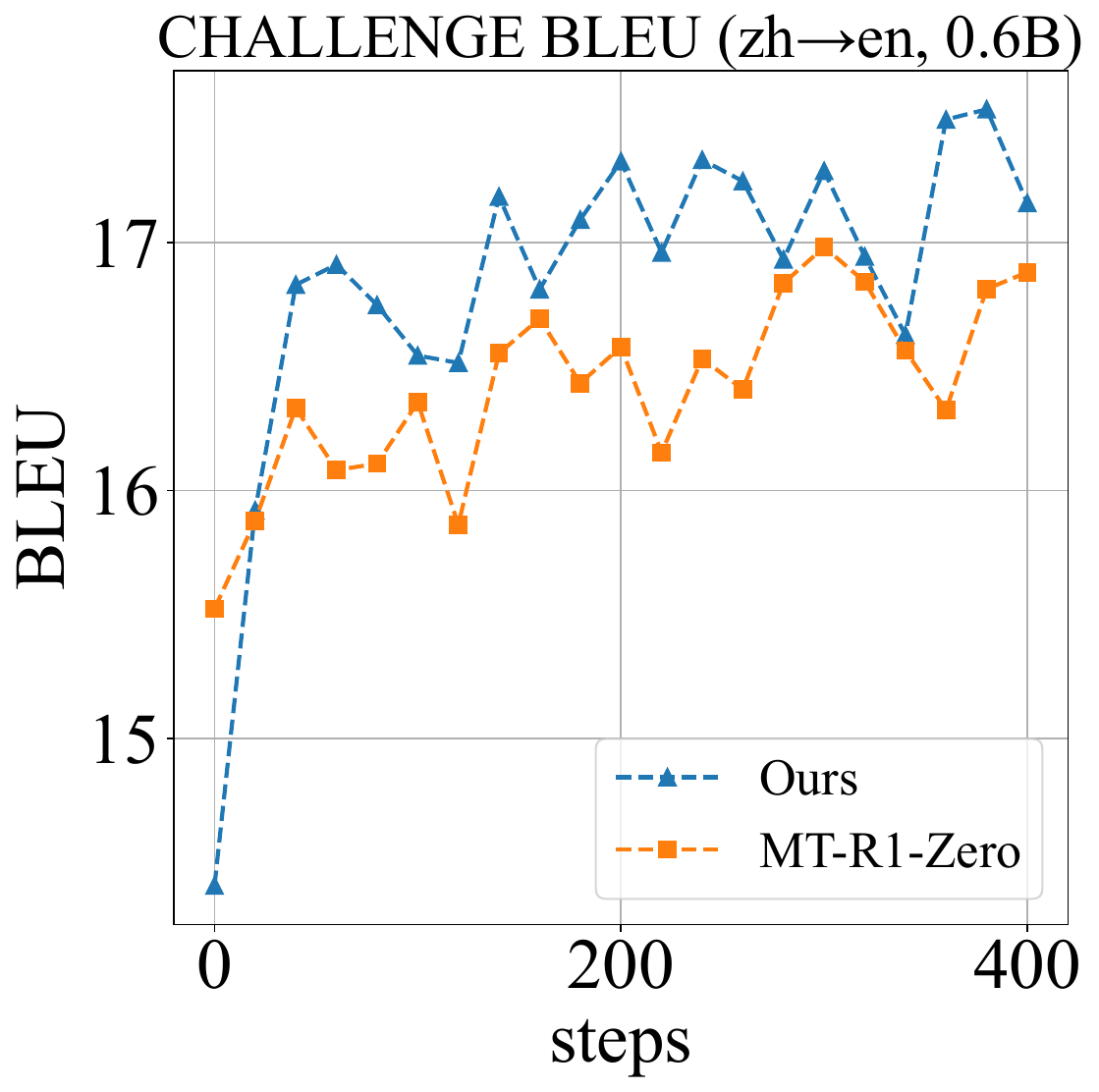}
    }
    \hfill
    \subfigure{
        \includegraphics[width=0.29\textwidth]{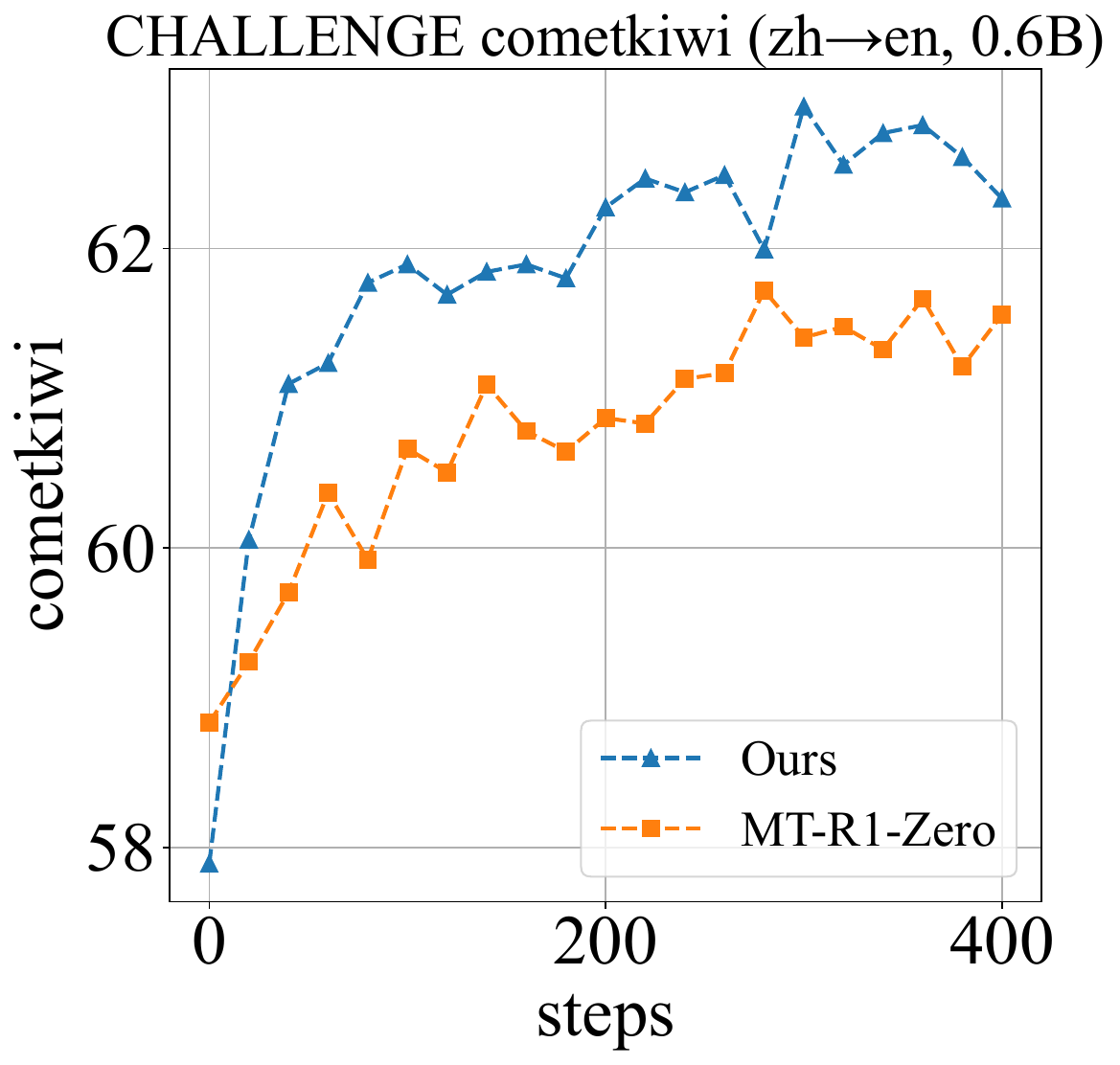}
    }
    \hfill
    \subfigure{
        \includegraphics[width=0.29\textwidth]{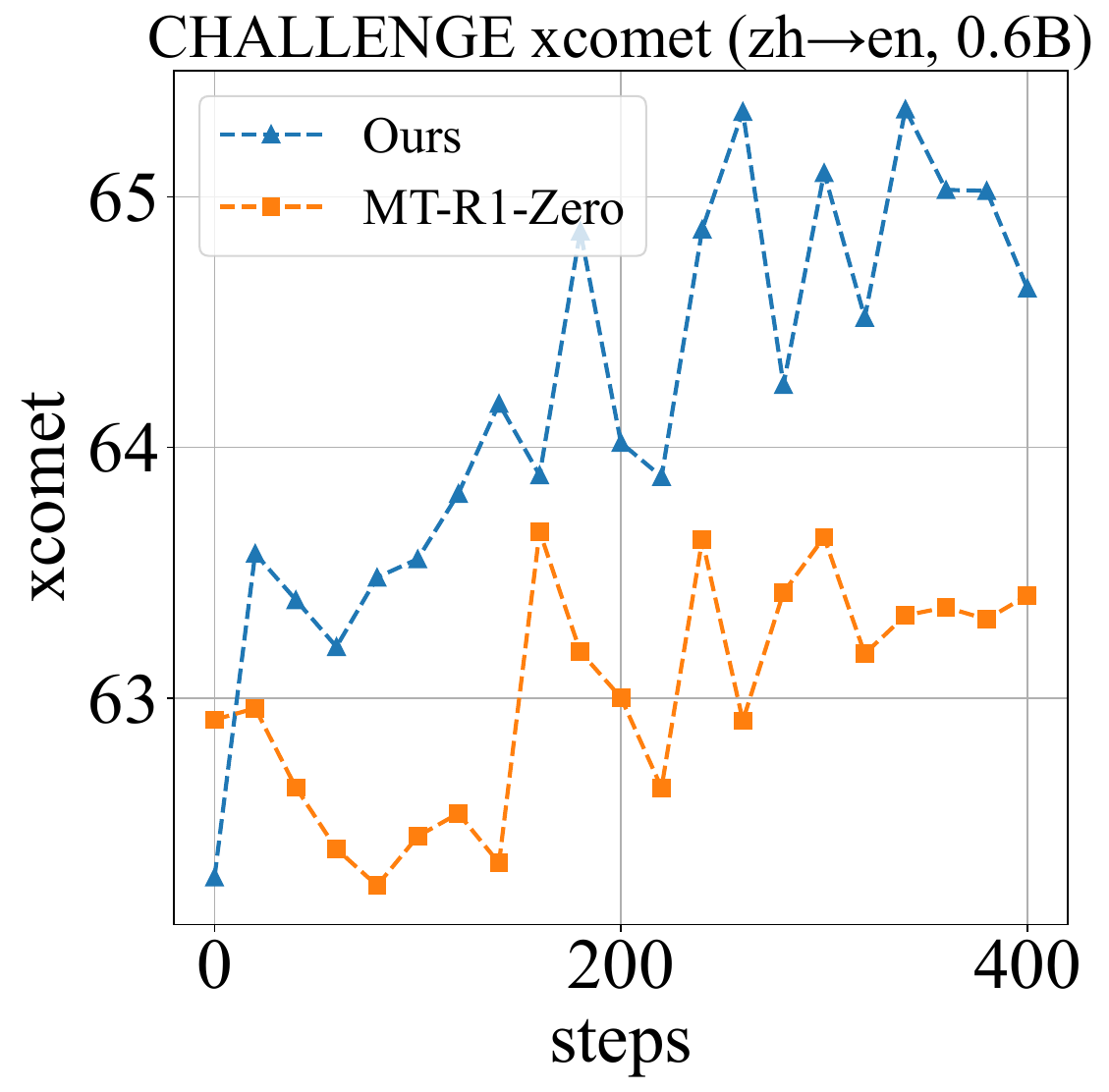}
    }

    \caption{
        Training dynamics on FLORES, WMT24, and Challenge for ZH$\rightarrow$EN
        with a 0.6B model, evaluated by BLEU, COMET-Kiwi, and XCOMET.
    }
    \label{fig:appendix_main_zh2en}
\end{figure*}

\end{document}